\documentclass[10pt,twocolumn,letterpaper]{article}
\usepackage[letterpaper,textwidth=6.875in,textheight=8.875in,columnsep=0.3125in,top=1in,headsep=0in]{geometry}
\usepackage{times}
\usepackage[T1]{fontenc}
\usepackage{graphicx,booktabs,amsmath,amssymb,amsthm,mathtools}
\usepackage{capt-of}
\usepackage[font=small,labelsep=period]{caption}
\usepackage{algorithm,algpseudocode}
\usepackage{placeins}
\usepackage[numbers,sort&compress]{natbib}
\usepackage{xcolor}
\graphicspath{{figures/}}
\definecolor{linkblue}{rgb}{0.21,0.49,0.74}
\usepackage[pagebackref,breaklinks,colorlinks,allcolors=linkblue]{hyperref}
\newcommand{\certex}{\mbox{CertEx}}
\newtheorem{proposition}{Proposition}

\theoremstyle{definition}
\newtheorem{definition}{Definition}
\DeclareMathOperator*{\argmin}{arg\,min}
\newcommand{\lift}{\ell}
\newcommand{\R}{\mathbb{R}}
\date{}

\title{\certex: Exact Visual Retrieval at Scale by Certified Exclusion}
\author{Dongfang Zhao (dzhao@uw.edu)\\
University of Washington, Tacoma School of Engineering and Technology\\
{\small GitHub: \url{https://github.com/hpdic/libcertex}}\\
{\small PyPI: \url{https://pypi.org/project/certex/}}}

\begin{document}
\maketitle

\begin{abstract}
Exact $k$-nearest-neighbor (kNN) search over image descriptors and embeddings
underlies visual retrieval, near-duplicate detection, the evaluation of learned
representations, and increasingly the memory of LLM agents. However, exhaustive search computes every distance at
full precision, whereas approximate indexes return recall below one. We observe
that exact search only needs a proof that each excluded vector cannot be a
neighbor. Because such a proof only has to be conservative, it can be checked in
low precision, much as the retina examines the whole scene at low resolution. We
call this principle certified exclusion and realize it in an exact kNN method for
GPUs, which we call CertEx. Specifically, CertEx stores a lower-bounding
projection of each vector in half precision and folds a rounding certificate into
the same tensor-core dot product, so that only the few vectors that the test keeps
need exact distances. We prove that CertEx returns the exact kNN
of every query. On sixteen datasets with up to 100M vectors, CertEx is exact for every query at 7.1$\times$ to
18.0$\times$ the throughput of the fastest cuVS search on one H100 (8.8$\times$ to 11.8$\times$ on an A100) and outperforms quantized
and graph indexes whenever every query must be exact.
\end{abstract}

\begin{figure*}[t]
\centering
\includegraphics[width=\linewidth]{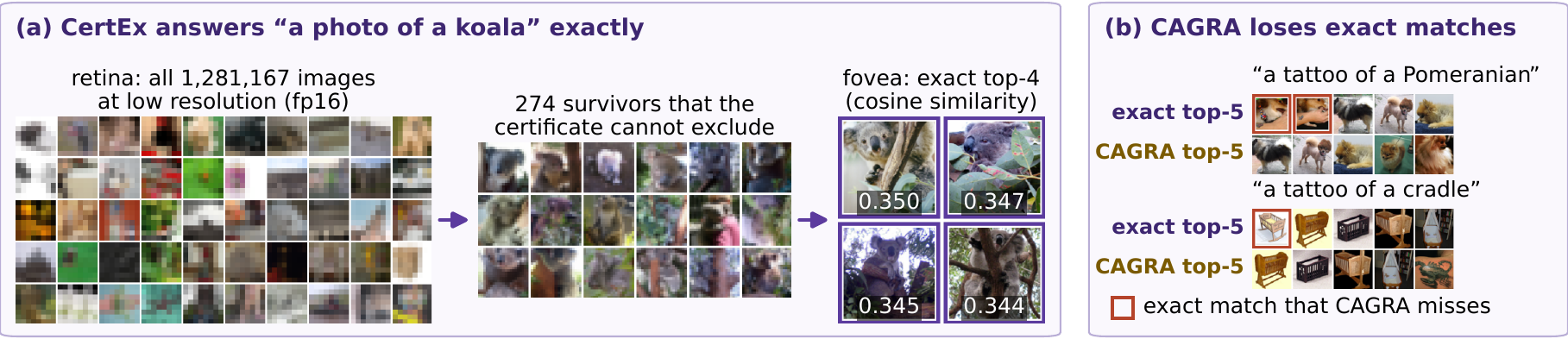}
\caption{Text-to-image retrieval over 1.28M ImageNet images: \certex\ (a); exact and CAGRA top-5, ranked from left to right (b).}
\label{fig:teaser}
\end{figure*}

\section{Introduction}
\label{sec:intro}

Exact kNN search over large datasets of visual features is a basic operation
whose errors propagate into the tasks built on top of it. For instance, image
retrieval ranks database descriptors such as SIFT~\cite{lowe2004sift} by their
distance to a query descriptor, so that a missed neighbor removes a relevant
image from the ranked list. Similarly, near-duplicate detection in web-scale
training corpora such as LAION~\cite{schuhmann2022laion} compares every image
embedding with its nearest neighbors~\cite{pizzi2022sscd}, so that a missed
neighbor keeps a duplicate pair in the training data. Likewise, the kNN
classifier that evaluates self-supervised
representations~\cite{wu2018instance,caron2021dino,oquab2024dinov2} reports the
search error together with the quality of the representation whenever its
neighbors are approximate. Moreover, LLMs and agents increasingly read long contexts and
memories by nearest-neighbor
search~\cite{khandelwal2020knnlm,wu2022memorizing,wang2023longmem,xu2024retrievallong,lewis2020rag,borgeaud2022retro,asai2024selfrag,gutierrez2024hipporag},
so that a missed neighbor drops a relevant memory. Finally, every benchmark of approximate search
measures recall against exact neighbors and therefore inherits any error of the
exhaustive search that computes them.

Existing systems respond to the cost of exact search in two ways, both of which
leave a gap. First, exhaustive search on
GPUs~\cite{johnson2021faiss,douze2024faisslib} computes every query-to-vector
distance at full precision. The throughput of exhaustive search is therefore
bounded by one full-precision distance per vector, and its capacity is bounded by
device memory, which must hold every vector at full precision. Second, approximate
indexes based on quantization~\cite{jegou2011pq,guo2020scann,gao2024rabitq} or
proximity graphs~\cite{malkov2020hnsw,subramanya2019diskann,ootomo2024cagra} are
much faster. However, these indexes trade recall for speed by design and rely on
an index that must be built before the first query and repaired when the
dataset changes~\cite{zhao2026sivf}. Practitioners therefore choose between exact results at high
cost and fast results whose errors are not controlled per query.

To close this gap, we follow the division of labor in human vision: the
periphery examines the whole visual field in parallel at low resolution, whereas
high resolution is spent only on the few locations that attention
selects~\cite{curcio1990photoreceptor,treisman1980feature}. Transferred to kNN
search, this principle examines every database vector at low precision in one
parallel pass and spends full precision only on the few vectors that the pass
cannot exclude. However, a low-precision distance is not exact. Consequently,
unless every low-precision decision is certified to be conservative, the pass
loses neighbors in the same way as an approximate index.

We address this difficulty with a principle that we call \emph{certified
exclusion}: a test may run in any arithmetic as long as it never excludes a vector
within the threshold, so that verifying the kept vectors yields the exact
neighbors. We realize this principle in an exact search method that we call
\certex. Specifically, \certex\ stores a lower-bounding projection of every
vector in half precision, together with four coordinates that fold the rounding and accumulation
errors~\cite{higham2002accuracy,blanchard2020tensorcore} of the tensor-core dot
product into the product itself. As a result, one dot product and one comparison
per vector decide whether the vector can still be a neighbor, without a top-$k$
selection and without a false dismissal. By analogy with human vision, we call
this half-precision copy of the database the \emph{retina} and the exact
verification of the vectors that it keeps the \emph{fovea}. Unlike
quantization~\cite{jegou2011pq,gao2024rabitq} and
graph~\cite{malkov2020hnsw,ootomo2024cagra} indexes, which decide where a query
should look, the retina examines every vector and backs each exclusion by a
proof.

For instance, Figure~\ref{fig:teaser}(a) follows one text query through \certex:
the certified test keeps 274 of the 1,281,167 ImageNet~\cite{deng2009imagenet}
training images, from which the fovea returns the exact best matches. By
comparison, CAGRA~\cite{ootomo2024cagra} loses the best matching image for the two
prompts of Figure~\ref{fig:teaser}(b), even at its most accurate configuration.
Beyond these examples, \certex\ returns exact results on one H100 GPU with
7.1$\times$ to 18.0$\times$ the throughput of the fastest
cuVS~\cite{nvidia2025cuvs} fp32 configuration on sixteen datasets, with a speedup
that grows with the database size. \certex\ also outperforms the GPU indexes
IVF-Flat~\cite{sivic2003videogoogle}, IVF-PQ~\cite{jegou2011pq},
IVF-RaBitQ~\cite{gao2024rabitq}, and CAGRA~\cite{ootomo2024cagra} whenever every
query must be exact.

In summary, we make three contributions.
\begin{itemize}
\item We propose certified exclusion, a new principle for exact search that
replaces the ranking of all $N$ vectors by one keep-or-exclude decision per
vector. Because the decision never excludes a vector within the threshold, even
when it runs in half precision, verifying only the kept vectors returns the exact
neighbors.
\item We design \certex\ as an exact kNN method that realizes certified exclusion
on GPU tensor cores~\cite{blanchard2020tensorcore}. We call its data structure
the retina, which is a single half-precision matrix with one row per database
vector. Each row holds a lower-bounding projection of one vector and four
coordinates of a rounding certificate. We prove that its one-pass
tensor-core test never discards a true neighbor, answers range queries and
similarity self-joins (all pairs within a radius) exactly, and keeps survivors between two balls whose gap shrinks as the
principal coordinates capture more of the data.
\item We evaluate \certex\ on sixteen datasets with up to 100M vectors, including
the embeddings of seven vision encoders such as CLIP~\cite{radford2021clip},
DINOv2~\cite{oquab2024dinov2}, and DINOv3~\cite{simeoni2025dinov3}. \certex\
answers every query exactly at 7.1$\times$ to 18.0$\times$ the throughput of the
fastest exhaustive search of NVIDIA cuVS~\cite{nvidia2025cuvs} and is faster than
IVF-Flat~\cite{sivic2003videogoogle}, IVF-PQ~\cite{jegou2011pq},
IVF-RaBitQ~\cite{gao2024rabitq}, and CAGRA~\cite{ootomo2024cagra} whenever every
query must be exact. On the 12.8M web images of DataComp~\cite{gadre2023datacomp}, \certex\ is
exact for every query and finds all 31.4 billion duplicate pairs, which makes
exact deduplication~\cite{abbas2023semdedup} practical at web scale.
\end{itemize}

\section{Related Work}
\label{sec:related}

\begin{figure*}[t]
\centering
\includegraphics[width=\linewidth]{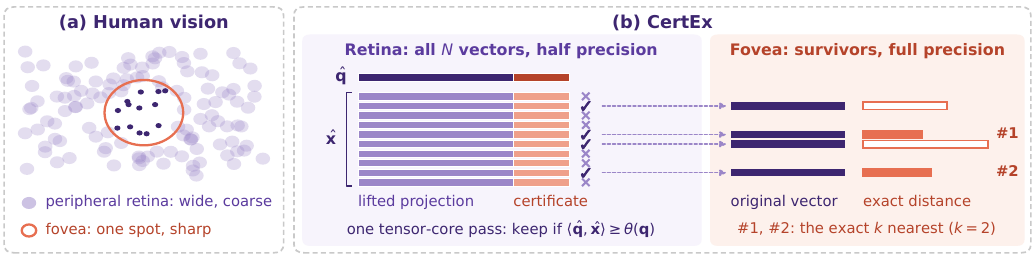}
\caption{Overview of \certex\ and the visual system that motivates it.}
\label{fig:overview}
\end{figure*}

\paragraph{Exact search with lower bounds.}
Exact methods avoid computing every distance by excluding the vectors whose
lower bound exceeds the current neighbor distance. For instance, the k-d
tree~\cite{bentley1975kdtree} and the cover tree~\cite{beygelzimer2006cover}
bound whole regions, whereas AESA~\cite{vidal1986aesa} bounds single points by
the triangle inequality with precomputed pivots. However, the pruning power of such bounds decays with the
dimension~\cite{weber1998vafile,beyer1999meaningful}. The GEMINI
framework~\cite{faloutsos1994gemini} therefore moves the bound into a
contractive feature space with a proof of no false dismissal, which the
VA-file~\cite{weber1998vafile} turns into a sequential scan of compact
approximations. In contrast, random projections preserve distances only up to a
distortion factor~\cite{johnson1984jl} and support only approximate
hashing~\cite{indyk1998lsh,datar2004pstable}, whereas exhaustive GPU
search~\cite{garcia2008gpuknn,johnson2021faiss} avoids bounds altogether. In our
terms, GEMINI and the VA-file are certified exclusion tests in full precision whose
survivors are still ranked against a running candidate set; in contrast, \certex\
certifies a half-precision tensor-core decision inside the same product and
replaces the ranking by one comparison per vector.

\paragraph{Approximate indexes.}
Approximate indexes trade exactness for fewer examined vectors. Specifically,
product quantization~\cite{jegou2011pq,ge2013opq,babenko2012imi,guo2020scann}
compresses vectors and probes a few partitions, to which
RaBitQ~\cite{gao2024rabitq,gao2025extrabitq} adds a probabilistic error bound. In
addition, FLANN~\cite{muja2014flann} automates the choice among tree and k-means
partitions.
In contrast, graph indexes such as HNSW~\cite{malkov2020hnsw},
DiskANN~\cite{subramanya2019diskann}, SONG~\cite{zhao2020song},
GGNN~\cite{groh2023ggnn}, and CAGRA~\cite{ootomo2024cagra} walk a proximity graph
toward the query. Benchmarks~\cite{aumuller2020annbench,simhadri2022bigann}
accordingly rank these methods by recall at a given throughput, whereas \certex\
examines every vector and returns exact results.

\paragraph{Nearest neighbors in visual recognition.}
Visual recognition uses nearest neighbors as a component whose errors enter the
task metric. For instance, object
retrieval~\cite{sivic2003videogoogle,philbin2007oxford,jegou2008hamming,arandjelovic2016netvlad}
is evaluated against exhaustive ground truth~\cite{radenovic2018roxford}.
Similarly, copy detection~\cite{pizzi2022sscd,douze2021isc} and dataset
curation~\cite{abbas2023semdedup,gadre2023datacomp} remove near and semantic
duplicates from web-scale corpora. In representation learning, weighted kNN classifiers evaluate
self-supervised representations~\cite{caron2021dino,oquab2024dinov2,simeoni2025dinov3}.
Retrieval-augmented classification~\cite{long2022rac} even feeds the retrieved
neighbors into the model. In practice, these pipelines obtain their neighbors from
approximate indexes, cluster-restricted comparisons, or kNN graphs of fixed
degree. In contrast, \certex\ returns exact neighbors faster than exhaustive
search and answers deduplication with an exact similarity self-join
(Section~\ref{sec:join}).

\paragraph{Foveation and certified low-precision arithmetic.}
Vision models allocate resolution unevenly, for instance through recurrent
glimpses~\cite{mnih2014ram}, foveal-peripheral segmentation~\cite{abid2026gazemo},
or foveated rendering~\cite{guenter2012foveated,itti1998saliency}. Numerical
computing allocates precision in the same way, for instance in mixed-precision
training~\cite{micikevicius2018mixed}. Computational geometry goes further: its
floating-point filters evaluate a predicate in fast arithmetic together with an
error bound~\cite{higham2002accuracy,blanchard2020tensorcore} and fall back to
exact arithmetic only near zero~\cite{fortune1993exact,shewchuk1997adaptive}.
Following this idea, \certex\ folds the error bound into the same half-precision
product, so that every low-precision decision of \certex\ carries a proof of no
false dismissal, unlike approximate foveation.

\section{Methodology}
\label{sec:method}

\subsection{Overview}
\label{sec:overview}

Exact kNN search must examine all $N$ database vectors, although only the few
vectors near the query need full precision. Human vision solves the same problem
by dividing its work between two parts of the
retina~\cite{curcio1990photoreceptor} (Figure~\ref{fig:overview}(a)): the
periphery covers the whole scene at low resolution, whereas the fovea sees only
one spot at full resolution. Our exact kNN method divides its work in the same
way (Figure~\ref{fig:overview}(b)). We call this method \certex, which stands for
certified exclusion.

The left half of Figure~\ref{fig:overview}(b) shows the retina of \certex, which
stores one half-precision vector $\hat{\mathbf{x}}$ for each of the $N$ database
vectors $\mathbf{x}$. We call $\hat{\mathbf{x}}$ the \emph{retina vector} of
$\mathbf{x}$. The first part of $\hat{\mathbf{x}}$ keeps a few principal coordinates of
$\mathbf{x}$ and summarizes the rest by one norm. We call this part the
\emph{lifted projection} of $\mathbf{x}$ (Section~\ref{sec:lift}). The inner
product of two lifted projections determines a lower bound on the distance
between their vectors.
Although half precision is the fastest arithmetic of GPU tensor
cores~\cite{blanchard2020tensorcore}, rounding to half precision can
overestimate the distance of a true neighbor. To guard against this error, the
second part of $\hat{\mathbf{x}}$ records how large the rounding error of
$\hat{\mathbf{x}}$ can be.
We call this part the \emph{certificate} of $\hat{\mathbf{x}}$. With the certificate, the test excludes a vector
only if the vector would be excluded even under the worst possible rounding.
Moreover, the tensor cores evaluate this worst case within the same product
(Section~\ref{sec:retina}). Concretely, a tensor-core pass over the retina keeps a
vector $\mathbf{x}$
only if $\langle\hat{\mathbf{q}},\hat{\mathbf{x}}\rangle\ge\theta(\mathbf{q})$,
where the retina vector $\hat{\mathbf{q}}$ encodes the query $\mathbf{q}$ in
the same form as $\hat{\mathbf{x}}$ and the threshold $\theta(\mathbf{q})$ follows from an upper
bound $U$ on the distance of the $k$-th nearest neighbor (Section~\ref{sec:fovea}).
Because $U$ is needed before the pass, \certex\ obtains it from a few sampled
vectors. The test of each
vector therefore depends on no other database vector, so that one pass over the
retina decides all $N$ vectors without ranking them. We call this pass the
\emph{retina pass}. As Figure~\ref{fig:overview}(b) shows, the pass excludes most
vectors (crosses) but never a true neighbor; we call the kept vectors (checks)
\emph{survivors}.

The right half of Figure~\ref{fig:overview}(b) shows the fovea, which takes over
after the retina pass. For each survivor, the fovea reads the original vector in
full precision and computes the exact distance between that vector and the query.
As a result, the $k$ survivors with the smallest exact distances are the exact
$k$ nearest neighbors of $\mathbf{q}$, marked \#1 and \#2 for $k=2$ in the
figure.

\subsection{Certified Exclusion}
\label{sec:setup}

Consider a database $\mathcal{X}\coloneqq\{\mathbf{x}_1,\dots,\mathbf{x}_N\}\subset\R^D$
and a query $\mathbf{q}\in\R^D$. Exact kNN search returns the index set
$\mathcal{N}_k(\mathbf{q})\coloneqq\argmin_{\mathcal{S}\subseteq[N],|\mathcal{S}|=k}
\sum_{i\in\mathcal{S}}\|\mathbf{q}-\mathbf{x}_i\|$, whose farthest member lies at
the distance $r_k(\mathbf{q})$ from $\mathbf{q}$. However,
computing $\mathcal{N}_k(\mathbf{q})$ requires exact distances only for the vectors
that may be neighbors. Definition~\ref{def:exclusion} therefore formalizes what a
test that excludes the other vectors must guarantee, whatever arithmetic it
uses.

\begin{definition}[Certified exclusion]
\label{def:exclusion}
A test $T(\mathbf{q},\mathbf{x},U)\in\{\textup{keep},\textup{exclude}\}$ is a
\emph{certified exclusion test} if its computed output keeps every vector $\mathbf{x}$
with $\|\mathbf{q}-\mathbf{x}\|\le U$, for every query $\mathbf{q}$, vector
$\mathbf{x}$, and threshold $U\ge0$.
\end{definition}

Definition~\ref{def:exclusion} constrains only the direction of the errors of a
test. Specifically, a certified exclusion test may keep a vector farther than
$U$ at the cost of one exact distance, but it never excludes a vector within
$U$.

\begin{proposition}[Exact search by certified exclusion]
\label{prop:framework}
For a certified exclusion test $T$ and a threshold $U\ge r_k(\mathbf{q})$, the
$k$ vectors with the smallest exact distances among the vectors that $T$ keeps
form $\mathcal{N}_k(\mathbf{q})$.
\end{proposition}

\begin{proof}
Every $\mathbf{x}\in\mathcal{N}_k(\mathbf{q})$ satisfies
$\|\mathbf{q}-\mathbf{x}\|\le r_k(\mathbf{q})\le U$ and is therefore kept by
Definition~\ref{def:exclusion}. The $k$ smallest exact distances among a superset
of $\mathcal{N}_k(\mathbf{q})$ are the $k$ smallest over $\mathcal{X}$.
\end{proof}

\subsection{Lifted Lower Bound}
\label{sec:lift}

Proposition~\ref{prop:framework} leaves open which certified exclusion test to
use. \certex\ needs a test that reads few bytes, because the retina pass reads
every database vector and its cost grows with the number of bytes per vector.
\certex\ therefore keeps a few principal coordinates of each vector and one more
coordinate that summarizes the rest, which together still give a lower bound on
the distance. Specifically, consider a center
$\mathbf{c}\in\R^D$ and a matrix $\mathbf{V}\in\R^{m\times D}$ whose orthonormal
rows are the $m$ leading principal directions of a sample of $\mathcal{X}$. We
call $m$ the \emph{width} of the retina. The \emph{lifted projection} of a vector
keeps its $m$ principal coordinates and summarizes
the remaining $D-m$ coordinates by their norm:
\begin{equation}
\lift(\mathbf{x}) \coloneqq \begin{bmatrix} \mathbf{V}(\mathbf{x}-\mathbf{c}) \\
\|(\mathbf{I}-\mathbf{V}^{\top}\mathbf{V})(\mathbf{x}-\mathbf{c})\| \end{bmatrix}
\in\R^{m+1}.
\label{eq:lift}
\end{equation}

\begin{proposition}[Lower bound]
\label{prop:lb}
For all $\mathbf{q},\mathbf{x}\in\R^D$,
$\|\lift(\mathbf{q})-\lift(\mathbf{x})\|\le\|\mathbf{q}-\mathbf{x}\|$.
\end{proposition}

\begin{proof}
The entry $(i,j)$ of $\mathbf{V}\mathbf{V}^{\top}$ is the inner product of rows $i$
and $j$ of $\mathbf{V}$, which is 1 for $i=j$ and 0 otherwise because the rows are
orthonormal. Hence $\mathbf{V}\mathbf{V}^{\top}=\mathbf{I}_m$. We write
$\mathbf{\Pi}\coloneqq\mathbf{V}^{\top}\mathbf{V}$ and prove the claim in five steps.

Step 1 splits the centered vectors. We define
$\mathbf{u}\coloneqq\mathbf{q}-\mathbf{c}$, $\mathbf{v}\coloneqq\mathbf{x}-\mathbf{c}$,
$\mathbf{u}_1\coloneqq\mathbf{\Pi}\mathbf{u}$, $\mathbf{u}_2\coloneqq\mathbf{u}-\mathbf{u}_1$,
$\mathbf{v}_1\coloneqq\mathbf{\Pi}\mathbf{v}$, and
$\mathbf{v}_2\coloneqq\mathbf{v}-\mathbf{v}_1$. Two facts follow. First, with
$\mathbf{a}\coloneqq\mathbf{V}(\mathbf{u}-\mathbf{v})\in\R^m$,
\begin{equation}
\mathbf{u}_1-\mathbf{v}_1=\mathbf{V}^{\top}\mathbf{V}\mathbf{u}-\mathbf{V}^{\top}\mathbf{V}\mathbf{v}
=\mathbf{V}^{\top}\mathbf{V}(\mathbf{u}-\mathbf{v})=\mathbf{V}^{\top}\mathbf{a}.
\label{eq:lb-a}
\end{equation}
Second, since $\mathbf{V}\mathbf{V}^{\top}=\mathbf{I}_m$,
\begin{equation}
\mathbf{V}\mathbf{u}_2=\mathbf{V}\mathbf{u}-\mathbf{V}\mathbf{V}^{\top}\mathbf{V}\mathbf{u}
=\mathbf{V}\mathbf{u}-\mathbf{V}\mathbf{u}=\mathbf{0}
\label{eq:lb-b}
\end{equation}
and likewise $\mathbf{V}\mathbf{v}_2=\mathbf{0}$.

Step 2 writes the lifted projections in these parts. Eq.~\eqref{eq:lift} applied
to $\mathbf{q}$ reads
$\lift(\mathbf{q})=[\mathbf{V}\mathbf{u};\|(\mathbf{I}-\mathbf{V}^{\top}\mathbf{V})\mathbf{u}\|]$,
because $\mathbf{q}-\mathbf{c}=\mathbf{u}$. Its first block is
$\mathbf{V}\mathbf{u}=\mathbf{V}(\mathbf{u}_1+\mathbf{u}_2)=\mathbf{V}\mathbf{u}_1+\mathbf{V}\mathbf{u}_2=\mathbf{V}\mathbf{u}_1$
by Eq.~\eqref{eq:lb-b}. Its last coordinate is $\|\mathbf{u}_2\|$, because
$(\mathbf{I}-\mathbf{V}^{\top}\mathbf{V})\mathbf{u}=\mathbf{u}-\mathbf{\Pi}\mathbf{u}=\mathbf{u}-\mathbf{u}_1=\mathbf{u}_2$.
The same computation for $\mathbf{x}$ gives
\begin{equation*}
\lift(\mathbf{q})=\begin{bmatrix}\mathbf{V}\mathbf{u}_1\\\|\mathbf{u}_2\|\end{bmatrix},
\qquad
\lift(\mathbf{x})=\begin{bmatrix}\mathbf{V}\mathbf{v}_1\\\|\mathbf{v}_2\|\end{bmatrix},
\end{equation*}
and subtracting the two gives
\begin{equation*}
\lift(\mathbf{q})-\lift(\mathbf{x})=\begin{bmatrix}\mathbf{V}(\mathbf{u}_1-\mathbf{v}_1)\\\|\mathbf{u}_2\|-\|\mathbf{v}_2\|\end{bmatrix}.
\end{equation*}
The squared norm of a stacked vector is the sum of the squared norms of its
blocks, so that
\begin{equation}
\|\lift(\mathbf{q})-\lift(\mathbf{x})\|^2
=\|\mathbf{V}(\mathbf{u}_1-\mathbf{v}_1)\|^2+(\|\mathbf{u}_2\|-\|\mathbf{v}_2\|)^2.
\label{eq:lb-i}
\end{equation}

Step 3 shows that $\mathbf{V}$ keeps the norm of $\mathbf{u}_1-\mathbf{v}_1$. By Eq.~\eqref{eq:lb-a},
$\mathbf{V}(\mathbf{u}_1-\mathbf{v}_1)=\mathbf{V}\mathbf{V}^{\top}\mathbf{a}=\mathbf{a}$ and
$\|\mathbf{u}_1-\mathbf{v}_1\|^2=\mathbf{a}^{\top}\mathbf{V}\mathbf{V}^{\top}\mathbf{a}=\|\mathbf{a}\|^2$.
Hence
\begin{equation}
\|\mathbf{u}_1-\mathbf{v}_1\|^2=\|\mathbf{V}(\mathbf{u}_1-\mathbf{v}_1)\|^2.
\label{eq:lb-ii}
\end{equation}

Step 4 applies the Pythagorean theorem. The center cancels in the difference,
$\mathbf{q}-\mathbf{x}=(\mathbf{q}-\mathbf{c})-(\mathbf{x}-\mathbf{c})=\mathbf{u}-\mathbf{v}$,
which splits as $\mathbf{y}+\mathbf{z}$ with
$\mathbf{y}\coloneqq\mathbf{u}_1-\mathbf{v}_1$ and
$\mathbf{z}\coloneqq\mathbf{u}_2-\mathbf{v}_2$. By Eqs.~\eqref{eq:lb-a} and~\eqref{eq:lb-b}, the two parts are
orthogonal:
\begin{equation*}
\langle\mathbf{y},\mathbf{z}\rangle=\langle\mathbf{V}^{\top}\mathbf{a},\mathbf{z}\rangle
=\mathbf{a}^{\top}\mathbf{V}\mathbf{z}
=\mathbf{a}^{\top}(\mathbf{V}\mathbf{u}_2-\mathbf{V}\mathbf{v}_2)=0.
\end{equation*}
Expanding $\|\mathbf{y}+\mathbf{z}\|^2=\|\mathbf{y}\|^2+2\langle\mathbf{y},\mathbf{z}\rangle+\|\mathbf{z}\|^2$
therefore gives
\begin{equation}
\|\mathbf{q}-\mathbf{x}\|^2=\|\mathbf{u}_1-\mathbf{v}_1\|^2+\|\mathbf{u}_2-\mathbf{v}_2\|^2.
\label{eq:lb-iii}
\end{equation}

Step 5 applies the reverse triangle inequality. The triangle inequality gives
$\|\mathbf{u}_2\|\le\|\mathbf{u}_2-\mathbf{v}_2\|+\|\mathbf{v}_2\|$ and
$\|\mathbf{v}_2\|\le\|\mathbf{u}_2-\mathbf{v}_2\|+\|\mathbf{u}_2\|$, so that
$\|\mathbf{u}_2-\mathbf{v}_2\|\ge|\|\mathbf{u}_2\|-\|\mathbf{v}_2\||$. Squaring both
nonnegative sides gives
\begin{equation}
\|\mathbf{u}_2-\mathbf{v}_2\|^2\ge(\|\mathbf{u}_2\|-\|\mathbf{v}_2\|)^2.
\label{eq:lb-iv}
\end{equation}

Combining the four relations,
\begin{equation*}
\begin{split}
\|\mathbf{q}-\mathbf{x}\|^2
&\overset{\eqref{eq:lb-iii}}{=}\|\mathbf{u}_1-\mathbf{v}_1\|^2+\|\mathbf{u}_2-\mathbf{v}_2\|^2\\
&\overset{\eqref{eq:lb-iv}}{\ge}\|\mathbf{u}_1-\mathbf{v}_1\|^2+(\|\mathbf{u}_2\|-\|\mathbf{v}_2\|)^2\\
&\overset{\eqref{eq:lb-ii}}{=}\|\mathbf{V}(\mathbf{u}_1-\mathbf{v}_1)\|^2+(\|\mathbf{u}_2\|-\|\mathbf{v}_2\|)^2\\
&\overset{\eqref{eq:lb-i}}{=}\|\lift(\mathbf{q})-\lift(\mathbf{x})\|^2.
\end{split}
\end{equation*}
Taking the square root of both nonnegative sides proves the claim.
\end{proof}

Proposition~\ref{prop:lb} turns the lifted projection into an exclusion test.
Specifically, if $\|\lift(\mathbf{q})-\lift(\mathbf{x})\|>U$, then $\|\mathbf{q}-\mathbf{x}\|>U$ as
well, so that excluding $\mathbf{x}$ never removes a vector within $U$. However,
the test must take a form that tensor cores evaluate quickly. This form is an
inner product, because tensor cores multiply matrices and thus compute the inner
products of many pairs of vectors at once. We therefore rewrite the test in terms
of one inner product per pair. To this end, expanding the squared norm of a
difference with $a(\mathbf{x})\coloneqq\|\lift(\mathbf{x})\|^2$ gives
\begin{equation}
\|\lift(\mathbf{q})-\lift(\mathbf{x})\|^2 = a(\mathbf{q})+a(\mathbf{x})
-2\langle\lift(\mathbf{q}),\lift(\mathbf{x})\rangle .
\label{eq:inner}
\end{equation}
Since both sides of the test $\|\lift(\mathbf{q})-\lift(\mathbf{x})\|\le U$ are
nonnegative, the test holds exactly when the squared test holds. By
Eq.~\eqref{eq:inner}, the squared test reads
\begin{equation*}
a(\mathbf{q})+a(\mathbf{x})-2\langle\lift(\mathbf{q}),\lift(\mathbf{x})\rangle\le U^2.
\end{equation*}
Moving the inner product to the right-hand side and $U^2$ to the left-hand side
gives $a(\mathbf{q})+a(\mathbf{x})-U^2\le2\langle\lift(\mathbf{q}),\lift(\mathbf{x})\rangle$.
Dividing by 2 and moving $a(\mathbf{x})/2$ to the right-hand side finally gives
the equivalent test
\begin{equation}
\langle\lift(\mathbf{q}),\lift(\mathbf{x})\rangle-\frac{a(\mathbf{x})}{2}
\ge\frac{a(\mathbf{q})-U^2}{2}.
\label{eq:lbtest}
\end{equation}
In this form, the right-hand side depends only on the query, so that \certex\
computes it once per query. Similarly, on the left-hand side, $a(\mathbf{x})$ depends only on the database
vector and is computed once when the retina is built. The test of one pair
therefore costs one inner product and one comparison, so that one matrix product
on the tensor cores tests all pairs.

\subsection{Certified Half-Precision Retina}
\label{sec:retina}

To evaluate Eq.~\eqref{eq:lbtest} for all vectors in one tensor-core pass,
\certex\ stores each lifted projection in half precision inside its retina vector.
However, rounding to half precision may underestimate the inner product, which
overestimates the bound and can discard a true neighbor. To bound this effect,
denote by $\tilde{\lift}(\cdot)$ the rounding of $\lift(\cdot)$ to the nearest
half-precision numbers and by
$\delta(\cdot)\coloneqq\|\lift(\cdot)-\tilde{\lift}(\cdot)\|$ its rounding error.
Splitting the inner product into its rounded part and two error terms and
applying the Cauchy--Schwarz inequality to the error terms gives
\begin{equation}
\begin{split}
\langle\lift(\mathbf{q}),\lift(\mathbf{x})\rangle \le{}&
\langle\tilde{\lift}(\mathbf{q}),\tilde{\lift}(\mathbf{x})\rangle\\
&+\delta(\mathbf{q})\|\lift(\mathbf{x})\|+\|\tilde{\lift}(\mathbf{q})\|\delta(\mathbf{x}).
\end{split}
\label{eq:cert}
\end{equation}
Eq.~\eqref{eq:cert} follows from the exact split
\begin{equation*}
\begin{split}
\langle\lift(\mathbf{q}),\lift(\mathbf{x})\rangle
={}&\langle\tilde{\lift}(\mathbf{q}),\tilde{\lift}(\mathbf{x})\rangle
+\langle\lift(\mathbf{q})-\tilde{\lift}(\mathbf{q}),\lift(\mathbf{x})\rangle\\
&+\langle\tilde{\lift}(\mathbf{q}),\lift(\mathbf{x})-\tilde{\lift}(\mathbf{x})\rangle,
\end{split}
\end{equation*}
since the Cauchy--Schwarz inequality bounds the last two terms by
$\delta(\mathbf{q})\|\lift(\mathbf{x})\|$ and
$\|\tilde{\lift}(\mathbf{q})\|\delta(\mathbf{x})$.
A second error arises in the tensor cores, which multiply half-precision
operands exactly but accumulate the products in fp32. This accumulation adds an
error of at most $\gamma$ times the sum of the absolute products, where $\gamma$
depends only on the padded length $L$ of the retina
vectors~\cite{higham2002accuracy,blanchard2020tensorcore}.

\certex\ folds both error terms into the same dot product, in the manner of
floating-point filters~\cite{shewchuk1997adaptive}. Specifically, each database
retina vector $\hat{\mathbf{x}}$ appends four coordinates to
$\tilde{\lift}(\mathbf{x})$, which form its certificate: upper bounds
$\overline{n}_{\mathbf{x}}$ on $\|\lift(\mathbf{x})\|$ and $\overline{e}_{\mathbf{x}}$ on
the rounding error $\delta(\mathbf{x})$ plus an accumulation term, and a high part
$h_{\mathbf{x}}$ and a low part $l_{\mathbf{x}}$ of $a(\mathbf{x})/2$ (the bars mark
upper bounds). Correspondingly, the retina vector $\hat{\mathbf{q}}$ of a query
appends an upper bound $\overline{e}_{\mathbf{q}}$ on $\delta(\mathbf{q})$, an upper
bound $\overline{n}_{\mathbf{q}}$ on $\|\tilde{\lift}(\mathbf{q})\|$, a coefficient
$1-\varepsilon$ with $\varepsilon$ slightly larger than $\gamma$, and a one.
The two retina vectors thus read
\begin{equation*}
\setlength{\arraycolsep}{5pt}
\begin{array}{rccccc}
\hat{\mathbf{x}}^{\top}={}&[\,\tilde{\lift}(\mathbf{x})^{\top},&\overline{n}_{\mathbf{x}},&\overline{e}_{\mathbf{x}},&-h_{\mathbf{x}},&-l_{\mathbf{x}}\,],\\[3pt]
\hat{\mathbf{q}}^{\top}={}&[\,\tilde{\lift}(\mathbf{q})^{\top},&\overline{e}_{\mathbf{q}},&\overline{n}_{\mathbf{q}},&1-\varepsilon,&1\,],
\end{array}
\end{equation*}
where the columns show which coordinates multiply each other. Denote by $P$ the
exact inner product of $\hat{\mathbf{q}}$ and $\hat{\mathbf{x}}$. By construction,
\begin{equation}
P=\langle\tilde{\lift}(\mathbf{q}),\tilde{\lift}(\mathbf{x})\rangle
+\overline{e}_{\mathbf{q}}\overline{n}_{\mathbf{x}}
+\overline{n}_{\mathbf{q}}\overline{e}_{\mathbf{x}}-(1-\varepsilon)h_{\mathbf{x}}-l_{\mathbf{x}}.
\label{eq:P}
\end{equation}
In Eq.~\eqref{eq:P}, the products $\overline{e}_{\mathbf{q}}\overline{n}_{\mathbf{x}}$
and $\overline{n}_{\mathbf{q}}\overline{e}_{\mathbf{x}}$ bound the two error terms of
Eq.~\eqref{eq:cert} from above, whereas $h_{\mathbf{x}}+l_{\mathbf{x}}$ represents
$a(\mathbf{x})/2$. Instead of $P$, the tensor cores return the value
$\widehat{P}$, which results from exact products of the half-precision coordinates
and an accumulation of these products in fp32. The test of \certex\ compares this
value with a threshold: \certex\ keeps a vector if and only if $\widehat{P}$
satisfies
\begin{equation}
\widehat{P} \ge \theta(\mathbf{q}) \coloneqq \frac{a(\mathbf{q})-U^2}{2}-\mu(\mathbf{q}),
\label{eq:test}
\end{equation}
where the \emph{margin}
$\mu(\mathbf{q})\coloneqq\gamma(\overline{e}_{\mathbf{q}}M_n+\overline{n}_{\mathbf{q}}M_e+M_l)$
uses the database maxima $M_n$, $M_e$, and $M_l$ of $\overline{n}_{\mathbf{x}}$,
$\overline{e}_{\mathbf{x}}$, and $|l_{\mathbf{x}}|$.

The proof of Proposition~\ref{prop:cert} uses every coordinate of the retina vectors, which we now define. We assume that the data are scaled once so that every lifted coordinate and every $a(\mathbf{x})/2$ lies in the normal range of half precision. This assumption loses no generality, because a common scale multiplies all distances by the same factor and therefore leaves the neighbors unchanged.

The coordinates that carry the error terms are themselves stored in half
precision and must still bound norms and errors from above. We therefore write
$\operatorname{fl}(v)$ for the rounding of $v$ to the nearest half-precision number
and $\operatorname{up}(v)\coloneqq\operatorname{fl}((1+2^{-10})v)$ for rounding up.
Since rounding to nearest changes a value in the normal range by at most the
relative $2^{-11}$, $\operatorname{up}(v)\ge(1+2^{-10})(1-2^{-11})v>v$ for every
$v>0$. Although the norms inside $\operatorname{up}$ are computed in fp32, their
relative error stays far below this margin. With these two operators, \certex\
computes for each database vector $\mathbf{x}$
\begin{align*}
\nu_{\mathbf{x}}&\coloneqq(1+2^{-10})\max(\|\lift(\mathbf{x})\|,\|\tilde{\lift}(\mathbf{x})\|),\\
\overline{n}_{\mathbf{x}}&\coloneqq\operatorname{up}(\|\lift(\mathbf{x})\|),\\
\overline{e}_{\mathbf{x}}&\coloneqq\operatorname{up}(\delta(\mathbf{x})+\gamma\nu_{\mathbf{x}}),\\
h_{\mathbf{x}}&\coloneqq\operatorname{fl}(a(\mathbf{x})/2),\\
l_{\mathbf{x}}&\coloneqq\operatorname{fl}(a(\mathbf{x})/2-h_{\mathbf{x}}).
\end{align*}
Here, the auxiliary value $\nu_{\mathbf{x}}$ bounds both $\|\lift(\mathbf{x})\|$ and
$\|\tilde{\lift}(\mathbf{x})\|$ from above. By the property of $\operatorname{up}$,
$\overline{n}_{\mathbf{x}}\ge\|\lift(\mathbf{x})\|$ bounds the norm, and
$\overline{e}_{\mathbf{x}}\ge\delta(\mathbf{x})+\gamma\nu_{\mathbf{x}}$ bounds the
rounding error $\delta(\mathbf{x})$ together with the term $\gamma\nu_{\mathbf{x}}$.
This term absorbs the accumulation error of the products of the first $m+1$
coordinates (Step 3 of the proof below). Finally, the pair
$h_{\mathbf{x}},l_{\mathbf{x}}$ stores $a(\mathbf{x})/2$ as a high and a low
half-precision part. Correspondingly, for a query $\mathbf{q}$, \certex\ computes
$\overline{e}_{\mathbf{q}}\coloneqq\operatorname{up}(\delta(\mathbf{q}))$ and
$\overline{n}_{\mathbf{q}}\coloneqq\operatorname{up}(\|\tilde{\lift}(\mathbf{q})\|)$.
The coefficient $1-\varepsilon$ of $\hat{\mathbf{q}}$ uses the smallest multiple
$\varepsilon$ of $2^{-11}$ that exceeds $\gamma$, so that $1-\varepsilon$ is exact
in half precision.

\begin{proposition}[No false dismissal]
\label{prop:cert}
If $\|\lift(\mathbf{q})-\lift(\mathbf{x})\|\le U$, the computed product satisfies
Eq.~\eqref{eq:test}, so that \certex\ keeps $\mathbf{x}$.
\end{proposition}

\begin{proof}
We bound the exact product $P$ from below, then the accumulation error, and
combine the two bounds in four steps.

Step 1 bounds $P$ from below. The coordinates satisfy
$\overline{e}_{\mathbf{q}}\ge\delta(\mathbf{q})$,
$\overline{n}_{\mathbf{x}}\ge\|\lift(\mathbf{x})\|$,
$\overline{n}_{\mathbf{q}}\ge\|\tilde{\lift}(\mathbf{q})\|$, and
$\overline{e}_{\mathbf{x}}\ge\delta(\mathbf{x})+\gamma\nu_{\mathbf{x}}$, so that
\begin{equation*}
\overline{e}_{\mathbf{q}}\overline{n}_{\mathbf{x}}+\overline{n}_{\mathbf{q}}\overline{e}_{\mathbf{x}}
\ge\delta(\mathbf{q})\|\lift(\mathbf{x})\|+\|\tilde{\lift}(\mathbf{q})\|\delta(\mathbf{x})
+\gamma\overline{n}_{\mathbf{q}}\nu_{\mathbf{x}}.
\end{equation*}
Adding $\langle\tilde{\lift}(\mathbf{q}),\tilde{\lift}(\mathbf{x})\rangle$ to both sides
and applying Eq.~\eqref{eq:cert} bounds the first three terms of Eq.~\eqref{eq:P}
from below by $\langle\lift(\mathbf{q}),\lift(\mathbf{x})\rangle+\gamma\overline{n}_{\mathbf{q}}\nu_{\mathbf{x}}$.
The last two terms of Eq.~\eqref{eq:P} equal
$-(h_{\mathbf{x}}+l_{\mathbf{x}})+\varepsilon h_{\mathbf{x}}$, which is at least
$-a(\mathbf{x})/2-\sigma_{\mathbf{x}}+\varepsilon h_{\mathbf{x}}$, where
$\sigma_{\mathbf{x}}\coloneqq|a(\mathbf{x})/2-h_{\mathbf{x}}-l_{\mathbf{x}}|$ is the error
of the split. Hence
\begin{equation}
P\ge\langle\lift(\mathbf{q}),\lift(\mathbf{x})\rangle-\frac{a(\mathbf{x})}{2}
+\gamma\overline{n}_{\mathbf{q}}\nu_{\mathbf{x}}+\varepsilon h_{\mathbf{x}}-\sigma_{\mathbf{x}}.
\label{eq:ce-i}
\end{equation}

Step 2 bounds the split error. Rounding to nearest gives
$|a(\mathbf{x})/2-h_{\mathbf{x}}|\le2^{-11}a(\mathbf{x})/2$ and then
$\sigma_{\mathbf{x}}\le2^{-11}|a(\mathbf{x})/2-h_{\mathbf{x}}|\le2^{-22}a(\mathbf{x})/2$.
Since $h_{\mathbf{x}}\ge(1-2^{-11})a(\mathbf{x})/2$, this gives
$\sigma_{\mathbf{x}}\le2^{-21}h_{\mathbf{x}}$. The padded length $L$ is a multiple
of 64, so that $\gamma=16Lu$ with $u=2^{-24}$ is a multiple of $2^{-14}$, and so is
$\varepsilon$. Their difference is positive and therefore at least $2^{-14}$, which
gives
\begin{equation}
(\varepsilon-\gamma)h_{\mathbf{x}}\ge\sigma_{\mathbf{x}}.
\label{eq:ce-ii}
\end{equation}

Step 3 bounds the accumulation error. Since the tensor cores multiply exactly and
accumulate in fp32,
\begin{equation*}
|\widehat{P}-P|\le\gamma\sum_i|\hat{q}_i\hat{x}_i|,
\end{equation*}
where $\hat{q}_i$ and $\hat{x}_i$ denote the $i$-th coordinates of
$\hat{\mathbf{q}}$ and $\hat{\mathbf{x}}$. The
Cauchy--Schwarz inequality bounds the sum over the first $m+1$ coordinates by
$\|\tilde{\lift}(\mathbf{q})\|\|\tilde{\lift}(\mathbf{x})\|\le\overline{n}_{\mathbf{q}}\nu_{\mathbf{x}}$.
The remaining four products are $\overline{e}_{\mathbf{q}}\overline{n}_{\mathbf{x}}$,
$\overline{n}_{\mathbf{q}}\overline{e}_{\mathbf{x}}$,
$(1-\varepsilon)h_{\mathbf{x}}\le h_{\mathbf{x}}$, and $|l_{\mathbf{x}}|$. Replacing
$\overline{n}_{\mathbf{x}}$, $\overline{e}_{\mathbf{x}}$, and $|l_{\mathbf{x}}|$ by their
database maxima gives
\begin{equation}
|\widehat{P}-P|\le\gamma\overline{n}_{\mathbf{q}}\nu_{\mathbf{x}}+\gamma h_{\mathbf{x}}+\mu(\mathbf{q}).
\label{eq:ce-iii}
\end{equation}

Step 4 combines the bounds. Subtracting Eq.~\eqref{eq:ce-iii} from Eq.~\eqref{eq:ce-i} and applying
Eq.~\eqref{eq:ce-ii},
\begin{equation*}
\begin{split}
\widehat{P}&\ge\langle\lift(\mathbf{q}),\lift(\mathbf{x})\rangle-\frac{a(\mathbf{x})}{2}
+(\varepsilon-\gamma)h_{\mathbf{x}}-\sigma_{\mathbf{x}}-\mu(\mathbf{q})\\
&\ge\langle\lift(\mathbf{q}),\lift(\mathbf{x})\rangle-\frac{a(\mathbf{x})}{2}-\mu(\mathbf{q}).
\end{split}
\end{equation*}
By the premise and Eq.~\eqref{eq:lbtest},
$\langle\lift(\mathbf{q}),\lift(\mathbf{x})\rangle-a(\mathbf{x})/2\ge(a(\mathbf{q})-U^2)/2$.
Hence $\widehat{P}\ge\theta(\mathbf{q})$.
\end{proof}

In the implementation, \certex\ sets $\gamma\coloneqq16Lu$, where $u=2^{-24}$ is
the unit roundoff of fp32. This constant exceeds the worst-case constant that
Blanchard et al.~\cite{blanchard2020tensorcore} derive for the accumulation of
tensor cores, and it also covers an accumulator that truncates instead of
rounding. In addition, \certex\ inflates $U^2$ by the relative factor
$1+10^{-4}$. The threshold $\theta(\mathbf{q})$ is evaluated in fp64 and rounded
down to fp32, so that neither rounding shrinks the kept set. Finally, the retina
pass reads the database retina vectors once per batch of queries and writes out
only the pairs that satisfy Eq.~\eqref{eq:test} (Appendix~\ref{A-sec:impl}). The same certificate
extends to the 8-bit integer tensor cores, whose retina halves the device memory
per vector at the same width while every answer stays exact (Section~\ref{sec:int8}).

\subsection{Foveal Verification and Exactness}
\label{sec:fovea}

The fovea needs the threshold $U$ before the retina pass and the exact distances
of the survivors after it. Before the pass, \certex\ samples every $s$-th retina
vector, for a stride $s$, and scores the sampled vectors by the same computed
product $\widehat{P}$ as the retina pass. Among them, the sampled vectors with the
largest products become the \emph{anchors} $\mathcal{A}(\mathbf{q})$. The threshold
$U$ is then the $k$-th
smallest exact distance among the anchors. Because the anchors are at least $k$
distinct vectors of $\mathcal{X}$, $U\ge r_k(\mathbf{q})$ holds for every query.
After the pass, the fovea computes the exact distance of each survivor from the
full-precision difference $\mathbf{q}-\mathbf{x}$ and returns the $k$ smallest
per query.

\begin{proposition}[Exactness]
\label{prop:exact}
\certex\ returns $\mathcal{N}_k(\mathbf{q})$ for every query $\mathbf{q}$.
\end{proposition}

\begin{proof}
Every $\mathbf{x}$ with $\|\mathbf{q}-\mathbf{x}\|\le U$ satisfies
$\|\lift(\mathbf{q})-\lift(\mathbf{x})\|\le U$ by Proposition~\ref{prop:lb} and
survives by Proposition~\ref{prop:cert}. The retina is therefore a certified
exclusion test. Since $U\ge r_k(\mathbf{q})$, Proposition~\ref{prop:framework}
yields the claim.
\end{proof}

The anchors, the rotation $\mathbf{V}$, and the width $m$ therefore influence
only the cost of \certex. For instance, a loose threshold admits more survivors
and more exact distances but cannot remove a true neighbor. To quantify this
cost, Proposition~\ref{prop:surv} bounds the number of survivors. For this
purpose, we write
$\rho(\mathbf{x})\coloneqq\|(\mathbf{I}-\mathbf{V}^{\top}\mathbf{V})(\mathbf{x}-\mathbf{c})\|$
for the last coordinate of $\lift(\mathbf{x})$, which is the norm of the component
of $\mathbf{x}-\mathbf{c}$ that is orthogonal to the rows of $\mathbf{V}$.
Similarly, we denote by $\eta(\mathbf{q})$ the \emph{slack} that $\mu(\mathbf{q})$
and the rounded-up coordinates add to the squared threshold, which is of the order
of the unit roundoff times the squared norms.

\begin{proposition}[Survivor ball]
\label{prop:surv}
Every vector that the retina keeps for the threshold $U$ satisfies
$\|\mathbf{q}-\mathbf{x}\|^2\le U^2+\eta(\mathbf{q})+4\rho(\mathbf{q})\rho(\mathbf{x})$.
\end{proposition}

\begin{proof}
The proof reuses the notation of the proof of Proposition~\ref{prop:lb}. There, $\mathbf{u}=\mathbf{q}-\mathbf{c}$ and
$\mathbf{v}=\mathbf{x}-\mathbf{c}$ split as $\mathbf{u}=\mathbf{u}_1+\mathbf{u}_2$ and
$\mathbf{v}=\mathbf{v}_1+\mathbf{v}_2$, where $\mathbf{u}_1,\mathbf{v}_1$ lie in the row
space of $\mathbf{V}$ and $\mathbf{u}_2,\mathbf{v}_2$ are orthogonal to it. The norms
of $\mathbf{u}_2$ and $\mathbf{v}_2$ are the last coordinates $\rho(\mathbf{q})$ and
$\rho(\mathbf{x})$ of the lifted projections. The proof proceeds in three steps.

First, the retina keeps $\mathbf{x}$ only if its lifted distance is small. By the
definition of the slack $\eta(\mathbf{q})$, every kept vector satisfies
\begin{equation}
\|\lift(\mathbf{q})-\lift(\mathbf{x})\|^2\le U^2+\eta(\mathbf{q}).
\label{eq:kept}
\end{equation}

Second, the squared distance exceeds the squared lifted distance by a term that
involves only $\mathbf{u}_2$ and $\mathbf{v}_2$. Eqs.~\eqref{eq:lb-iii}, \eqref{eq:lb-ii}, and~\eqref{eq:lb-i}
give
$\|\mathbf{q}-\mathbf{x}\|^2=\|\mathbf{u}_1-\mathbf{v}_1\|^2+\|\mathbf{u}_2-\mathbf{v}_2\|^2$
and
$\|\lift(\mathbf{q})-\lift(\mathbf{x})\|^2=\|\mathbf{u}_1-\mathbf{v}_1\|^2+(\|\mathbf{u}_2\|-\|\mathbf{v}_2\|)^2$.
Subtracting the second equation from the first and expanding both squares,
\begin{equation}
\begin{split}
&\|\mathbf{q}-\mathbf{x}\|^2-\|\lift(\mathbf{q})-\lift(\mathbf{x})\|^2\\
&\quad=\|\mathbf{u}_2-\mathbf{v}_2\|^2-(\|\mathbf{u}_2\|-\|\mathbf{v}_2\|)^2\\
&\quad=\bigl(\|\mathbf{u}_2\|^2-2\langle\mathbf{u}_2,\mathbf{v}_2\rangle+\|\mathbf{v}_2\|^2\bigr)\\
&\qquad-\bigl(\|\mathbf{u}_2\|^2-2\|\mathbf{u}_2\|\|\mathbf{v}_2\|+\|\mathbf{v}_2\|^2\bigr)\\
&\quad=2\bigl(\|\mathbf{u}_2\|\|\mathbf{v}_2\|-\langle\mathbf{u}_2,\mathbf{v}_2\rangle\bigr).
\end{split}
\label{eq:excess}
\end{equation}

Third, the Cauchy--Schwarz inequality bounds this excess. It gives
$-\langle\mathbf{u}_2,\mathbf{v}_2\rangle\le\|\mathbf{u}_2\|\|\mathbf{v}_2\|$, so that the
right-hand side of Eq.~\eqref{eq:excess} is at most
$4\|\mathbf{u}_2\|\|\mathbf{v}_2\|=4\rho(\mathbf{q})\rho(\mathbf{x})$. Adding this bound
to Eq.~\eqref{eq:kept} yields
$\|\mathbf{q}-\mathbf{x}\|^2\le U^2+\eta(\mathbf{q})+4\rho(\mathbf{q})\rho(\mathbf{x})$,
which is the claim.
\end{proof}

The survivors therefore lie between two balls.
Specifically, the inner ball has radius $r_k(\mathbf{q})$, and its vectors always
survive. In contrast, the outer ball is slightly larger, by an excess that grows with $\rho(\mathbf{q})$ and
$\rho(\mathbf{x})$. Because the
anchors are the best of every $s$-th vector, $U$ approximates the distance to the
$(ks)$-th nearest neighbor. The number of survivors thus approaches $ks$ when the
width captures most of the variance of the data. A batch of $B$ queries then costs
$BN(m+5)$ half-precision multiply-adds plus about $ks$ exact distances per query,
whereas exhaustive search costs $BND$ full-precision multiply-adds and a
selection among all $N$ candidates.

\begin{table*}[t]
\centering
\caption{Exact search ($k=10$, one H100): QPS of \certex, inexact queries of cuVS fp32, and speedups of \certex\ over both baselines.}
\label{tab:main}
\footnotesize
\setlength{\tabcolsep}{2pt}
\begin{tabular*}{\linewidth}{@{\extracolsep{\fill}}lrrrrrrr|lrrrrrrr}
\toprule
& & & & \certex & \multicolumn{2}{c}{cuVS fp32} & fp16 + rerank & & & & & \certex & \multicolumn{2}{c}{cuVS fp32} & fp16 + rerank \\ \cmidrule(lr){6-7}\cmidrule(lr){14-15}
Data & $N$ & $D$ & $m$ & QPS & inexact & speedup & speedup & ImageNet & $N$ & $D$ & $m$ & QPS & inexact & speedup & speedup \\
\midrule
SIFT1M & 1M & 128 & 123 & 719.6k & 0 & 11.0$\times$ & 8.0$\times$ & CLIP-B & 1.28M & 512 & 507 & 195.1k & 1 & 7.5$\times$ & 3.4$\times$ \\
GIST1M & 1M & 960 & 507 & 157.6k & 0 & 7.7$\times$ & 2.4$\times$ & CLIP-L & 1.28M & 768 & 507 & 179.6k & 2 & 9.1$\times$ & 3.5$\times$ \\
GloVe & 1.18M & 100 & 99 & 642.8k & 1 & 11.0$\times$ & 8.7$\times$ & SigLIP & 1.28M & 1152 & 635 & 157.7k & 4 & 10.9$\times$ & 3.6$\times$ \\
Deep-image & 9.99M & 96 & 95 & 104.1k & 3 & 14.3$\times$ & 11.2$\times$ & DINOv2-B & 1.28M & 768 & 635 & 152.2k & 0 & 7.7$\times$ & 3.0$\times$ \\
BIGANN-10M & 10M & 128 & 123 & 103.3k & 0 & 15.5$\times$ & 11.1$\times$ & DINOv2-L & 1.28M & 1024 & 891 & 121.0k & 2 & 7.7$\times$ & 2.6$\times$ \\
BIGANN-100M & 100M & 128 & 59 & 12.0k & 0 & 18.0$\times$ & 12.9$\times$ & DINOv3-L & 1.28M & 1024 & 763 & 131.3k & 1 & 8.4$\times$ & 2.9$\times$ \\
DataComp img. & 12.79M & 768 & 507 & 14.1k & 108 & 7.1$\times$ & 2.3$\times$ & DINOv3-7B & 1.28M & 4096 & 3579 & 36.2k & 2 & 7.9$\times$ & 1.8$\times$ \\
DataComp txt. & 12.79M & 768 & 635 & 18.1k & 1 & 9.1$\times$ & 2.9$\times$ & CLIP text & 1.28M & 512 & 507 & 192.5k & 0 & 7.4$\times$ & 3.4$\times$ \\
\bottomrule
\end{tabular*}
\end{table*}

\subsection{Certified Exclusion at Any Threshold}
\label{sec:join}

Certified exclusion holds for every threshold, so that \certex\ answers any query
that fixes its threshold in advance exactly and without the anchors.
Specifically, Definition~\ref{def:exclusion} and Proposition~\ref{prop:cert} hold
for every $U\ge0$, so that the test of the retina is certified whatever threshold
a query uses. The anchors of Section~\ref{sec:fovea} have a different role: they
choose the threshold $U$ when a query does not fix it. For instance, kNN search
does not fix its threshold, because Proposition~\ref{prop:framework} requires
$U\ge r_k(\mathbf{q})$ and the distance $r_k(\mathbf{q})$ becomes known only after
the search. In contrast, a \emph{range query} with radius $r$ asks for all
database vectors within $r$ of the query, so that its threshold is $r$ from the
start. \certex\ thus answers a range query by setting $U\coloneqq r$, without any
anchors. Since the retina may also keep vectors farther than $r$, the fovea
computes their exact distances and removes them. \certex\ therefore returns
exactly the vectors within $r$, by Propositions~\ref{prop:lb}
and~\ref{prop:cert}.

In particular, near-duplicate detection fixes its threshold in advance as well,
because it asks for all pairs of database vectors whose distance is at most a
given radius $r$. We call the set of these pairs the \emph{similarity self-join}.
The self-join therefore consists of the range queries of all database vectors
against the database, which \certex\ answers exactly as above. Because each pair
appears in two of these range queries, the retina pass evaluates only the pairs
$(\mathbf{x}_i,\mathbf{x}_j)$ with $i<j$, so that each pair is tested once. In
practice, the radius $r$ follows from the criterion of deduplication pipelines
such as SemDeDup~\cite{abbas2023semdedup}, which keep the pairs whose cosine
similarity is at least $\tau$. For unit vectors, such as normalized image
embeddings, this criterion corresponds to the radius $r=\sqrt{2-2\tau}$. Because
the self-join is defined by a radius rather than by a number of neighbors, it
places no limit on the number of duplicates per vector. In contrast, a kNN graph keeps at most $k$ neighbors of each vector and therefore
misses every duplicate beyond the $k$-th neighbor.

\section{Evaluation}
\label{sec:eval}

\begin{figure*}[t]
\centering
\includegraphics[width=\linewidth]{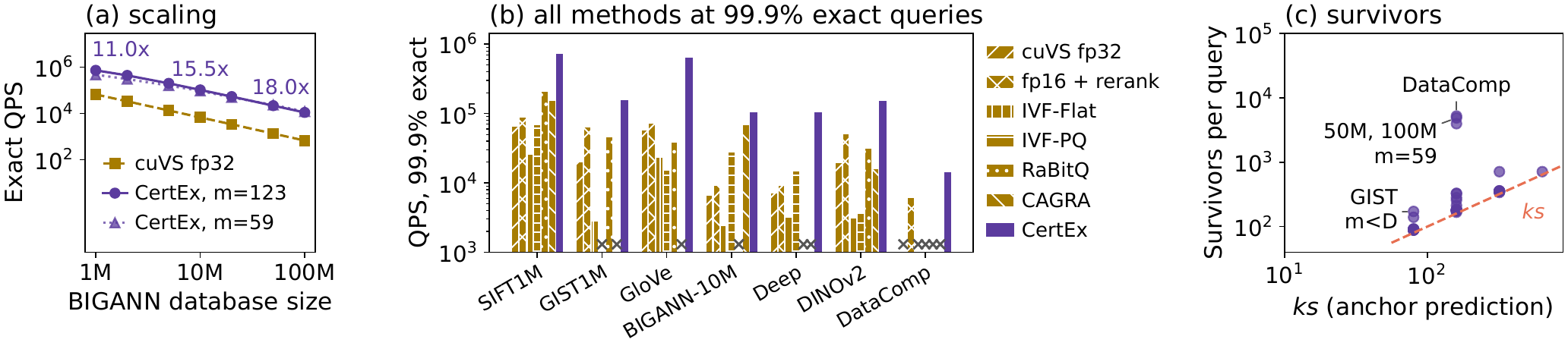}
\caption{Scaling with $N$, all methods at 99.9\% exact queries ($\times$: never reached), and survivors against $ks$ (dashed).}
\label{fig:eval}
\end{figure*}

\begin{figure*}[t]
\centering
\includegraphics[width=\linewidth]{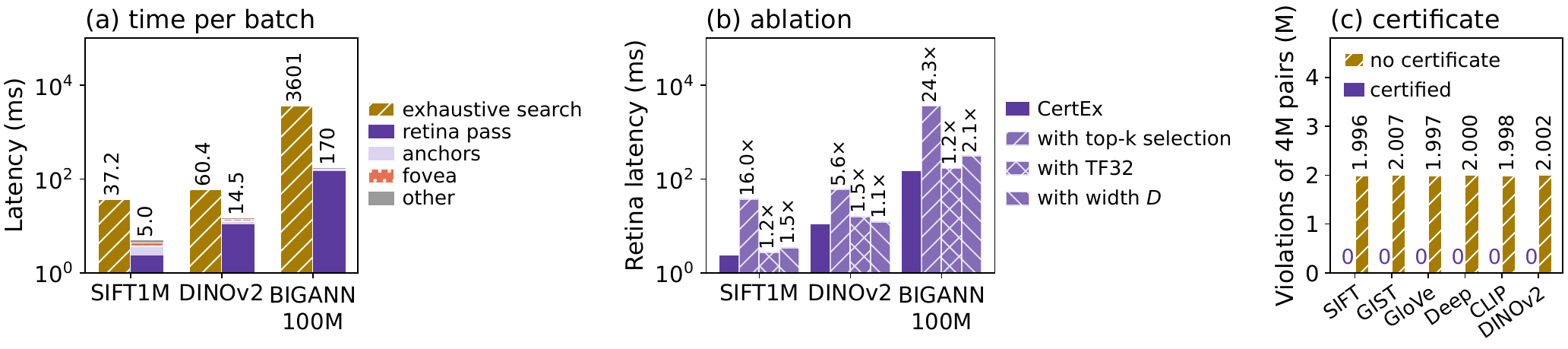}
\caption{Time per batch of 2048 queries, ablation of the retina pass, and violations without the certificate.}
\label{fig:breakdown}
\end{figure*}

\subsection{Experimental Setup}
\label{sec:exp-setup}

\paragraph{Datasets and encoders.}
The evaluation covers sixteen descriptor, word, and image-embedding datasets that
span dimension and scale (Table~\ref{tab:main}). Specifically, SIFT1M, GIST1M, and
the BIGANN prefixes up to 100M vectors~\cite{jegou2011pq,jegou2011bigann} contain
SIFT~\cite{lowe2004sift} and GIST~\cite{oliva2001gist} descriptors, whereas
Deep-image~\cite{babenko2016deep1b} and GloVe~\cite{pennington2014glove} contain
9.99M CNN descriptors and 1.18M word embeddings. Moreover, the
ImageNet-1K datasets~\cite{deng2009imagenet,russakovsky2015ilsvrc} use 1.28M
training images as the database and 10{,}000 validation images as queries, both
embedded by seven vision encoders, which are vision
transformers~\cite{dosovitskiy2021vit} of the
CLIP~\cite{radford2021clip}, SigLIP~\cite{zhai2023siglip},
DINOv2~\cite{oquab2024dinov2}, and DINOv3~\cite{simeoni2025dinov3} families
(Table~\ref{A-tab:imagenet}). In addition, a cross-modal workload queries the CLIP image
embeddings with 10{,}000 CLIP text prompts of the 1{,}000 class names
(Appendix~\ref{A-sec:baselines}). At web
scale, DataComp~\cite{gadre2023datacomp} contributes the CLIP ViT-L/14
embeddings of 12.79M web image-text pairs with 10{,}000 image and caption
queries.

\paragraph{Platform, metrics, and baselines.}
Unless stated otherwise, every experiment runs on one NVIDIA H100 GPU with the kernels of \certex\ written
in Triton~\cite{tillet2019triton} on top of PyTorch~\cite{paszke2019pytorch}. Every measurement in this paper uses this Triton implementation, whereas the open-source library\footnote{\url{https://github.com/hpdic/libcertex}} implements \certex\ with CUDA kernels. The library is also distributed on PyPI as \texttt{certex}\footnote{\url{https://pypi.org/project/certex/}}, and its interfaces replace \texttt{torch.topk}, the flat indexes of Faiss~\cite{johnson2021faiss}, and the brute-force search of cuVS~\cite{nvidia2025cuvs}, so that moving existing code to exact search changes little more than an import. Section~\ref{sec:a100} repeats four
datasets on an A100, at 8.8$\times$ to 11.8$\times$ cuVS. Every experiment keeps
the database in device memory next to the retina, which adds $2(m+5)$ bytes per
vector (Appendix~\ref{A-sec:impl}). Because the fovea reads only the survivors, the
database could also reside in host memory. Throughput counts the queries per second over
all queries with $k=10$ after a warm-up pass. For exactness, we recompute the
squared distances of the returned neighbors by direct differences in fp64 and
call a query exact when these sorted distances equal those of the true neighbors
to a relative $10^{-12}$, which accepts ties among equidistant vectors and no
other deviation. The true neighbors come from an fp64 search that is reranked by
direct differences in fp64. We compare \certex\ with six GPU baselines: the exhaustive search of
NVIDIA cuVS~\cite{nvidia2025cuvs} in fp32 at its fastest selection size
(Table~\ref{A-tab:s-fp32}), the same search in half precision with exact reranking, and
the cuVS indexes IVF-Flat~\cite{sivic2003videogoogle}, IVF-PQ~\cite{jegou2011pq},
IVF-RaBitQ~\cite{gao2024rabitq}, and CAGRA~\cite{ootomo2024cagra}. For a fair
comparison, each baseline uses the best configuration of a grid, with exact
reranking for the quantized indexes (Appendix~\ref{A-sec:baselines}), as does \certex\ with
a few widths and anchor strides per dataset (Table~\ref{A-tab:s-config}).

\subsection{Exact Search Throughput}
\label{sec:throughput}

We next compare the exact throughput of \certex\ with exhaustive cuVS search in
fp32 and in fp16 with exact reranking (Table~\ref{tab:main}). On all sixteen datasets, \certex\ is exact and
7.1$\times$ to 18.0$\times$ faster than the fastest cuVS~\cite{nvidia2025cuvs}
fp32 configuration. In contrast, exhaustive fp32 search evaluates the expanded
distance in fp32, which leaves 1 to 108 of the 10{,}000 queries inexact on ten of
the datasets (Appendix~\ref{A-sec:tools}).
Half-precision exhaustive search with exact reranking remains 1.8$\times$ to
12.9$\times$ slower than \certex, so that half precision explains only part of
the speed. Another part comes from the width, because the ImageNet embeddings
concentrate their variance in fewer directions than their dimension and
therefore leave only 139 to 327 survivors per query (Table~\ref{A-tab:s-config}).
Conversely, isotropic data enlarge the term $4\rho(\mathbf{q})\rho(\mathbf{x})$ of Proposition~\ref{prop:surv} and thus the survivors.

Figure~\ref{fig:eval} then varies the database size and compares all methods. On
seven BIGANN prefixes~\cite{jegou2011bigann} from 1M to 100M vectors
(Figure~\ref{fig:eval}(a)), the speedup over the fastest cuVS fp32 configuration
rises monotonically from 11.0$\times$ to 18.0$\times$, because only the retina
pass scales with $N$ whereas the anchors and the fovea cost nearly the same at
every size (Table~\ref{A-tab:s-scaling}). At 99.9\% exact
queries (Figure~\ref{fig:eval}(b)), \certex\ is 27.2$\times$ to 84.9$\times$
faster than IVF-Flat~\cite{sivic2003videogoogle}, 3.7$\times$ to 124$\times$
faster than IVF-PQ~\cite{jegou2011pq}, and 2.3$\times$ to 10.7$\times$ faster
than IVF-RaBitQ~\cite{gao2024rabitq}, which is the strongest index on the
ImageNet embeddings. Moreover, \certex\ is exact for all queries on every
dataset, whereas CAGRA~\cite{ootomo2024cagra} reaches this level only on SIFT1M among
the thirteen datasets on which it can be built (Table~\ref{A-tab:s-cagra}). In addition,
the retina is built in 0.03 to 2.47 seconds, against 2.2 to 29.7 seconds for each
CAGRA graph (Table~\ref{A-tab:s-config} and Table~\ref{A-tab:s-cagra}).
On the
12.79M DataComp images~\cite{gadre2023datacomp}, \certex\ is the only
method that is exact for every image query. Finally, Figure~\ref{fig:eval}(c)
compares the survivors of all datasets and BIGANN strides with the estimate
$ks$. Specifically, survivors lie at
most 1.15$\times$ above $ks$ wherever the width nears the dimension and further above it for narrow widths
(Proposition~\ref{prop:surv}).

\subsection{Ablation of the Design Elements}
\label{sec:ablation}
\label{sec:valid}

We then ask where the speed of \certex\ comes from and what its certificate costs,
with every element isolated on the same kernel and threshold
(Figure~\ref{fig:breakdown}). For 2048 queries, a complete \certex\ search takes
5.0 ms against 37.2 ms for exhaustive search on SIFT1M and 170 ms against 3.6 s
on BIGANN-100M, of which the anchors and the fovea take only 2.2 ms and 22 ms
(Figure~\ref{fig:breakdown}(a)). Most of the remaining time therefore belongs to
the retina pass, which costs less than a plain fp16 matrix multiplication (GEMM)
over the same vectors (2.3 ms against 4.1 ms on SIFT1M and 148 ms against 289 ms
on BIGANN-100M, Table~\ref{A-tab:s-breakdown}) because it compares every product with
its threshold inside one kernel instead of writing the matrix of all products. In
contrast, exhaustive search must select the top 100 candidates among all $N$
products, which costs 3.3$\times$ to 12.5$\times$ the GEMM alone. Moreover, the
ablation in Figure~\ref{fig:breakdown}(b) shows that certified exclusion provides
most of the speed: the top-$k$ selection instead slows the pass by 5.6$\times$ to 24.3$\times$, whereas TF32 instead of fp16 and
the width $D$ instead of $m$ slow it by only 1.2$\times$ to 1.5$\times$ and
1.1$\times$ to 2.1$\times$.

In addition, the certificate makes the half-precision decision exact
(Figure~\ref{fig:breakdown}(c)). The hardest test places every pair exactly on the
decision boundary: for four million random pairs per dataset and width, we set
the threshold of each pair to its own lower bound. In this test, the retina keeps
every pair, whereas the same retina without the error terms of Eq.~\eqref{eq:cert} excludes 1.996M to 2.007M of the 4M pairs, because half-precision rounding errs in either direction with equal
frequency. Moreover, the largest accumulation error over up to $10^8$
pairs per dataset and every kernel configuration is 0.42\% to 0.76\% of the
certified bound (Table~\ref{A-tab:acc}), so that the certificate holds with a margin of
more than two orders of magnitude. The certificate costs only 0.4\% to 2.7\% of
throughput, yet it decides real boundary pairs: without it, the exact self-join of
the DINOv2~\cite{oquab2024dinov2} embeddings misses 299, 50, and 1 duplicate pairs
at $\tau=0.90$, 0.95, and 0.98. The remaining parameters affect only the cost: \certex\ stays faster than exhaustive search for $k$ from 1 to 100
and batches of 1 to 4096 queries (Figure~\ref{A-fig:ablation}).

\subsection{Vision Encoders}
\label{sec:encoders}

\begin{figure}[t]
\centering
\includegraphics[width=\linewidth]{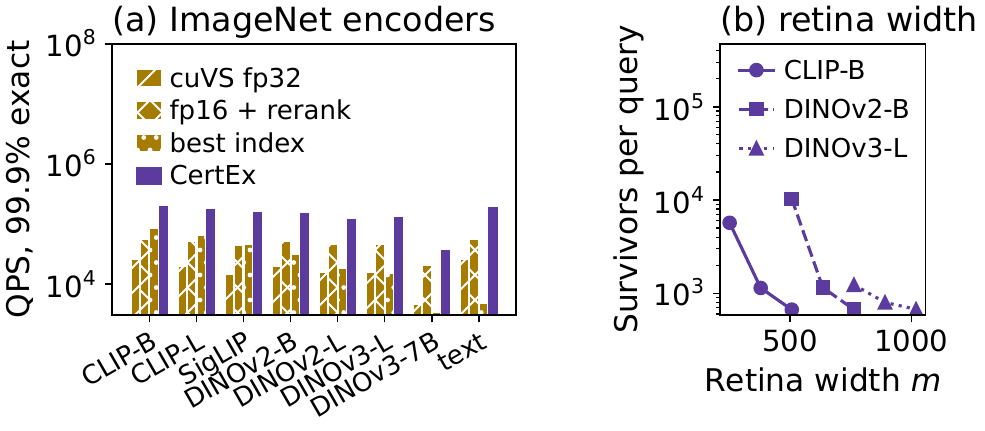}
\caption{Seven ImageNet encoders and the retina width.}
\label{fig:enc}
\end{figure}

We next ask whether the retina stays selective on modern vision encoders (Figure~\ref{fig:enc}). On all seven ImageNet encoders, the retina keeps
few survivors, so that \certex\ is exact for every query and faster than every
baseline at 99.9\% exact queries (Figure~\ref{fig:enc}(a)). Moreover, a larger
width tightens the bound, so that the survivors per query on
DINOv2~\cite{oquab2024dinov2} fall from 10.3k at $m=507$ to 671 at $m=763$
(Figure~\ref{fig:enc}(b)).

Cross-modal queries further test exactness under a shift of the query distribution, because CLIP~\cite{radford2021clip} text embeddings occupy a different region of the joint embedding space than the image embeddings~\cite{liang2022mindgap}. Consequently,
IVF-Flat~\cite{sivic2003videogoogle} is the only index that reaches 99.9\% exact
text queries, whereas \certex\ is exact for every text query at 41$\times$ its
throughput (Figure~\ref{fig:enc}(a), right), because certified
exclusion does not depend on the query distribution.

\subsection{Downstream Tasks}
\label{sec:downstream}

\begin{figure}[t]
\centering
\includegraphics[width=\linewidth]{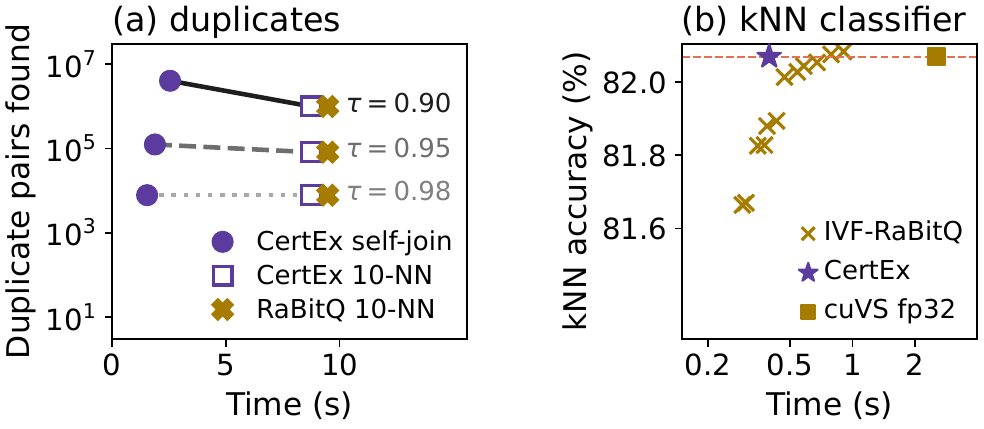}
\caption{Duplicates and kNN classification on DINOv2.}
\label{fig:down}
\end{figure}

\begin{figure}[t]
\centering
\includegraphics[width=\linewidth]{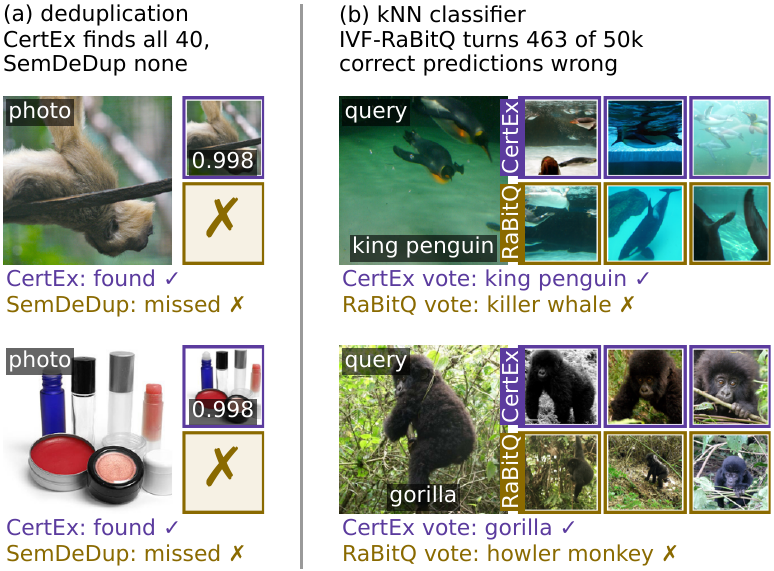}
\caption{Real ImageNet examples of both downstream tasks.}
\label{fig:qual}
\end{figure}

We finally ask what exact neighbors change in deduplication and kNN
classification (Figure~\ref{fig:down} and Figure~\ref{fig:qual}). For duplicate detection, the exact
self-join of the 1.28M ImageNet~\cite{deng2009imagenet} training images takes 1.5
to 2.5 seconds, which is 22$\times$ to 37$\times$ faster than an exhaustive fp32
join, whereas an exact 10-NN graph contains only 24.6\% of its 4.0M pairs at
$\tau=0.90$ (Figure~\ref{fig:down}(a)). Similarly,
SemDeDup~\cite{abbas2023semdedup} avoids the quadratic join by comparing only the
pairs inside each k-means cluster. Measured against the exact self-join of
\certex, this restriction misses up to 26\% of the pairs of the worst ImageNet
classes and up to 66\% of all pairs with 50{,}000 clusters (Table~\ref{A-tab:semdedup}).
At web scale, the exact self-join finds all 31.4 billion pairs among 1.97M of the
12.79M DataComp images~\cite{gadre2023datacomp} at $\tau=0.95$, whereas the
cluster restriction cannot detect the duplicates of 41{,}875 to 56{,}416 of these
images. Similarly, at $\tau=0.98$, the exact self-join finds all 40 ImageNet images
whose copies k-means places in other clusters, whereas SemDeDup with one cluster
per class finds none (Figure~\ref{fig:qual}(a)).

For the kNN classifier of DINO~\cite{caron2021dino}, \certex\ returns the exact
neighbors of all 50{,}000 validation images in 0.40 seconds, which is 6.4$\times$
faster than the fastest cuVS fp32 configuration (Figure~\ref{fig:down}(b)). In
contrast, the approximate neighbors of IVF-RaBitQ~\cite{gao2024rabitq} lower the
top-1 accuracy of 82.07\% by up to 0.42 points, so that a reported kNN accuracy
also measures the index. Specifically, IVF-RaBitQ turns 463 of the 50{,}000
correct predictions wrong and only 259 wrong ones correct. For
instance, the exact neighbors of a king penguin are king penguins, whereas its approximate ones are killer whales (Figure~\ref{fig:qual}(b)).

\subsection{Integer Retina}
\label{sec:int8}

The certificate of Proposition~\ref{prop:cert} needs only an upper bound on the
rounding error of each lifted vector, so that it extends to the 8-bit integer
tensor cores, which multiply integer codes and accumulate their products exactly
in 32-bit integers. Specifically, \certex\ stores the principal coordinates
$\mathbf{z}(\mathbf{x})\coloneqq\mathbf{V}(\mathbf{x}-\mathbf{c})$ of each database
vector as an integer code $\mathbf{w}_{\mathbf{x}}\in\{-127,\dots,127\}^m$ with one
scale $s_{\mathbf{x}}\coloneqq\|\mathbf{z}(\mathbf{x})\|_\infty/127$ per vector, so
that $\mathbf{z}(\mathbf{x})=s_{\mathbf{x}}(\mathbf{w}_{\mathbf{x}}+\mathbf{d}_{\mathbf{x}})$
with $\|\mathbf{d}_{\mathbf{x}}\|_\infty\le1/2$. The residual coordinate
$\rho(\mathbf{x})$ stays in fp32, because it is often the largest coordinate and
would otherwise coarsen the scale of all others. A query is encoded in the same
way. Splitting $\langle\mathbf{z}(\mathbf{q}),\mathbf{z}(\mathbf{x})\rangle$ into the
product of the codes and the rounding terms and applying the Cauchy--Schwarz
inequality to the rounding terms, as in Eq.~\eqref{eq:cert}, gives
\begin{equation}
\begin{split}
\langle\mathbf{z}(\mathbf{q}),\mathbf{z}(\mathbf{x})\rangle&\le s_{\mathbf{q}}s_{\mathbf{x}}\beta(\mathbf{q},\mathbf{x}),\quad\text{where}\\
\beta(\mathbf{q},\mathbf{x})&\coloneqq\langle\mathbf{w}_{\mathbf{q}},\mathbf{w}_{\mathbf{x}}\rangle
+\|\mathbf{d}_{\mathbf{q}}\|\|\mathbf{w}_{\mathbf{x}}\|\\
&\quad+\|\mathbf{d}_{\mathbf{x}}\|(\|\mathbf{w}_{\mathbf{q}}\|+\|\mathbf{d}_{\mathbf{q}}\|).
\end{split}
\label{eq:int8}
\end{equation}

Since $\langle\lift(\mathbf{q}),\lift(\mathbf{x})\rangle=\langle\mathbf{z}(\mathbf{q}),\mathbf{z}(\mathbf{x})\rangle+\rho(\mathbf{q})\rho(\mathbf{x})$,
Eq.~\eqref{eq:inner} and Eq.~\eqref{eq:int8} bound $\|\lift(\mathbf{q})-\lift(\mathbf{x})\|^2$
from below by
$B(\mathbf{q},\mathbf{x})\coloneqq a(\mathbf{q})+a(\mathbf{x})-2(s_{\mathbf{q}}s_{\mathbf{x}}\beta(\mathbf{q},\mathbf{x})+\rho(\mathbf{q})\rho(\mathbf{x}))$.
The integer retina evaluates $B$ in fp32 and keeps $\mathbf{x}$ if and only if the
computed value $\widehat{B}(\mathbf{q},\mathbf{x})$ satisfies
\begin{equation}
\begin{split}
\widehat{B}(\mathbf{q},\mathbf{x})&\le U^2(1+10^{-4})+\omega(\mathbf{q}),\quad\text{where}\\
\omega(\mathbf{q})&\coloneqq64u\bigl(\sqrt{a(\mathbf{q})}+\sqrt{M_a}\bigr)^2,
\end{split}
\label{eq:int8test}
\end{equation}
where $u=2^{-24}$ and $M_a$ denotes the database maximum of $a(\mathbf{x})$.

\begin{proposition}[No false dismissal with integer codes]
\label{prop:int8}
Assume $127^2m<2^{24}$. If $\|\lift(\mathbf{q})-\lift(\mathbf{x})\|\le U$, the integer
retina keeps $\mathbf{x}$.
\end{proposition}

\begin{proof}
We bound the exact value, then the product of the codes, and finally the rounding
of the remaining arithmetic in three steps.

Step 1 bounds the exact value. By Eq.~\eqref{eq:int8} and the premise,
$B(\mathbf{q},\mathbf{x})\le\|\lift(\mathbf{q})-\lift(\mathbf{x})\|^2\le U^2$.

Step 2 concerns the product of the codes. The integer tensor cores return
$\langle\mathbf{w}_{\mathbf{q}},\mathbf{w}_{\mathbf{x}}\rangle$ exactly. Because
the magnitude of this product is at most $127^2m<2^{24}$, the conversion of the
product to fp32 is exact as well.

Step 3 bounds the rounding of the remaining arithmetic. The other scalars are
stored in fp32 and combined by fp32 multiplications and additions, so that each of
the four summands $a(\mathbf{q})$, $a(\mathbf{x})$,
$2s_{\mathbf{q}}s_{\mathbf{x}}\beta$, and $2\rho(\mathbf{q})\rho(\mathbf{x})$ of $B$
passes through at most 16 roundings. Hence
$|\widehat{B}-B|\le\gamma_{16}S$, where $\gamma_{16}\coloneqq16u/(1-16u)$ and $S$ is
the sum of the magnitudes of the four summands. Since
$s_{\mathbf{q}}\|\mathbf{w}_{\mathbf{q}}\|\le\|\mathbf{z}(\mathbf{q})\|+s_{\mathbf{q}}\|\mathbf{d}_{\mathbf{q}}\|$
and $\|\mathbf{d}_{\mathbf{q}}\|\le\sqrt{m}/2$, the scale satisfies
$s_{\mathbf{q}}(\|\mathbf{w}_{\mathbf{q}}\|+\|\mathbf{d}_{\mathbf{q}}\|)\le(1+\sqrt{m}/127)\|\mathbf{z}(\mathbf{q})\|$,
and likewise for $\mathbf{x}$. Because
$|\beta|\le(\|\mathbf{w}_{\mathbf{q}}\|+\|\mathbf{d}_{\mathbf{q}}\|)(\|\mathbf{w}_{\mathbf{x}}\|+\|\mathbf{d}_{\mathbf{x}}\|)$,
the third summand is therefore at most
$2(1+\sqrt{m}/127)^2\sqrt{a(\mathbf{q})a(\mathbf{x})}$, where the factor
$(1+\sqrt{m}/127)^2$ stays below 1.6 under the assumption. The fourth summand is at
most $2\sqrt{a(\mathbf{q})a(\mathbf{x})}$, so that
$S\le a(\mathbf{q})+a(\mathbf{x})+5.2\sqrt{a(\mathbf{q})a(\mathbf{x})}\le2.6\bigl(\sqrt{a(\mathbf{q})}+\sqrt{a(\mathbf{x})}\bigr)^2$.
Consequently, $|\widehat{B}-B|\le42u\bigl(\sqrt{a(\mathbf{q})}+\sqrt{a(\mathbf{x})}\bigr)^2\le\omega(\mathbf{q})$.
Combined with Step 1, this bound gives $\widehat{B}\le U^2+\omega(\mathbf{q})$, so
that the test of Eq.~\eqref{eq:int8test} keeps $\mathbf{x}$.
\end{proof}

The integer retina runs as one Triton~\cite{tillet2019triton} kernel, which
multiplies the codes on the integer tensor cores and evaluates
Eq.~\eqref{eq:int8test} on the accumulated products. Each database vector occupies
its $m$ codes, padded as the half-precision retina vectors, and five fp32
scalars, namely $a(\mathbf{x})$, $s_{\mathbf{x}}$, $\|\mathbf{d}_{\mathbf{x}}\|$,
$\|\mathbf{w}_{\mathbf{x}}\|$, and $\rho(\mathbf{x})$. At the same width, the
integer retina therefore needs about half of the $2(m+5)$ bytes of the
half-precision retina. As in the half-precision pass, the threshold $U$ comes from
the exact distances of the anchors, and the fovea verifies every survivor in fp64.

\begin{figure}[t]
\centering
\includegraphics[width=\linewidth]{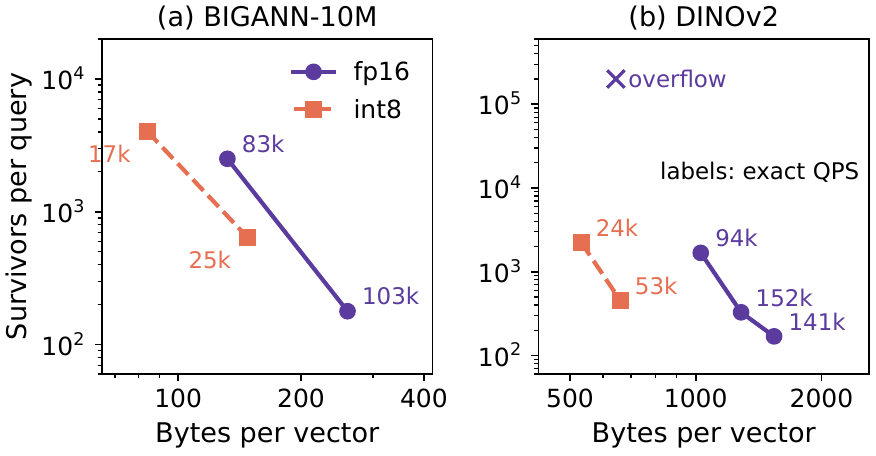}
\caption{Survivors of the fp16 and int8 retinas.}
\label{fig:s-int8}
\end{figure}

\begin{figure*}[t]
\centering
\includegraphics[width=\linewidth]{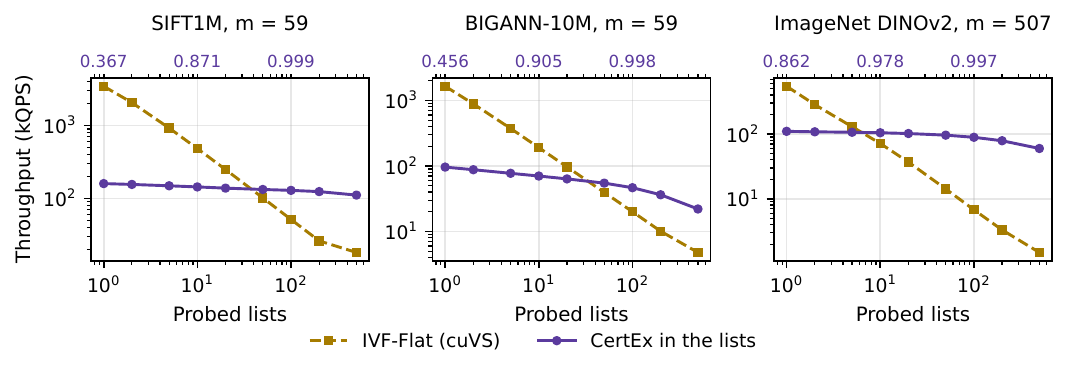}
\caption{Throughput of IVF-Flat and of \certex\ in the same probed lists, with the recall@10 of \certex\ on top.}
\label{fig:ivf}
\end{figure*}

Figure~\ref{fig:s-int8} compares the two retinas on BIGANN-10M and DINOv2 with an
anchor stride of 16 and 16 anchors, which is the configuration of the fastest
DINOv2 run. Every query of every run is exact. In addition, the validation of one million
pairs per run at their own lower bound finds no violation. At the same width, the integer
retina halves the device memory at the cost of more survivors, which rise from
327 to 452 per query on DINOv2 at $m=635$ and from 178 to 641 on BIGANN-10M at
$m=123$. At the same memory, the integer retina keeps far fewer survivors. On
BIGANN-10M, for example, 148 bytes per vector at $m=123$ keep 641 survivors per
query, whereas the half-precision retina keeps 2,507 at $m=59$ with 132 bytes. On
DINOv2, the half-precision retina with 644 bytes per vector at $m=315$ overflows the
survivor list, whereas the integer retina with 660 bytes per vector at $m=635$ is
exact at 53.5k queries per second, which is 2.7$\times$ the throughput of the
fastest cuVS search.

The integer retina is still slower than the half-precision retina at the same
width, with 53.5k against 152k queries per second on DINOv2 at $m=635$. Its kernel
is neither persistent nor tuned, whereas the integer tensor cores of the H100 offer
twice the fp16 rate per coordinate. The integer retina therefore suits datasets
that would otherwise exceed device memory, since it stores about twice as many
vectors in the same memory while its certificate keeps every answer exact.

\subsection{Certified Screening in Inverted Lists}
\label{sec:ivf}

Inverted-file indexes restrict every query to the lists nearest to it and compute the distance to every vector of these lists in full precision, so that their cost grows linearly with the number of probed lists. Certified exclusion applies inside the probed lists without changing the candidates of a query. Specifically, screening the retina rows of the probed lists under a threshold from their anchors, followed by ordered verification, returns the exact $k$ nearest neighbors among the vectors of the probed lists. The answer of \certex\ therefore equals the answer of an exhaustive scan of the same lists, and both reach the same recall. We evaluate this combination against IVF-Flat~\cite{sivic2003videogoogle} of cuVS~\cite{nvidia2025cuvs} on SIFT1M, BIGANN-10M, and ImageNet DINOv2 with $k=10$, 1,024 lists, and 1 to 500 probed lists on one H100.

The two methods share the centers of the lists, which cuVS trains, and every vector joins the list of its nearest center. \certex\ stores the vectors in list order with one global projection, so that a routed kernel screens the queries that probe a list against the rows of that list only. Every list keeps 128 evenly spaced anchors. The threshold of a query is the $k$-th exact distance among the 128 anchors of its probed lists whose retina products are largest. Ordered verification then computes the exact distances of the 32 survivors of smallest certified bound and drops every survivor whose bound exceeds the resulting $k$-th distance. Exhaustive search over the probed lists in double precision confirms that \certex\ is exact within the probed lists for 200 queries at every setting. The answers of the two methods coincide on every query of SIFT1M and DINOv2, whereas on the uint8 vectors of BIGANN-10M the lists of cuVS place some vectors differently, so that the two answers differ at few probed lists while their recall stays within 0.02.

Figure~\ref{fig:ivf} reports the throughput of both methods against the number of probed lists, with the recall@10 of \certex\ at three settings on the top axis. The throughput of IVF-Flat falls in proportion to the probed lists, because it computes the distance to every vector of these lists. In contrast, \certex\ verifies an almost constant number of vectors per query (100 to 244 on SIFT1M, 155 to 258 on BIGANN-10M, and 129 to 147 on DINOv2), so that its throughput falls slowly. \certex\ therefore overtakes IVF-Flat at 50 probed lists on SIFT1M and BIGANN-10M, where both reach recall 0.99, and at 10 probed lists on DINOv2, where both reach recall 0.978. At 200 probed lists, where the recall of \certex\ exceeds 0.999, \certex\ is 4.7, 3.6, and 23.7 times faster than IVF-Flat on the three datasets, and on DINOv2 it is 40.1 times faster at 500 probed lists.

\begin{table*}[t]
\centering
\caption{Exact search on one A100 ($k=10$): fastest exact \certex, fastest cuVS fp32, and speedups on both GPUs.}
\label{tab:s-a100}
\small
\setlength{\tabcolsep}{4pt}
\begin{tabular}{lrrrrrrrrr}
\toprule
Data & $N$ & $D$ & $m$ & $s$ & Surv. & \certex & cuVS fp32 & A100 & H100 \\
\midrule
SIFT1M & 1M & 128 & 123 & 8 & 91 & 312.1k & 34.7k & 9.0$\times$ & 11.0$\times$ \\
GIST1M & 1M & 960 & 379 & 8 & 380 & 69.1k & 7.44k & 9.3$\times$ & 7.7$\times$ \\
GloVe & 1.18M & 100 & 99 & 8 & 88 & 272.9k & 31.1k & 8.8$\times$ & 11.0$\times$ \\
BIGANN-10M & 10M & 128 & 59 & 8 & 1531 & 40.8k & 3.46k & 11.8$\times$ & 15.5$\times$ \\
\bottomrule
\end{tabular}
\end{table*}

At one to five probed lists, IVF-Flat is faster than \certex, because it reads only a few thousand rows per query, whereas the anchor stage of \certex\ costs 51 to 63~ms per 10,000 queries regardless of the probed lists. Certified exclusion therefore turns the exhaustive scan of the probed lists into a screen whose cost hardly depends on the number of probed lists, which matters at the high recall that applications sensitive to missed neighbors require. The retina of DINOv2 needs a wide projection, because the 768-dimensional embeddings spread their variance over many directions. For instance, at $m=255$ the survivors grow to 45,800 per query at 200 probed lists, and \certex\ only matches IVF-Flat, whereas $m=507$ keeps them at about 1,000.

\subsection{Exact Search on an A100}
\label{sec:a100}

The certificate of \certex\ assumes only half-precision products with fp32
accumulation~\cite{blanchard2020tensorcore}, which the tensor cores of other GPUs
provide as well. We therefore repeat the exact search of SIFT1M, GIST1M, GloVe,
and BIGANN-10M with the same code on an NVIDIA A100 80GB PCIe GPU
(Table~\ref{tab:s-a100}). The A100 runs as one MIG instance of the 7g.80gb
profile, which provides 98 of its 108 streaming multiprocessors and its full
memory. On this instance, \certex\ searches the anchor strides 8, 16, and 32 at the
widths around its H100 configuration. The baseline is the fastest configuration
of the same cuVS~\cite{nvidia2025cuvs} fp32 grid on the same instance.

On all four datasets, \certex\ is exact for every query at 8.8$\times$ to
11.8$\times$ the throughput of the fastest cuVS fp32 configuration, against
7.7$\times$ to 15.5$\times$ on the H100. In addition, the boundary validation of
Section~\ref{sec:valid} keeps every pair in all 24 runs. As on the H100, the
largest speedup belongs to BIGANN-10M~\cite{jegou2011bigann}, the largest of the
four datasets. Moving from the H100 to the A100, \certex\ loses 2.3$\times$ to
2.5$\times$ of its throughput, whereas exhaustive search loses 1.9$\times$ on
three datasets and 2.75$\times$ on the 960 dimensions of GIST1M. The speedup on
GIST1M therefore rises from 7.7$\times$ to 9.3$\times$. Moreover, the fastest
width changes on two datasets: GIST1M prefers $m=379$ instead of 507, and
BIGANN-10M prefers $m=59$ instead of 123. Because the width affects only the cost,
every width remains exact on both GPUs.

\section{Conclusion}
\label{sec:conclusion}

Exact kNN search for image retrieval, deduplication, and the memory of LLM agents
has so far required either a full-precision distance to every vector or an index
that gives up exactness. \certex\ shows that exact search needs
neither: a lifted lower bound with a certificate for its rounding error excludes almost
every vector in one half-precision tensor-core pass without a false dismissal. On
sixteen datasets of up to 100M vectors, \certex\ is exact for every query at up to 18.0$\times$ the throughput of
the fastest exhaustive GPU search. Exact search thus becomes a practical default
for visual retrieval, deduplication, and representation evaluation.

\clearpage
\onecolumn
\appendix

\section{Implementation}
\label{A-sec:impl}

\begin{algorithm}[!htb]
\caption{\certex\ search for a batch of queries.}
\label{A-alg:search}
\begin{algorithmic}[1]
\Require retina vectors $\{\hat{\mathbf{x}}:\mathbf{x}\in\mathcal{X}\}$, rotation $(\mathbf{c},\mathbf{V})$, queries $\mathcal{Q}$
\Ensure exact $\mathcal{N}_k(\mathbf{q})$ for each $\mathbf{q}\in\mathcal{Q}$
\State lift, round, and augment each query to $\hat{\mathbf{q}}$
\State score every $s$-th retina vector and keep the best as anchors $\mathcal{A}(\mathbf{q})$
\State $U(\mathbf{q})\gets$ $k$-th exact distance over $\mathcal{A}(\mathbf{q})$
\State $\theta(\mathbf{q})\gets$ threshold of Eq.~\eqref{eq:test}
\State $\mathcal{S}\gets\{(\mathbf{q},\mathbf{x}):\langle\hat{\mathbf{q}},\hat{\mathbf{x}}\rangle\ge\theta(\mathbf{q})\}$ in one tensor-core pass
\State compute exact distances for all pairs in $\mathcal{S}$
\State \Return the $k$ smallest distances for each query
\end{algorithmic}
\end{algorithm}

Algorithm~\ref{A-alg:search} summarizes the search of a batch of queries, which
runs as three Triton~\cite{tillet2019triton} kernels and one selection. The first
kernel is the retina pass, which runs as a persistent tensor-core kernel: its
units of parallel work stay resident and process tiles until the pass ends.
Triton calls such a unit a program. In the retina pass, each program keeps one
tile of query retina vectors and their thresholds in registers and streams every
$G$-th tile of database retina vectors from device memory, so that a query tile is
loaded once per pass instead of once per database tile. After a tile of products
is accumulated, the kernel compares every product with its threshold and appends
the surviving pairs to a global list, which we call the survivor list. Because
this list has a fixed capacity, the kernel checks the capacity and reports an
overflow instead of truncating the answer. The tile shapes are tuned once per
vector length, database size, and batch size. For every tile shape, the kernel
reduces over the vector length in one sequential loop without splitting it across
programs, so that the error model of Proposition~\ref{prop:cert} holds for every
tile configuration.

The second kernel performs the selection of the anchors, which we call the anchor
stage. This kernel scores every $s$-th retina vector by the same computed product
$\widehat{P}$ as the retina pass. After a tile of products is accumulated, the
kernel keeps only the best vector of each group of consecutive sampled vectors
and therefore never writes a score matrix either. The anchors are then the best
vectors of the best groups, whose exact distances set $U$. The third kernel
implements the fovea: it gathers the survivors and computes their distances by
direct differences with an fp64 accumulation. A stable sort of the flat survivor
list, first by distance and then by query, finally selects the $k$ smallest per
query. As for memory, each retina vector holds the $m$ principal coordinates, the
coordinate $\rho(\mathbf{x})$, and the four certificate coordinates, padded to 64,
128, or a multiple of 128 so that every tile of the tensor-core product is full.
The retina therefore occupies $2(m+5)$ bytes per vector plus one float for
$a(\mathbf{x})$. In addition, the anchor sample occupies $2(m+1)/s$ bytes per
vector.

\FloatBarrier
\section{Experimental Details}
\label{A-sec:baselines}

\begin{table}[t]
\centering
\caption{Evaluated datasets ($k=10$).}
\label{A-tab:datasets}
\footnotesize
\setlength{\tabcolsep}{3pt}
\begin{tabular}{lrrrl}
\toprule
Dataset & Vectors & Dim. & Queries & Feature \\
\midrule
SIFT1M~\cite{jegou2011pq} & 1.0M & 128 & 10k & SIFT~\cite{lowe2004sift} \\
GIST1M~\cite{jegou2011pq} & 1.0M & 960 & 1k & GIST~\cite{oliva2001gist} \\
GloVe~\cite{aumuller2020annbench} & 1.18M & 100 & 10k & word~\cite{pennington2014glove} \\
Deep-image~\cite{aumuller2020annbench} & 9.99M & 96 & 10k & CNN~\cite{babenko2016deep1b} \\
BIGANN-10M~\cite{jegou2011bigann} & 10M & 128 & 10k & SIFT \\
BIGANN-100M~\cite{jegou2011bigann} & 100M & 128 & 10k & SIFT \\
ImageNet~\cite{deng2009imagenet} & 1.28M & 512 & 10k & CLIP B/16 \\
ImageNet & 1.28M & 768 & 10k & DINOv2 B/14 \\
ImageNet & 1.28M & 1024 & 10k & DINOv3 L/16 \\
ImageNet & 1.28M & 768 & 10k & CLIP L/14 \\
ImageNet & 1.28M & 1024 & 10k & DINOv2 L/14 \\
ImageNet & 1.28M & 1152 & 10k & SigLIP so400m \\
ImageNet & 1.28M & 4096 & 10k & DINOv3 7B/16 \\
ImageNet & 1.28M & 512 & 10k & CLIP B/16 text \\
DataComp~\cite{gadre2023datacomp} & 12.79M & 768 & 10k & CLIP L/14 image \\
DataComp & 12.79M & 768 & 10k & CLIP L/14 text \\
\bottomrule
\end{tabular}
\end{table}
\begin{table}[t]
\centering
\caption{Variation over five timed repetitions.}
\label{A-tab:repeat}
\footnotesize
\setlength{\tabcolsep}{3pt}
\begin{tabular}{lrrr}
\toprule
Data & \certex\ QPS & cuVS fp32 QPS & Speedup \\
\midrule
SIFT1M & 719.9k $\pm$ 1.3\% & 65.8k $\pm$ 0.0\% & 10.94 $\pm$ 0.14 \\
GIST1M & 158.8k $\pm$ 0.8\% & 20.6k $\pm$ 1.6\% & 7.72 $\pm$ 0.09 \\
GloVe & 657.5k $\pm$ 0.5\% & 56.5k $\pm$ 0.2\% & 11.63 $\pm$ 0.05 \\
BIGANN-10M & 102.2k $\pm$ 0.4\% & 6.7k $\pm$ 0.0\% & 15.23 $\pm$ 0.07 \\
Deep-image & 104.2k $\pm$ 0.9\% & 7.2k $\pm$ 0.0\% & 14.37 $\pm$ 0.12 \\
CLIP-B & 190.0k $\pm$ 0.6\% & 25.7k $\pm$ 0.0\% & 7.40 $\pm$ 0.04 \\
DINOv2-B & 150.2k $\pm$ 0.4\% & 19.7k $\pm$ 0.0\% & 7.63 $\pm$ 0.03 \\
DINOv3-L & 131.4k $\pm$ 0.3\% & 15.5k $\pm$ 0.0\% & 8.46 $\pm$ 0.03 \\
\bottomrule
\end{tabular}
\end{table}
\begin{table}[t]
\centering
\caption{Methods and device bytes per SIFT1M vector.}
\label{A-tab:baselines}
\footnotesize
\setlength{\tabcolsep}{3pt}
\begin{tabular}{lllll}
\toprule
Method & Result & Screening & Index & Bytes \\
\midrule
cuVS exhaustive & exact & fp32 & none & 512 \\
fp16 + rerank & empirical & fp16, top-$k'$ & none & 768 \\
IVF-Flat & approx. & fp32 lists & k-means & 512 \\
IVF-PQ + rerank & approx. & PQ codes & k-means & 576 \\
RaBitQ + rerank & approx. & 3- or 8-bit & k-means & 526, 606 \\
CAGRA & approx. & fp32 walk & graph & 768 \\
\certex & certified & fp16 retina & rotation & 772 \\
\bottomrule
\end{tabular}
\end{table}

Table~\ref{A-tab:datasets} lists the evaluated datasets. Except for Section~\ref{sec:a100}, the experiments run on one
NVIDIA H100 GPU with 100 GB (93.6 GiB) of device memory. The GPU is connected by PCIe 5.0 x16 to a host with 48 virtual cores of
an Intel Xeon Platinum 8468 and 240 GB of memory. On the software side, the stack
consists of CUDA 12.6 with driver 560.35, PyTorch~\cite{paszke2019pytorch} 2.14,
Triton~\cite{tillet2019triton} 3.8, and cuVS~\cite{nvidia2025cuvs} 26.10. Every
method runs alone on the GPU. In addition, the timings exclude the transfer of the
dataset to the device. For the image embeddings, every image is resized to a
shorter side of 256 pixels and center-cropped to $224\times224$ before encoding in
half precision. The 10{,}000 cross-modal text prompts combine ten templates such
as \textit{a photo of a} with the 1{,}000 class names and are encoded by the text
encoder of CLIP ViT-B/16~\cite{radford2021clip}. The timings are stable: five repetitions of the headline
configurations vary by at most 1.3\% for \certex\ and 1.6\% for cuVS
(Table~\ref{A-tab:repeat}). These repetitions run in a separate session, so that their means differ from the
single headline runs of the other tables by up to 3\%.

The true neighbors of every floating-point dataset come from an fp64 search,
whose 128 best candidates per query are reranked by direct differences in fp64.
In contrast, the integer-valued BIGANN~\cite{jegou2011bigann} vectors make every
fp32 distance exact, so that their published ground truth serves directly. In both
cases, the comparison uses distances rather than identifiers, because several
datasets contain exact duplicates with equal distances. On BIGANN-10M and
BIGANN-100M, the returned identifiers therefore differ from the shipped ones for
a few queries, while every returned distance equals the true distance.

Table~\ref{A-tab:baselines} summarizes the baselines, each of which isolates a
different source of the speed of \certex. Specifically, the exhaustive search of
NVIDIA cuVS~\cite{nvidia2025cuvs} in fp32 represents full-precision exhaustive
search. For this baseline, cuVS returns the 10 nearest vectors 1.2$\times$ to
4.1$\times$ faster when it first selects more than 64 candidates. Every
comparison therefore uses the fastest cuVS configuration, which selects between 65
and 100 candidates and keeps the 10 nearest (Section~\ref{A-sec:results}). The second
baseline runs the same search in half precision and selects the $k'$ best
candidates for $k'$ between 10 and 256, which the exact-distance kernel of the
fovea of \certex\ then reranks. This baseline thus isolates the contribution of
half precision alone.

The indexes cover quantized and graph-based approximate search with the grids of
their cuVS 26.10 implementations. Specifically, the grid of
IVF-RaBitQ~\cite{gao2024rabitq} covers 3-bit and 8-bit codes, the number of probed
lists, and reranking ratios from 1 to 10, with 1024 lists for the datasets of
about one million vectors and 4096 lists for BIGANN-10M and Deep-image. The grid
of CAGRA~\cite{ootomo2024cagra} uses graph degree 64, internal candidate lists of
32 to 1024 entries, and search widths of 1, 4, and 8. Finally,
IVF-Flat~\cite{sivic2003videogoogle} and IVF-PQ~\cite{jegou2011pq} use 1024
k-means lists and probe 1 to 1024 of them. For every quantized configuration, the
candidates are reranked by the same exact-distance kernel as the fovea of \certex.
This kernel is about two orders of magnitude faster than the refinement routine
of cuVS on these batches and therefore favors the baselines.

\FloatBarrier
\section{Certificate and Configurations}
\label{A-sec:valid-supp}

\begin{table}[t]
\centering
\caption{Measured accumulation error against the bound.}
\label{A-tab:acc}
\footnotesize
\setlength{\tabcolsep}{3pt}
\begin{tabular}{lrrrrrr}
\toprule
Data & MP & Pairs & Tiles & Error ratio & $c$ & $16\,\mathrm{MP}$ \\
\midrule
SIFT1M & 128 & 101M & 9 & 0.67\% & 14 & 2048 \\
GIST1M & 512 & 49M & 14 & 0.76\% & 63 & 8192 \\
GloVe & 128 & 101M & 9 & 0.42\% & 9 & 2048 \\
DINOv2-B & 640 & 101M & 14 & 0.47\% & 48 & 10240 \\
DINOv3-7B & 3584 & 101M & 14 & 0.56\% & 323 & 57344 \\
BIGANN-10M & 128 & 101M & 9 & 0.75\% & 15 & 2048 \\
\bottomrule
\end{tabular}
\end{table}

Table~\ref{A-tab:acc} measures the error model of the tensor-core product with a
kernel of the same tile and loop structure as the retina pass. The measurement
covers every tile configuration, because the retina pass may use any of them
depending on the tuned tile shapes. Specifically, for 49M to 101M pairs per dataset,
the largest error of the computed product is 0.42\% to 0.76\% of the certified
bound. Equivalently, the measured constant $c$ in
$|\widehat{P}-P|\le c\,u\sum_i|\hat{q}_i\hat{x}_i|$ never exceeds $0.13\,L$,
whereas the certificate uses $16\,L$, which leaves a margin of more than two
orders of magnitude.

Table~\ref{A-tab:s-config} lists the configuration of every headline run, which is
the fastest exact run over the evaluated widths and anchor strides, together with
the violations in the boundary test of Section~\ref{sec:valid}, of which there
are none. As the table shows, the selected width follows the variance of each
dataset: it keeps 123 of the 128 SIFT coordinates and 99 of the 100 GloVe
coordinates, whereas 507 of the 960 GIST coordinates, 507 to 891 coordinates of
the 512- to 1152-dimensional ImageNet embeddings, and 3579 of the 4096
DINOv3-7B~\cite{simeoni2025dinov3} coordinates suffice. Similarly, the anchor
stride ranges from 8 to 32. Sixteen anchors per query suffice, because they
provide a threshold with as few survivors as 128 anchors.

\begin{table*}[t]
\centering
\caption{Retina configuration and measurements per dataset.}
\label{A-tab:s-config}
\small
\setlength{\tabcolsep}{4pt}
\begin{tabular}{lrrrrrrrr}
\toprule
Data & $m$ & $L$ & Stride & Anchors & Survivors & Build (s) & Violations & QPS \\
\midrule
SIFT1M & 123 & 128 & 8 & 16 & 91 & 0.03 & 0 & 719.6k \\
GIST1M & 507 & 512 & 8 & 16 & 172 & 0.17 & 0 & 157.6k \\
GloVe & 99 & 128 & 8 & 16 & 88 & 0.35 & 0 & 642.8k \\
BIGANN-10M & 123 & 128 & 16 & 16 & 178 & 0.54 & 0 & 103.3k \\
BIGANN-100M & 59 & 64 & 16 & 16 & 4847 & 1.40 & 0 & 12.0k \\
Deep-image & 95 & 128 & 32 & 16 & 338 & 0.16 & 0 & 104.1k \\
ImageNet-CLIP-B & 507 & 512 & 16 & 16 & 167 & 0.13 & 0 & 195.1k \\
ImageNet-CLIP-L & 507 & 512 & 8 & 16 & 139 & 0.18 & 0 & 179.6k \\
ImageNet-SigLIP & 635 & 640 & 16 & 16 & 212 & 0.30 & 0 & 157.7k \\
ImageNet-DINOv2-B & 635 & 640 & 16 & 16 & 327 & 0.19 & 0 & 152.2k \\
ImageNet-DINOv2-L & 891 & 896 & 16 & 16 & 281 & 0.33 & 0 & 121.0k \\
ImageNet-DINOv3-L & 763 & 768 & 16 & 16 & 326 & 0.30 & 0 & 131.3k \\
ImageNet-DINOv3-7B & 3579 & 3584 & 16 & 16 & 257 & 2.47 & 0 & 36.2k \\
Text-to-image & 507 & 512 & 16 & 16 & 185 & 0.13 & 0 & 192.5k \\
DataComp images & 507 & 512 & 16 & 16 & 5193 & 1.14 & 0 & 14.1k \\
DataComp captions & 635 & 640 & 32 & 16 & 711 & 1.69 & 0 & 18.1k \\
\bottomrule
\end{tabular}
\end{table*}

\FloatBarrier
\section{Exhaustive Search and Scaling}
\label{A-sec:results}

\begin{table*}[t]
\centering
\caption{cuVS fp32 exhaustive search, QPS by selected candidates, 10 nearest kept ($^*$: not all queries exact).}
\label{A-tab:s-fp32}
\small
\setlength{\tabcolsep}{4pt}
\begin{tabular}{lrrrrrrrr}
\toprule
Data & $k=10$ & $k=32$ & $k=64$ & $k=65$ & $k=100$ & $k=128$ & $k=256$ & Best/$k{=}10$ \\
\midrule
SIFT1M & 44.9k & 44.9k & 44.3k & 65.5k & 65.6k & 65.0k & 59.4k & 1.46$\times$ \\
GIST1M & 6.1k$^*$ & 6.1k$^*$ & 6.0k$^*$ & 20.5k & 20.3k & 20.4k & 19.6k & 3.36$\times$ \\
GloVe & 43.1k$^*$ & 43.1k$^*$ & 42.3k$^*$ & 56.5k$^*$ & 58.2k$^*$ & 56.1k$^*$ & 52.5k$^*$ & 1.35$\times$ \\
BIGANN-10M & 4.5k & 4.5k & 4.5k & 6.7k & 6.7k & 6.6k & 6.1k & 1.48$\times$ \\
BIGANN-100M & 447 & 444 & 447 & 669 & 668 & 663 & 609 & 1.49$\times$ \\
Deep-image & 5.9k$^*$ & 5.9k$^*$ & 5.9k$^*$ & 7.3k$^*$ & 7.3k$^*$ & 7.2k$^*$ & 6.6k$^*$ & 1.24$\times$ \\
ImageNet-CLIP-B & 9.2k$^*$ & 9.2k$^*$ & 9.1k$^*$ & 25.8k$^*$ & 26.1k$^*$ & 25.7k$^*$ & 24.9k$^*$ & 2.85$\times$ \\
ImageNet-CLIP-L & 6.1k$^*$ & 6.1k$^*$ & 6.0k$^*$ & 19.6k$^*$ & 19.7k$^*$ & 19.5k$^*$ & 19.0k$^*$ & 3.24$\times$ \\
ImageNet-SigLIP & 4.1k$^*$ & 4.1k$^*$ & 4.0k$^*$ & 14.3k$^*$ & 14.4k$^*$ & 14.3k$^*$ & 14.0k$^*$ & 3.56$\times$ \\
ImageNet-DINOv2-B & 6.1k & 6.1k & 6.0k & 19.5k & 19.7k & 19.5k & 19.0k & 3.24$\times$ \\
ImageNet-DINOv2-L & 4.5k$^*$ & 4.6k$^*$ & 4.5k$^*$ & 15.5k$^*$ & 15.6k$^*$ & 15.5k$^*$ & 15.2k$^*$ & 3.44$\times$ \\
ImageNet-DINOv3-L & 4.5k$^*$ & 4.6k$^*$ & 4.5k$^*$ & 15.5k$^*$ & 15.7k$^*$ & 15.5k$^*$ & 15.2k$^*$ & 3.44$\times$ \\
ImageNet-DINOv3-7B & 1.1k$^*$ & 1.1k$^*$ & 1.1k$^*$ & 4.6k$^*$ & 4.6k$^*$ & 4.6k$^*$ & 4.6k$^*$ & 4.13$\times$ \\
Text-to-image & 9.1k & 9.2k & 8.9k & 25.7k & 26.0k & 25.6k & 24.7k & 2.84$\times$ \\
DataComp images & 610$^*$ & 611$^*$ & 597$^*$ & 2.0k$^*$ & 2.0k$^*$ & 2.0k$^*$ & 1.9k$^*$ & 3.26$\times$ \\
DataComp captions & 610$^*$ & 611$^*$ & 600$^*$ & 2.0k$^*$ & 2.0k & 2.0k$^*$ & 1.9k$^*$ & 3.25$\times$ \\
\bottomrule
\end{tabular}
\end{table*}

Table~\ref{A-tab:s-fp32} reports the full-precision exhaustive search of
cuVS~\cite{nvidia2025cuvs} for several numbers $k$ of selected candidates, of
which the 10 nearest form the answer. Section~\ref{sec:throughput} compares \certex\ with the
fastest entry of each row, which is a stronger baseline than the default $k=10$.
In terms of exactness, the integer-valued BIGANN prefixes and SIFT1M are exact in
every entry. In contrast, on the other datasets, the expanded fp32 distance leaves
0.01\% to 1.1\% of the queries inexact in the fastest configuration, whereas
\certex\ is exact everywhere. Among these datasets, the largest fraction occurs on
the near-duplicate images of DataComp~\cite{gadre2023datacomp}.

Table~\ref{A-tab:s-scaling} lists the fastest configuration of each width on the
BIGANN prefixes~\cite{jegou2011bigann}. On these prefixes, the width $m=123$ is
fastest up to 20M vectors, where its tight bound leaves about $ks$ survivors.
However, the anchor stage scores $N/s$ sampled vectors, so that its cost grows
with $N$ whereas the cost of the fovea does not. The best stride of the width 123
therefore grows from 8 to 32. Beyond 20M vectors, the retina pass itself dominates
the search. The width 59 therefore becomes the fastest configuration. Although the
width 59 admits more survivors, the fovea verifies them in a fraction of the time
that the shorter retina vectors save.

\begin{table*}[t]
\centering
\caption{BIGANN prefixes: best QPS (stride, survivors per query) per width, and the fastest exact cuVS fp32.}
\label{A-tab:s-scaling}
\small
\setlength{\tabcolsep}{4pt}
\begin{tabular}{lrrrr}
\toprule
$N$ & $m=59$ & $m=123$ & cuVS fp32 & Speedup \\
\midrule
1M & 466.3k (8, 816) & 722.3k (8, 89) & 65.7k & 10.99$\times$ \\
2M & 302.7k (8, 986) & 435.8k (8, 89) & 33.0k & 13.19$\times$ \\
5M & 158.7k (8, 1261) & 198.0k (16, 178) & 13.3k & 14.86$\times$ \\
10M & 93.2k (8, 1528) & 103.3k (16, 178) & 6.7k & 15.47$\times$ \\
20M & 49.1k (16, 3034) & 53.1k (32, 353) & 3.4k & 15.80$\times$ \\
50M & 23.1k (16, 3949) & 21.5k (32, 356) & 1.3k & 17.26$\times$ \\
100M & 12.0k (16, 4847) & 10.9k (64, 711) & 669 & 17.97$\times$ \\
\bottomrule
\end{tabular}
\end{table*}

Table~\ref{A-tab:s-fp16} reports the half-precision exhaustive baseline for every
evaluated $k'$. As for fp32, the cuVS selection of the $k'$ best candidates
becomes 2.1$\times$ to 25$\times$ faster when $k'$ exceeds 64. The best
configuration of this baseline therefore uses $k'=100$ or $k'=128$, even though
smaller values already recover the exact neighbors. Even this configuration
remains 1.8$\times$ to 12.9$\times$ slower than \certex.

\begin{table*}[t]
\centering
\caption{Half-precision exhaustive search with exact reranking, QPS by $k'$.}
\label{A-tab:s-fp16}
\small
\setlength{\tabcolsep}{4pt}
\begin{tabular}{lrrrrrrr}
\toprule
Data & $k'=10$ & $k'=32$ & $k'=64$ & $k'=100$ & $k'=128$ & $k'=256$ & Retina \\
\midrule
SIFT1M & 33.7k & 33.6k & 32.6k & 90.0k & 89.3k & 78.7k & 719.6k \\
GIST1M & 4.7k (99.3\%) & 4.7k & 4.5k & 64.4k & 64.9k & 57.3k & 157.6k \\
GloVe & 33.1k (99.3\%) & 33.1k & 32.2k & 74.1k & 71.1k & 65.2k & 642.8k \\
BIGANN-10M & 3.4k & 3.4k & 3.3k & 9.3k & 9.2k & 8.2k & 103.3k \\
BIGANN-100M & 336 & 337 & 329 & 934 & 923 & 825 & 12.0k \\
Deep-image & 4.6k (99.3\%) & 4.6k & 4.5k & 9.3k & 9.2k & 8.2k & 104.1k \\
ImageNet-CLIP-B & 6.9k (99.5\%) & 6.9k & 6.7k & 56.8k & 55.0k & 51.0k & 195.1k \\
ImageNet-CLIP-L & 4.6k (99.7\%) & 4.6k & 4.5k & 51.1k & 49.6k & 46.1k & 179.6k \\
ImageNet-SigLIP & 3.1k (99.7\%) & 3.1k & 3.0k & 43.7k & 42.5k & 39.7k & 157.7k \\
ImageNet-DINOv2-B & 4.6k (99.8\%) & 4.6k & 4.5k & 51.1k & 49.6k & 46.1k & 152.2k \\
ImageNet-DINOv2-L & 3.5k (99.7\%) & 3.5k & 3.4k & 46.0k & 44.7k & 41.8k & 121.0k \\
ImageNet-DINOv3-L & 3.5k (99.9\%) & 3.5k & 3.4k & 45.9k & 44.7k & 41.8k & 131.3k \\
ImageNet-DINOv3-7B & 862 (99.9\%) & 860 & 809 & 20.4k & 20.1k & 19.3k & 36.2k \\
Text-to-image & 6.9k (94.1\%) & 6.9k & 6.7k & 57.0k & 55.0k & 51.1k & 192.5k \\
DataComp images & 459 (98.5\%) & 459 (99.9\%) & 442 (99.9\%) & 6.1k (99.9\%) & 6.1k (99.9\%) & 5.5k (100.0\%) & 14.1k \\
DataComp captions & 458 (91.6\%) & 459 (100.0\%) & 449 (100.0\%) & 6.1k (100.0\%) & 6.1k & 5.5k & 18.1k \\
\bottomrule
\end{tabular}
\end{table*}

\FloatBarrier
\section{Comparison with the Indexes}
\label{A-sec:indexes}

Table~\ref{A-tab:s-rabitq} reports the best IVF-RaBitQ~\cite{gao2024rabitq}
configuration at each threshold from a grid of up to 154 configurations per
dataset. Against these configurations, \certex\ is faster at every threshold from
recall 0.9995 upward on all four descriptor and word datasets. For instance, at
recall 0.9995, \certex\ is faster by 12.6$\times$ on GloVe and 2.4$\times$ on
BIGANN-10M. On SIFT1M and GIST1M, \certex\ reaches this recall where IVF-RaBitQ
does not. When every query must be exact, \certex\ is faster by 46.2$\times$ on
SIFT1M and 4.4$\times$ on GIST1M. Because the error bound of RaBitQ is
probabilistic and its probing skips lists, its exact points hold for the
evaluated queries only, whereas the exactness of \certex\ holds for every query.

In Tables~\ref{A-tab:s-rabitq} to~\ref{A-tab:s-ivf}, recall counts the returned identifiers, whereas a query is exact
when its returned distances are (Section~\ref{A-sec:baselines}). Duplicate vectors with equal distances make the two
criteria differ. On SIFT1M, for instance, CAGRA is exact for 99.99\% of the queries but stays below recall 0.9995,
because it returns other identifiers among tied duplicates.

Table~\ref{A-tab:s-cagra} reports the best CAGRA~\cite{ootomo2024cagra}
configuration at each threshold from a grid of 18 configurations per dataset,
with the build time of each graph. By design, CAGRA computes full-precision
distances along its traversal, so that its errors come only from the vectors that
the traversal never visits. Enlarging the candidate list therefore raises recall
without guaranteeing exactness. Within the grid, CAGRA is exact for all queries only on SIFT1M among the
thirteen datasets. Finally, Table~\ref{A-tab:s-ivf} reports the
inverted-file indexes IVF-Flat~\cite{sivic2003videogoogle} and
IVF-PQ~\cite{jegou2011pq}, whose exact configurations probe almost every list and
thus approach exhaustive search over the lists.

\begin{table*}[t]
\centering
\caption{IVF-RaBitQ best throughput per threshold versus the exact retina.}
\label{A-tab:s-rabitq}
\small
\setlength{\tabcolsep}{4pt}
\begin{tabular}{llrlrr}
\toprule
Data & Threshold & RaBitQ QPS & bits, lists/probes, rerank & Retina QPS & Retina speedup \\
\midrule
SIFT1M & recall $\ge$ 0.9995 & not reached & -- & 719.6k & only exact method \\
 & exact queries $\ge$ 99.9\% & 210.2k & 3b, 1024/200, r6 & 719.6k & 3.42$\times$ \\
 & exact queries $\ge$ 99.99\% & 15.6k & 3b, 1024/300, r10 & 719.6k & 46.21$\times$ \\
 & all queries exact & 15.6k & 3b, 1024/300, r10 & 719.6k & 46.21$\times$ \\
\midrule
GIST1M & recall $\ge$ 0.9995 & not reached & -- & 157.6k & only exact method \\
 & exact queries $\ge$ 99.9\% & 46.1k & 8b, 1024/300, r2 & 157.6k & 3.42$\times$ \\
 & exact queries $\ge$ 99.99\% & 36.0k & 8b, 1024/500, r2 & 157.6k & 4.38$\times$ \\
 & all queries exact & 36.0k & 8b, 1024/500, r2 & 157.6k & 4.38$\times$ \\
\midrule
GloVe & recall $\ge$ 0.9995 & 51.1k & 8b, 1024/700, r4 & 642.8k & 12.59$\times$ \\
 & exact queries $\ge$ 99.9\% & 39.5k & 8b, 1024/1024, r4 & 642.8k & 16.29$\times$ \\
 & exact queries $\ge$ 99.99\% & not reached & -- & 642.8k & only exact method \\
 & all queries exact & not reached & -- & 642.8k & only exact method \\
\midrule
BIGANN-10M & recall $\ge$ 0.9995 & 42.7k & 8b, 4096/500, r3 & 103.3k & 2.42$\times$ \\
 & exact queries $\ge$ 99.9\% & not reached & -- & 103.3k & only exact method \\
 & exact queries $\ge$ 99.99\% & not reached & -- & 103.3k & only exact method \\
 & all queries exact & not reached & -- & 103.3k & only exact method \\
\midrule
Deep-image & recall $\ge$ 0.9995 & not reached & -- & 104.1k & only exact method \\
 & exact queries $\ge$ 99.9\% & not reached & -- & 104.1k & only exact method \\
 & exact queries $\ge$ 99.99\% & not reached & -- & 104.1k & only exact method \\
 & all queries exact & not reached & -- & 104.1k & only exact method \\
\midrule
ImageNet-CLIP-B & recall $\ge$ 0.9995 & 142.6k & 8b, 1024/150, r2 & 195.1k & 1.37$\times$ \\
 & exact queries $\ge$ 99.9\% & 84.3k & 8b, 1024/300, r3 & 195.1k & 2.32$\times$ \\
 & exact queries $\ge$ 99.99\% & 44.6k & 3b, 1024/700, r4 & 195.1k & 4.37$\times$ \\
 & all queries exact & 44.6k & 3b, 1024/700, r4 & 195.1k & 4.37$\times$ \\
\midrule
ImageNet-CLIP-L & recall $\ge$ 0.9995 & 86.6k & 3b, 1024/200, r3 & 179.6k & 2.07$\times$ \\
 & exact queries $\ge$ 99.9\% & 64.7k & 3b, 1024/300, r3 & 179.6k & 2.78$\times$ \\
 & exact queries $\ge$ 99.99\% & 44.1k & 3b, 1024/500, r3 & 179.6k & 4.07$\times$ \\
 & all queries exact & 44.1k & 3b, 1024/500, r3 & 179.6k & 4.07$\times$ \\
\midrule
ImageNet-SigLIP & recall $\ge$ 0.9995 & 48.2k & 3b, 1024/300, r2 & 157.7k & 3.27$\times$ \\
 & exact queries $\ge$ 99.9\% & 46.0k & 3b, 1024/300, r3 & 157.7k & 3.43$\times$ \\
 & exact queries $\ge$ 99.99\% & 17.4k & 3b, 1024/1024, r3 & 157.7k & 9.06$\times$ \\
 & all queries exact & 17.4k & 3b, 1024/1024, r3 & 157.7k & 9.06$\times$ \\
\midrule
ImageNet-DINOv2-B & recall $\ge$ 0.9995 & 39.0k & 3b, 1024/500, r2 & 152.2k & 3.90$\times$ \\
 & exact queries $\ge$ 99.9\% & 31.8k & 8b, 1024/500, r3 & 152.2k & 4.79$\times$ \\
 & exact queries $\ge$ 99.99\% & 21.7k & 3b, 1024/1024, r3 & 152.2k & 7.02$\times$ \\
 & all queries exact & 21.7k & 3b, 1024/1024, r3 & 152.2k & 7.02$\times$ \\
\midrule
ImageNet-DINOv2-L & recall $\ge$ 0.9995 & 23.0k & 3b, 1024/500, r3 & 121.0k & 5.27$\times$ \\
 & exact queries $\ge$ 99.9\% & 18.4k & 3b, 1024/700, r3 & 121.0k & 6.58$\times$ \\
 & exact queries $\ge$ 99.99\% & 14.2k & 3b, 1024/1024, r3 & 121.0k & 8.52$\times$ \\
 & all queries exact & 14.2k & 3b, 1024/1024, r3 & 121.0k & 8.52$\times$ \\
\midrule
ImageNet-DINOv3-L & recall $\ge$ 0.9995 & 17.9k & 3b, 1024/700, r4 & 131.3k & 7.35$\times$ \\
 & exact queries $\ge$ 99.9\% & 15.4k & 3b, 1024/1024, r2 & 131.3k & 8.53$\times$ \\
 & exact queries $\ge$ 99.99\% & 14.6k & 3b, 1024/1024, r3 & 131.3k & 9.02$\times$ \\
 & all queries exact & 14.6k & 3b, 1024/1024, r3 & 131.3k & 9.02$\times$ \\
\midrule
Text-to-image & recall $\ge$ 0.9995 & not reached & -- & 192.5k & only exact method \\
 & exact queries $\ge$ 99.9\% & not reached & -- & 192.5k & only exact method \\
 & exact queries $\ge$ 99.99\% & not reached & -- & 192.5k & only exact method \\
 & all queries exact & not reached & -- & 192.5k & only exact method \\
\bottomrule
\end{tabular}
\end{table*}
\begin{table*}[t]
\centering
\caption{CAGRA with graph degree 64 (build time in parentheses): best QPS per threshold and the speedup of \certex.}
\label{A-tab:s-cagra}
\small
\setlength{\tabcolsep}{4pt}
\begin{tabular}{llrlr}
\toprule
Data (build) & Threshold & CAGRA QPS & itopk, width & \certex\ speedup \\
\midrule
SIFT1M (2.2\,s) & recall $\ge$ 0.9995 & not reached & -- & only exact method \\
 & exact $\ge$ 99.9\% & 158.1k & 512, 4 & 4.55$\times$ \\
 & exact $\ge$ 99.99\% & 86.1k & 1024, 8 & 8.36$\times$ \\
 & all exact & 86.1k & 1024, 8 & 8.36$\times$ \\
\midrule
GIST1M (4.0\,s) & recall $\ge$ 0.9995 & not reached & -- & only exact method \\
 & exact $\ge$ 99.9\% & not reached & -- & only exact method \\
 & exact $\ge$ 99.99\% & not reached & -- & only exact method \\
 & all exact & not reached & -- & only exact method \\
\midrule
GloVe (2.4\,s) & recall $\ge$ 0.9995 & not reached & -- & only exact method \\
 & exact $\ge$ 99.9\% & not reached & -- & only exact method \\
 & exact $\ge$ 99.99\% & not reached & -- & only exact method \\
 & all exact & not reached & -- & only exact method \\
\midrule
BIGANN-10M (29.7\,s) & recall $\ge$ 0.9995 & 69.3k & 1024, 1 & 1.49$\times$ \\
 & exact $\ge$ 99.9\% & 69.3k & 1024, 1 & 1.49$\times$ \\
 & exact $\ge$ 99.99\% & not reached & -- & only exact method \\
 & all exact & not reached & -- & only exact method \\
\midrule
Deep-image (19.8\,s) & recall $\ge$ 0.9995 & 91.2k & 1024, 4 & 1.14$\times$ \\
 & exact $\ge$ 99.9\% & not reached & -- & only exact method \\
 & exact $\ge$ 99.99\% & not reached & -- & only exact method \\
 & all exact & not reached & -- & only exact method \\
\midrule
ImageNet-CLIP-B (3.4\,s) & recall $\ge$ 0.9995 & 54.0k & 512, 8 & 3.62$\times$ \\
 & exact $\ge$ 99.9\% & 21.5k & 1024, 8 & 9.07$\times$ \\
 & exact $\ge$ 99.99\% & not reached & -- & only exact method \\
 & all exact & not reached & -- & only exact method \\
\midrule
ImageNet-CLIP-L (4.1\,s) & recall $\ge$ 0.9995 & 14.5k & 1024, 4 & 12.38$\times$ \\
 & exact $\ge$ 99.9\% & not reached & -- & only exact method \\
 & exact $\ge$ 99.99\% & not reached & -- & only exact method \\
 & all exact & not reached & -- & only exact method \\
\midrule
ImageNet-SigLIP (6.3\,s) & recall $\ge$ 0.9995 & not reached & -- & only exact method \\
 & exact $\ge$ 99.9\% & not reached & -- & only exact method \\
 & exact $\ge$ 99.99\% & not reached & -- & only exact method \\
 & all exact & not reached & -- & only exact method \\
\midrule
ImageNet-DINOv2-B (5.3\,s) & recall $\ge$ 0.9995 & 16.3k & 1024, 1 & 9.36$\times$ \\
 & exact $\ge$ 99.9\% & 16.3k & 1024, 1 & 9.36$\times$ \\
 & exact $\ge$ 99.99\% & not reached & -- & only exact method \\
 & all exact & not reached & -- & only exact method \\
\midrule
ImageNet-DINOv2-L (6.9\,s) & recall $\ge$ 0.9995 & 12.1k & 1024, 4 & 10.01$\times$ \\
 & exact $\ge$ 99.9\% & 12.1k & 1024, 4 & 10.01$\times$ \\
 & exact $\ge$ 99.99\% & not reached & -- & only exact method \\
 & all exact & not reached & -- & only exact method \\
\midrule
ImageNet-DINOv3-L (6.9\,s) & recall $\ge$ 0.9995 & not reached & -- & only exact method \\
 & exact $\ge$ 99.9\% & not reached & -- & only exact method \\
 & exact $\ge$ 99.99\% & not reached & -- & only exact method \\
 & all exact & not reached & -- & only exact method \\
\midrule
Text-to-image (3.4\,s) & recall $\ge$ 0.9995 & not reached & -- & only exact method \\
 & exact $\ge$ 99.9\% & not reached & -- & only exact method \\
 & exact $\ge$ 99.99\% & not reached & -- & only exact method \\
 & all exact & not reached & -- & only exact method \\
\midrule
ImageNet-DINOv3-7B (25.9\,s) & recall $\ge$ 0.9995 & not reached & -- & only exact method \\
 & exact $\ge$ 99.9\% & not reached & -- & only exact method \\
 & exact $\ge$ 99.99\% & not reached & -- & only exact method \\
 & all exact & not reached & -- & only exact method \\
\midrule
\bottomrule
\end{tabular}
\end{table*}
\begin{table*}[t]
\centering
\caption{Speedup of \certex\ over IVF-Flat and IVF-PQ with exact reranking at three thresholds (n.r.: not reached).}
\label{A-tab:s-ivf}
\small
\setlength{\tabcolsep}{4pt}
\begin{tabular}{lrrrrrrrr}
\toprule
Data & Flat (s) & PQ (s) & \multicolumn{3}{c}{IVF-Flat: R 0.9995, E 99.9, E 100} & \multicolumn{3}{c}{IVF-PQ: R 0.9995, E 99.9, E 100} \\
\midrule
SIFT1M & 0.2 & 0.7 & n.r. & 27.6$\times$ & 30.9$\times$ & n.r. & 10.4$\times$ & 24.5$\times$ \\
GIST1M & 0.7 & 2.5 & n.r. & 55.4$\times$ & 55.4$\times$ & n.r. & n.r. & n.r. \\
GloVe & 0.3 & 0.6 & 27.2$\times$ & 27.2$\times$ & n.r. & 41.5$\times$ & 41.5$\times$ & 41.5$\times$ \\
BIGANN-10M & 1.2 & 2.1 & 41.9$\times$ & 41.9$\times$ & 41.9$\times$ & 3.7$\times$ & 3.7$\times$ & 8.4$\times$ \\
Deep-image & 1.1 & 1.7 & 31.9$\times$ & 32.3$\times$ & 32.3$\times$ & 3.0$\times$ & 7.0$\times$ & n.r. \\
ImageNet-CLIP-B & 0.5 & 1.5 & 41.0$\times$ & 41.0$\times$ & 41.0$\times$ & 7.8$\times$ & 18.7$\times$ & n.r. \\
ImageNet-CLIP-L & 0.7 & 2.1 & 56.9$\times$ & 56.9$\times$ & 56.9$\times$ & 23.6$\times$ & 23.6$\times$ & n.r. \\
ImageNet-SigLIP & 0.9 & 2.7 & 74.4$\times$ & 74.4$\times$ & 74.4$\times$ & 52.2$\times$ & 55.0$\times$ & n.r. \\
ImageNet-DINOv2-B & 0.7 & 2.1 & 48.3$\times$ & 48.3$\times$ & 48.3$\times$ & 20.2$\times$ & 40.4$\times$ & 40.4$\times$ \\
ImageNet-DINOv2-L & 0.9 & 2.8 & 50.9$\times$ & 50.9$\times$ & 50.9$\times$ & 32.9$\times$ & 65.7$\times$ & n.r. \\
ImageNet-DINOv3-L & 0.9 & 2.7 & 55.1$\times$ & 55.1$\times$ & n.r. & 75.8$\times$ & 79.5$\times$ & 79.5$\times$ \\
Text-to-image & 0.5 & 1.5 & 40.5$\times$ & 40.5$\times$ & 40.5$\times$ & n.r. & n.r. & n.r. \\
ImageNet-DINOv3-7B & 2.9 & 11.4 & 84.9$\times$ & 84.9$\times$ & n.r. & 124.1$\times$ & 124.1$\times$ & n.r. \\
DataComp images & 9.1 & 14.0 & n.r. & n.r. & n.r. & n.r. & n.r. & n.r. \\
DataComp captions & 5.3 & 9.8 & n.r. & 55.1$\times$ & n.r. & n.r. & n.r. & n.r. \\
\bottomrule
\end{tabular}
\end{table*}

\section{Embeddings and Sensitivity}
\label{A-sec:sens}

\begin{table*}[t]
\centering
\caption{Exact search on ImageNet-1K embeddings, 1.28M vectors and $10^4$ queries (QPS and speedup of \certex).}
\label{A-tab:imagenet}
\small
\setlength{\tabcolsep}{4pt}
\begin{tabular}{lrrrrrr}
\toprule
Encoder & $D$ & $m$ & Surv. & \certex & cuVS fp32 & fp16+rerank \\
\midrule
CLIP ViT-B/16 & 512 & 507 & 167 & 195.1k & 26.1k (7.5$\times$) & 56.8k (3.44$\times$) \\
CLIP ViT-L/14 & 768 & 507 & 139 & 179.6k & 19.7k (9.1$\times$) & 51.1k (3.51$\times$) \\
SigLIP so400m & 1152 & 635 & 212 & 157.7k & 14.4k (10.9$\times$) & 43.7k (3.61$\times$) \\
DINOv2 ViT-B/14 & 768 & 635 & 327 & 152.2k & 19.7k (7.7$\times$) & 51.1k (2.98$\times$) \\
DINOv2 ViT-L/14 & 1024 & 891 & 281 & 121.0k & 15.6k (7.7$\times$) & 46.0k (2.63$\times$) \\
DINOv3 ViT-L/16 & 1024 & 763 & 326 & 131.3k & 15.7k (8.4$\times$) & 45.9k (2.86$\times$) \\
DINOv3 ViT-7B/16 & 4096 & 3579 & 257 & 36.2k & 4.6k (7.9$\times$) & 20.4k (1.77$\times$) \\
CLIP B/16 text queries & 512 & 507 & 185 & 192.5k & 26.0k (7.4$\times$) & 57.0k (3.38$\times$) \\
\bottomrule
\end{tabular}
\end{table*}

Table~\ref{A-tab:imagenet} lists the measurements behind
Figure~\ref{fig:enc}(a). On the seven encoders, \certex\ is 1.8$\times$ to
3.6$\times$ faster than half-precision exhaustive search with exact reranking.
\certex\ is also faster than IVF-RaBitQ~\cite{gao2024rabitq}, which is the best
index on every encoder: by 2.3$\times$ and 2.8$\times$ on the
CLIP~\cite{radford2021clip} encoders, 3.4$\times$ to 6.6$\times$ on
SigLIP~\cite{zhai2023siglip} and DINOv2~\cite{oquab2024dinov2}, 8.5$\times$ on
DINOv3-L, and 10.7$\times$ on DINOv3-7B~\cite{simeoni2025dinov3}. The table also
shows a limit of the width: a width that is too narrow for a dataset overflows
the survivor list, as for DINOv2 with $m=251$ at stride 64. \certex\ detects this
case and reports it instead of returning an incomplete answer, so that the width
can be set safely by measurement.

Table~\ref{A-tab:s-sens} varies the anchor stride and the number of anchors on
SIFT1M and DINOv2. On both datasets, the number of anchors changes the throughput
by less than 10\% beyond ten anchors, because the threshold is the $k$-th exact
distance among the anchors, which the best few sampled vectors already determine. The stride is therefore the only anchor parameter that
matters: small strides suit these 1M-vector datasets, whereas larger strides suit
larger databases (Table~\ref{A-tab:s-scaling}).

\begin{table*}[t]
\centering
\caption{QPS (survivors per query) by anchor stride $s$ and number of anchors; every configuration is exact.}
\label{A-tab:s-sens}
\small
\setlength{\tabcolsep}{4pt}
\begin{tabular}{lrrrrrr}
\toprule
Data & $s$ & 10 & 16 & 32 & 64 & 128 \\
\midrule
SIFT1M & 8 & 716.5k (91) & 713.2k (91) & 714.0k (91) & 710.9k (91) & 704.8k (91) \\
 & 16 & 620.5k (182) & 625.6k (181) & 624.9k (181) & 622.0k (181) & 615.6k (181) \\
 & 64 & 351.1k (724) & 349.3k (721) & 350.6k (721) & 350.4k (721) & 355.1k (721) \\
 & 256 & 116.4k (3055) & 116.7k (3047) & 116.9k (3047) & 120.7k (2907) & 128.1k (2837) \\
\midrule
ImageNet-DINOv2 & 8 & 152.2k (191) & 140.9k (182) & 140.9k (181) & 147.0k (181) & 150.8k (181) \\
 & 16 & 149.7k (344) & 152.4k (327) & 149.9k (327) & 149.1k (327) & 149.4k (327) \\
 & 64 & 91.6k (1288) & 98.1k (1153) & 98.3k (1148) & 98.0k (1147) & 97.5k (1147) \\
 & 256 & 32.0k (10151) & 39.8k (7245) & 40.1k (7132) & 40.1k (7129) & 41.2k (6784) \\
\bottomrule
\end{tabular}
\end{table*}

Table~\ref{A-tab:s-breakdown} lists every timing behind
Figure~\ref{fig:breakdown}(a) and Figure~\ref{fig:breakdown}(b) for one batch
of 2048 queries. In addition to the bars of Figure~\ref{fig:breakdown}(a), the
table reports the GEMM followed by a separate threshold test and the retina pass
in two variants. Specifically, the separate test writes all products to memory
before it compares them with the threshold, whereas the variants with TF32 and
with the full width $m=D$ give the factors of Figure~\ref{fig:breakdown}(b).

\begin{table*}[t]
\centering
\caption{Time in ms per batch of 2048 queries for every pass and search stage.}
\label{A-tab:s-breakdown}
\small
\setlength{\tabcolsep}{4pt}
\begin{tabular}{lrrrrrrrrr}
\toprule
& \multicolumn{3}{c}{exhaustive pass} & \multicolumn{3}{c}{retina pass} & \multicolumn{3}{c}{\certex\ search} \\ \cmidrule(lr){2-4}\cmidrule(lr){5-7}\cmidrule(lr){8-10}
Data & GEMM + top-$k$ & GEMM + test & GEMM & TF32 & $m=D$ & \certex & anchors & fovea & total \\
\midrule
SIFT1M & 37.2 & 21.3 & 4.1 & 2.8 & 3.5 & 2.3 & 1.4 & 0.8 & 5.0 \\
ImageNet-DINOv2 & 60.4 & 40.1 & 18.2 & 16.1 & 12.3 & 10.8 & 1.7 & 1.6 & 14.5 \\
BIGANN-100M & 3600.6 & 2008.1 & 289.1 & 171.4 & 314.5 & 148.3 & 15.8 & 5.7 & 170.5 \\
\bottomrule
\end{tabular}
\end{table*}

\begin{figure}[!htb]
\centering
\includegraphics[width=\linewidth]{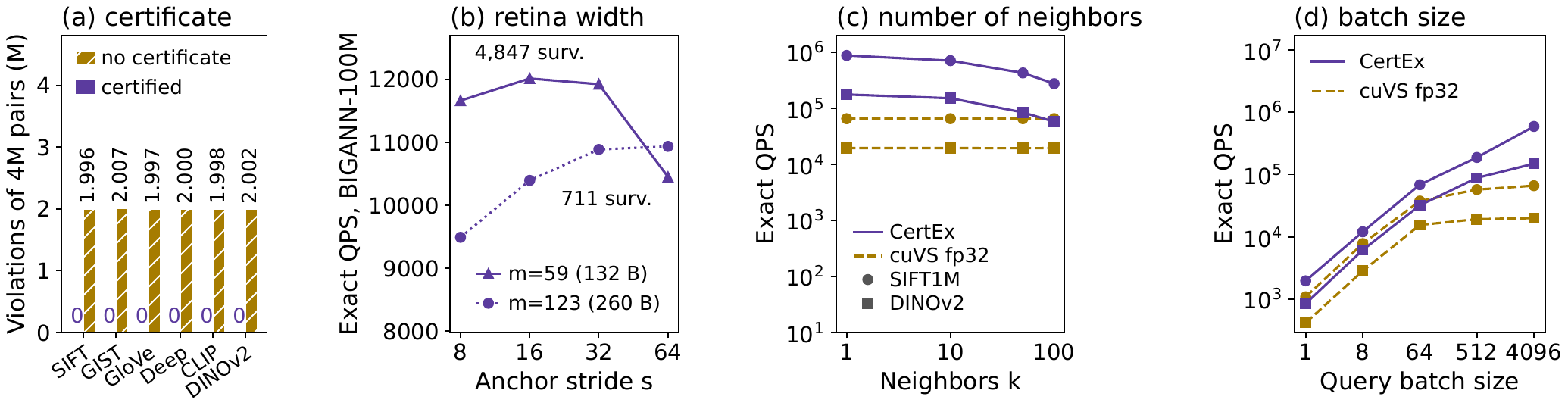}
\caption{Certificate, retina width, number of neighbors, and query batch size.}
\label{A-fig:ablation}
\end{figure}

Figure~\ref{A-fig:ablation} complements Figure~\ref{fig:breakdown}(c) with the
three parameters that affect only the cost of \certex. For the width, Figure~\ref{A-fig:ablation}(b)
shows that on BIGANN-100M the width 59 reaches 12.0k queries per second at strides
16 and 32, where its 4.8k to 7.8k survivors per query cost the fovea little. In
contrast, the width 123 leaves only 92 to 711 survivors and reaches 10.9k queries
per second at stride 64. The two curves cross between strides 32 and 64, because a
larger stride shortens the anchor stage of both widths but multiplies the
survivors of the narrow width, whose bound is looser.

For the number of neighbors, Figure~\ref{A-fig:ablation}(c) shows that with twice as many anchors as
neighbors, \certex\ remains exact for every $k$ from 1 to 100. The throughput of
\certex\ decreases from 885k to 278k queries per second on SIFT1M and from 178k to
58k on DINOv2, so that the speedup over exhaustive search ranges from
13.5$\times$ and 9.0$\times$ for $k=1$ to 4.2$\times$ and 3.0$\times$ for
$k=100$. For the batch size, Figure~\ref{A-fig:ablation}(d) covers single queries and small batches.
For batches of up to 64 queries, \certex\ replays the whole search as one CUDA
graph, whose shapes depend only on the batch size and a fixed capacity of the
survivor list. A batch that exceeds this capacity falls back to the general path.
A single query thus takes 0.51 ms on SIFT1M and 1.18 ms on DINOv2, which is
1.8$\times$ and 2.0$\times$ faster than exhaustive search. For larger batches, the
advantage grows to 8.9$\times$ and 7.5$\times$ at 4096 queries per batch.

\FloatBarrier
\FloatBarrier
\section{Duplicate Detection}
\label{A-sec:join-supp}

\begin{table}[t]
\centering
\caption{Duplicate pairs and join time (s) on DINOv2.}
\label{A-tab:join}
\footnotesize
\setlength{\tabcolsep}{3pt}
\begin{tabular}{lrrrrr}
\toprule
$\tau$ & Pairs & \certex & fp32 join & in 10-NN & in RaBitQ 10-NN \\
\midrule
0.90 & 3,997,988 & 2.5 & 56.9 & 984,288 & 984,189 \\
0.95 & 123,602 & 1.9 & 56.9 & 81,222 & 81,219 \\
0.98 & 7,783 & 1.5 & 56.9 & 7,783 & 7,783 \\
\bottomrule
\end{tabular}
\end{table}

Table~\ref{A-tab:join} lists the exact self-join of the
DINOv2~\cite{oquab2024dinov2} embeddings. At every threshold, the self-join
returns exactly the pairs of an exhaustive fp32 join in 1.5 to 2.5 seconds. The
table also lists the pairs that an exact 10-NN graph and an
IVF-RaBitQ~\cite{gao2024rabitq} 10-NN graph contain. In contrast to these graphs,
the exact self-join recovers at $\tau=0.90$ the many semantic duplicates beyond
the tenth neighbor. The self-join also finishes before IVF-RaBitQ has built its
graph, which takes 9.5 seconds.

Table~\ref{A-tab:semdedup} compares the exact self-join with the cluster restriction
of SemDeDup~\cite{abbas2023semdedup}, which compares only the pairs inside each
k-means cluster to avoid the quadratic join. With one cluster per ImageNet class,
the restriction misses only 0.11\% to 0.30\% of all pairs. However, these missed
pairs concentrate in a few classes: the 10\% of classes with the most misses hold
85\% to 100\% of them, so that the worst classes lose up to 26\% of their pairs.
At web scale, the exact self-join finds all 31.4 billion pairs among 1.97M of the
12.79M DataComp images~\cite{gadre2023datacomp} at $\tau=0.95$. These pairs
include the duplicates of 41{,}875 to 56{,}416 images whose partners all lie in
other clusters and, at $\tau=0.98$, 6{,}660 to 8{,}557 copies of the same
photograph.

\begin{table}[t]
\centering
\caption{Duplicates that the clusters of SemDeDup miss, relative to the exact self-join of \certex.}
\label{A-tab:semdedup}
\footnotesize
\setlength{\tabcolsep}{3pt}
\begin{tabular}{lrrrrrrr}
\toprule
& & \multicolumn{3}{c}{Missed pairs (\%)} & \multicolumn{3}{c}{Undetectable images} \\ \cmidrule(lr){3-5}\cmidrule(lr){6-8} Data & Clusters & 0.90 & 0.95 & 0.98 & 0.90 & 0.95 & 0.98 \\
\midrule
ImageNet & 1,000 & 0.17 & 0.11 & 0.30 & 1,733 & 204 & 40 \\
 & 10,000 & 29.98 & 10.90 & 1.02 & 12,008 & 1,983 & 107 \\
 & 50,000 & 66.37 & 33.09 & 1.12 & 21,507 & 3,658 & 82 \\
\midrule
DataComp & 1,500 & -- & 2.45 & 1.06 & -- & 41,875 & 6,660 \\
 & 10,000 & -- & 4.49 & 1.44 & -- & 48,517 & 7,984 \\
 & 50,000 & -- & 4.96 & 1.68 & -- & 56,416 & 8,557 \\
\bottomrule
\end{tabular}
\end{table}

\begin{figure}[!htb]
\centering
\includegraphics[width=\linewidth]{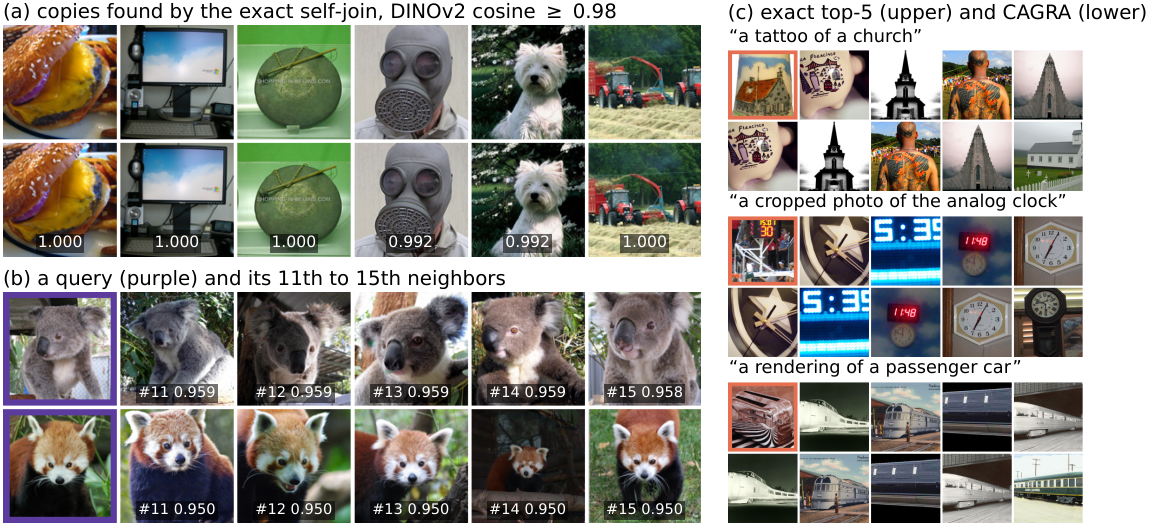}
\caption{Real ImageNet examples: copies, semantic duplicates beyond the tenth neighbor, and text queries.}
\label{A-fig:qual}
\end{figure}

Figure~\ref{A-fig:qual} extends Figure~\ref{fig:qual} with more examples. In Figure~\ref{A-fig:qual}(a),
the pairs at a cosine similarity of 0.98 and above are copies of the same
photograph inside the ImageNet~\cite{deng2009imagenet} training set. Below this
level, the pairs are semantic duplicates. For instance, a query such as the koala
or the red panda of Figure~\ref{A-fig:qual}(b) has hundreds of them, so that its 11th to 15th neighbors
lie beyond any 10-NN graph. For the text queries of Figure~\ref{A-fig:qual}(c), \certex\ returns the exact
neighbors of every prompt, whereas the most accurate CAGRA~\cite{ootomo2024cagra}
configuration loses an exact top-5 neighbor for 108 of the 10{,}000 prompts.

\FloatBarrier
\section{Exhaustive Search in Practice}
\label{A-sec:tools}

\begin{figure}[!htb]
\centering
\includegraphics[width=\linewidth]{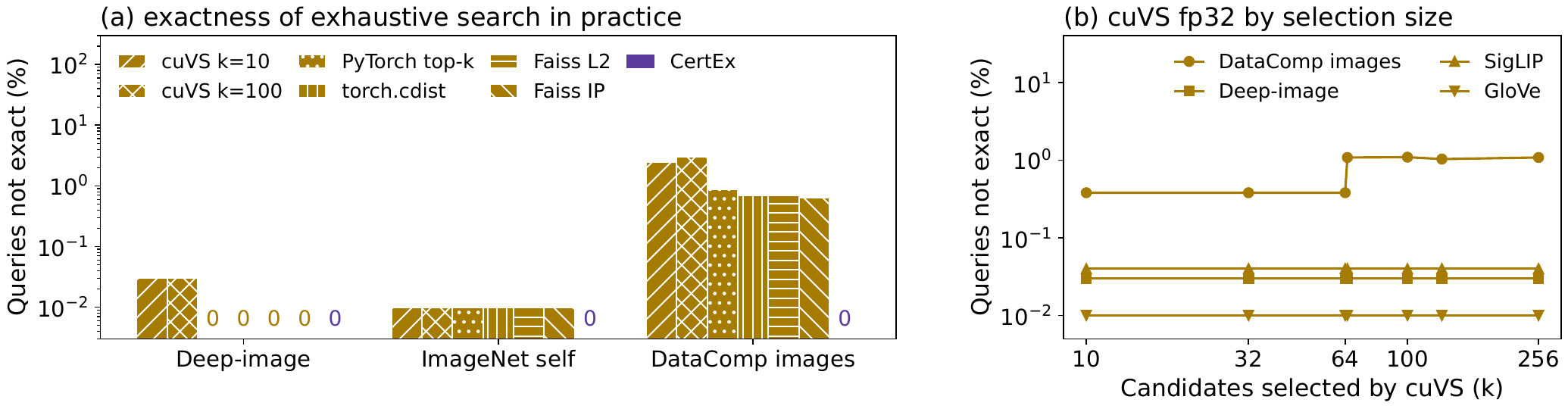}
\caption{Inexact queries of the fp32 tools of practice, and the exactness of cuVS against its selection size.}
\label{A-fig:tools}
\end{figure}

Figure~\ref{A-fig:tools}(a) compares \certex\ with the six fp32 tools that practice
uses to compute exact neighbors, against the fp64 true neighbors. Specifically,
the tools are the brute-force search of cuVS~\cite{nvidia2025cuvs} at $k=10$ and
at its faster selection with $k=100$, an inner-product matrix with top-$k$
selection in PyTorch~\cite{paszke2019pytorch}, \texttt{torch.cdist} with top-$k$
selection, and the flat L2 and inner-product indexes of
Faiss~\cite{johnson2021faiss}. Among them, the PyTorch variant is the evaluation
idiom of CLIP-style code~\cite{radford2021clip}. Of all these methods, only
\certex\ is exact on all three datasets. On the 12.79M DataComp
images~\cite{gadre2023datacomp}, the fp32 tools differ from the exact neighbors
for 0.64\% to 3.0\% of the queries. These differences persist at a relative
tolerance of $10^{-5}$. The differences are therefore not near-ties but neighbors
that the cancellation of the expanded distance displaces among near-duplicate web
images. Smaller differences occur on the other datasets: on Deep-image, cuVS
differs for 0.03\% of the queries. Similarly, on 10{,}000 ImageNet training images
queried against the training set, every fp32 tool differs for one query.

Figure~\ref{A-fig:tools}(b) shows that the fp32 error lies in the distances rather
than in the selection. Specifically, the fraction of inexact queries of cuVS is
flat in $k$ on Deep-image, SigLIP, and GloVe. Only on DataComp does the fraction
rise, from 0.38\% at $k\le64$ to 1.1\% for the faster selection above $k=64$.
Table~\ref{A-tab:s-batch} shows that the fraction also depends on the number of
queries per call. For instance, on DataComp, cuVS answers 2.5\% to 2.7\% of the
queries inexactly in calls of 100 to 2048 queries but only 0.38\% in calls of
10{,}000 queries, because cuVS chooses a different matrix multiplication routine
for each batch shape. The speed comparisons of Section~\ref{sec:eval} use calls of 10{,}000
queries. In contrast, the other datasets do not change with the call size.

\begin{table*}[t]
\centering
\caption{Inexact queries (\%) of cuVS fp32 exhaustive search by the number of queries per call.}
\label{A-tab:s-batch}
\small
\setlength{\tabcolsep}{5pt}
\begin{tabular}{lrrrrrrrr}
\toprule
& \multicolumn{4}{c}{$k=10$, queries per call} & \multicolumn{4}{c}{$k=100$, queries per call} \\ \cmidrule(lr){2-5}\cmidrule(lr){6-9} Data & 100 & 1,000 & 2,048 & 10,000 & 100 & 1,000 & 2,048 & 10,000 \\
\midrule
DataComp images & 2.67 & 2.51 & 2.59 & 0.38 & 2.97 & 3.02 & 3.02 & 1.09 \\
Deep-image & 0.03 & 0.03 & 0.03 & 0.03 & 0.03 & 0.03 & 0.03 & 0.03 \\
ImageNet-SigLIP & 0.04 & 0.04 & 0.04 & 0.04 & 0.04 & 0.04 & 0.04 & 0.04 \\
GloVe & 0.01 & 0.01 & 0.01 & 0.01 & 0.01 & 0.01 & 0.01 & 0.01 \\
\bottomrule
\end{tabular}
\end{table*}

The ground truth that public benchmarks ship is correct at depth 10 for all six
benchmarks that we audit: SIFT1M, GIST1M, GloVe, Deep-image, BIGANN-1M, and
BIGANN-10M~\cite{jegou2011pq,aumuller2020annbench,jegou2011bigann}. At depth 100,
however, one GIST query and three Deep-image queries are wrong. Exact references
therefore matter most for new embedding datasets rather than for these
established benchmarks, because near duplicates are common in new datasets
and the exactness of an fp32 reference depends on how the tool is called. For such
datasets, \certex\ computes the references exactly. At depth 100, \certex\ is
also 3.0$\times$ to 4.2$\times$ faster than the fastest cuVS fp32 configuration on
SIFT1M and DINOv2 (Figure~\ref{A-fig:ablation}(c)).

\FloatBarrier


{\small
\begin{thebibliography}{75}
\providecommand{\natexlab}[1]{#1}
\providecommand{\url}[1]{\texttt{#1}}
\expandafter\ifx\csname urlstyle\endcsname\relax
  \providecommand{\doi}[1]{doi: #1}\else
  \providecommand{\doi}{doi: \begingroup \urlstyle{rm}\Url}\fi

\bibitem[Abbas et~al.(2023)Abbas, Tirumala, Simig, Ganguli, and
  Morcos]{abbas2023semdedup}
Amro Abbas, Kushal Tirumala, Dániel Simig, Surya Ganguli, and Ari~S. Morcos.
\newblock {SemDeDup: Data-Efficient Learning at Web-Scale through Semantic
  Deduplication}.
\newblock In \emph{ICLR 2023 Workshop on Multimodal Representation Learning
  (MRL)}, 2023.
\newblock \href{https://iclr.cc/virtual/2023/13806}{[URL]}.

\bibitem[Abid and Timofte(2026)]{abid2026gazemo}
Mian Muhammad~Naeem Abid and Radu Timofte.
\newblock Gazemo: Mimicking human saccades via foveal-peripheral feature
  modeling for lightweight semantic segmentation.
\newblock In \emph{Proceedings of the IEEE/CVF Conference on Computer Vision
  and Pattern Recognition (CVPR) Findings}, pages 7293--7303, 2026.
\newblock \href{https://doi.org/10.1109/CVPRF71909.2026.00700}{[URL]}.

\bibitem[Arandjelovic et~al.(2016)Arandjelovic, Gronat, Torii, Pajdla, and
  Sivic]{arandjelovic2016netvlad}
Relja Arandjelovic, Petr Gronat, Akihiko Torii, Tomas Pajdla, and Josef Sivic.
\newblock {NetVLAD: CNN Architecture for Weakly Supervised Place Recognition}.
\newblock In \emph{2016 IEEE Conference on Computer Vision and Pattern
  Recognition (CVPR)}, pages 5297--5307, 2016.
\newblock \href{https://doi.org/10.1109/cvpr.2016.572}{[URL]}.

\bibitem[Asai et~al.(2024)Asai, Wu, Wang, Sil, and Hajishirzi]{asai2024selfrag}
Akari Asai, Zeqiu Wu, Yizhong Wang, Avirup Sil, and Hannaneh Hajishirzi.
\newblock {Self-RAG: Learning to Retrieve, Generate, and Critique through
  Self-Reflection}.
\newblock In \emph{International Conference on Learning Representations
  (ICLR)}, 2024.
\newblock \href{https://openreview.net/forum?id=hSyW5go0v8}{[URL]}.

\bibitem[Aumüller et~al.(2020)Aumüller, Bernhardsson, and
  Faithfull]{aumuller2020annbench}
Martin Aumüller, Erik Bernhardsson, and Alexander Faithfull.
\newblock {ANN-Benchmarks: A benchmarking tool for approximate nearest neighbor
  algorithms}.
\newblock \emph{Information Systems}, 87:\penalty0 101374, 2020.
\newblock \href{https://doi.org/10.1016/j.is.2019.02.006}{[URL]}.

\bibitem[Babenko and Lempitsky(2012)]{babenko2012imi}
A. Babenko and V. Lempitsky.
\newblock {The inverted multi-index}.
\newblock In \emph{2012 IEEE Conference on Computer Vision and Pattern
  Recognition}, pages 3069--3076, 2012.
\newblock \href{https://doi.org/10.1109/CVPR.2012.6248038}{[URL]}.

\bibitem[Bentley(1975)]{bentley1975kdtree}
Jon~Louis Bentley.
\newblock {Multidimensional binary search trees used for associative
  searching}.
\newblock \emph{Communications of the ACM}, 18\penalty0 (9):\penalty0 509--517,
  1975.
\newblock \href{https://doi.org/10.1145/361002.361007}{[URL]}.

\bibitem[Beyer et~al.(1999)Beyer, Goldstein, Ramakrishnan, and
  Shaft]{beyer1999meaningful}
Kevin Beyer, Jonathan Goldstein, Raghu Ramakrishnan, and Uri Shaft.
\newblock {When Is ``Nearest Neighbor'' Meaningful?}
\newblock In \emph{Lecture Notes in Computer Science}, pages 217--235, 1999.
\newblock \href{https://doi.org/10.1007/3-540-49257-7_15}{[URL]}.

\bibitem[Beygelzimer et~al.(2006)Beygelzimer, Kakade, and
  Langford]{beygelzimer2006cover}
Alina Beygelzimer, Sham Kakade, and John Langford.
\newblock {Cover trees for nearest neighbor}.
\newblock In \emph{Proceedings of the 23rd international conference on Machine
  learning}, pages 97--104, 2006.
\newblock \href{https://doi.org/10.1145/1143844.1143857}{[URL]}.

\bibitem[Blanchard et~al.(2020)Blanchard, Higham, Lopez, Mary, and
  Pranesh]{blanchard2020tensorcore}
Pierre Blanchard, Nicholas~J. Higham, Florent Lopez, Theo Mary, and Srikara
  Pranesh.
\newblock {Mixed Precision Block Fused Multiply-Add: Error Analysis and
  Application to GPU Tensor Cores}.
\newblock \emph{SIAM Journal on Scientific Computing}, 42\penalty0
  (3):\penalty0 C124--C141, 2020.
\newblock \href{https://doi.org/10.1137/19m1289546}{[URL]}.

\bibitem[Borgeaud et~al.(2022)Borgeaud, Mensch, Hoffmann, Cai, Rutherford,
  Millican, van~den Driessche, Lespiau, Damoc, Clark, de~Las~Casas, Guy,
  Menick, Ring, Hennigan, Huang, Maggiore, Jones, Cassirer, Brock, Paganini,
  Irving, Vinyals, Osindero, Simonyan, Rae, Elsen, and
  Sifre]{borgeaud2022retro}
Sebastian Borgeaud, Arthur Mensch, Jordan Hoffmann, Trevor Cai, Eliza
  Rutherford, Katie Millican, George van~den Driessche, Jean-Baptiste Lespiau,
  Bogdan Damoc, Aidan Clark, Diego de Las~Casas, Aurelia Guy, Jacob Menick,
  Roman Ring, Tom Hennigan, Saffron Huang, Loren Maggiore, Chris Jones, Albin
  Cassirer, Andy Brock, Michela Paganini, Geoffrey Irving, Oriol Vinyals, Simon
  Osindero, Karen Simonyan, Jack~W. Rae, Erich Elsen, and Laurent Sifre.
\newblock {Improving Language Models by Retrieving from Trillions of Tokens}.
\newblock In \emph{International Conference on Machine Learning (ICML)}, 2022.
\newblock \href{https://proceedings.mlr.press/v162/borgeaud22a.html}{[URL]}.

\bibitem[Caron et~al.(2021)Caron, Touvron, Misra, Jégou, Mairal, Bojanowski,
  and Joulin]{caron2021dino}
Mathilde Caron, Hugo Touvron, Ishan Misra, Hervé Jégou, Julien Mairal, Piotr
  Bojanowski, and Armand Joulin.
\newblock {Emerging Properties in Self-Supervised Vision Transformers}.
\newblock In \emph{IEEE/CVF International Conference on Computer Vision
  (ICCV)}, 2021.
\newblock \href{https://doi.org/10.1109/ICCV48922.2021.00951}{[URL]}.

\bibitem[Curcio et~al.(1990)Curcio, Sloan, Kalina, and
  Hendrickson]{curcio1990photoreceptor}
Christine~A. Curcio, Kenneth~R. Sloan, Robert~E. Kalina, and Anita~E.
  Hendrickson.
\newblock {Human photoreceptor topography}.
\newblock \emph{Journal of Comparative Neurology}, 292\penalty0 (4):\penalty0
  497--523, 1990.
\newblock \href{https://doi.org/10.1002/cne.902920402}{[URL]}.

\bibitem[Datar et~al.(2004)Datar, Immorlica, Indyk, and
  Mirrokni]{datar2004pstable}
Mayur Datar, Nicole Immorlica, Piotr Indyk, and Vahab~S. Mirrokni.
\newblock {Locality-sensitive hashing scheme based on p-stable distributions}.
\newblock In \emph{Proceedings of the twentieth annual symposium on
  Computational geometry}, pages 253--262, 2004.
\newblock \href{https://doi.org/10.1145/997817.997857}{[URL]}.

\bibitem[Deng et~al.(2009)Deng, Dong, Socher, Li, Li, and
  Fei-Fei]{deng2009imagenet}
Jia Deng, Wei Dong, Richard Socher, Li-Jia Li, Kai Li, and Li Fei-Fei.
\newblock {ImageNet}: A large-scale hierarchical image database.
\newblock In \emph{IEEE Conference on Computer Vision and Pattern Recognition
  (CVPR)}, pages 248--255, 2009.
\newblock \href{https://doi.org/10.1109/CVPR.2009.5206848}{[URL]}.

\bibitem[Dosovitskiy et~al.(2021)Dosovitskiy, Beyer, Kolesnikov, Weissenborn,
  Zhai, Unterthiner, Dehghani, Minderer, Heigold, Gelly, Uszkoreit, and
  Houlsby]{dosovitskiy2021vit}
Alexey Dosovitskiy, Lucas Beyer, Alexander Kolesnikov, Dirk Weissenborn,
  Xiaohua Zhai, Thomas Unterthiner, Mostafa Dehghani, Matthias Minderer, Georg
  Heigold, Sylvain Gelly, Jakob Uszkoreit, and Neil Houlsby.
\newblock {An Image is Worth 16x16 Words: Transformers for Image Recognition at
  Scale}.
\newblock In \emph{International Conference on Learning Representations
  (ICLR)}, 2021.
\newblock \href{https://iclr.cc/virtual/2021/oral/3458}{[URL]}.

\bibitem[Douze et~al.(2021)Douze, Tolias, Pizzi, Papakipos, Chanussot,
  Radenovic, Jenicek, Maximov, Leal-Taixé, Elezi, Chum, and
  Ferrer]{douze2021isc}
Matthijs Douze, Giorgos Tolias, Ed Pizzi, Zoë Papakipos, Lowik Chanussot,
  Filip Radenovic, Tomas Jenicek, Maxim Maximov, Laura Leal-Taixé, Ismail
  Elezi, Ondřej Chum, and Cristian~Canton Ferrer.
\newblock {The 2021 Image Similarity Dataset and Challenge}.
\newblock \emph{arXiv preprint arXiv:2106.09672}, 2021.
\newblock \href{https://arxiv.org/abs/2106.09672}{[URL]}.

\bibitem[Douze et~al.(2026)Douze, Guzhva, Deng, Johnson, Szilvasy, Mazaré,
  Lomeli, Hosseini, and Jégou]{douze2024faisslib}
Matthijs Douze, Alexandr Guzhva, Chengqi Deng, Jeff Johnson, Gergely Szilvasy,
  Pierre-Emmanuel Mazaré, Maria Lomeli, Lucas Hosseini, and Hervé Jégou.
\newblock {The Faiss Library}.
\newblock \emph{IEEE Transactions on Big Data}, 12\penalty0 (2):\penalty0
  346--361, 2026.
\newblock \href{https://doi.org/10.1109/TBDATA.2025.3618474}{[URL]}.

\bibitem[Faloutsos et~al.(1994)Faloutsos, Ranganathan, and
  Manolopoulos]{faloutsos1994gemini}
Christos Faloutsos, M. Ranganathan, and Yannis Manolopoulos.
\newblock {Fast subsequence matching in time-series databases}.
\newblock In \emph{Proceedings of the 1994 ACM SIGMOD international conference
  on Management of data}, pages 419--429, 1994.
\newblock \href{https://doi.org/10.1145/191839.191925}{[URL]}.

\bibitem[Fortune and Van~Wyk(1993)]{fortune1993exact}
Steven Fortune and Christopher~J. Van~Wyk.
\newblock Efficient exact arithmetic for computational geometry.
\newblock In \emph{Proceedings of the Ninth Annual Symposium on Computational
  Geometry (SoCG)}, pages 163--172, 1993.
\newblock \href{https://doi.org/10.1145/160985.161015}{[URL]}.

\bibitem[Gadre et~al.(2023)Gadre, Ilharco, Fang, Hayase, Smyrnis, Nguyen,
  Marten, Wortsman, Ghosh, Zhang, Orgad, Entezari, Daras, Pratt, Ramanujan,
  Bitton, Marathe, Mussmann, Vencu, Cherti, Krishna, Koh, Saukh, Ratner, Song,
  Hajishirzi, Farhadi, Beaumont, Oh, Dimakis, Jitsev, Carmon, Shankar, and
  Schmidt]{gadre2023datacomp}
Samir~Yitzhak Gadre, Gabriel Ilharco, Alex Fang, Jonathan Hayase, Georgios
  Smyrnis, Thao Nguyen, Ryan Marten, Mitchell Wortsman, Dhruba Ghosh, Jieyu
  Zhang, Eyal Orgad, Rahim Entezari, Giannis Daras, Sarah Pratt, Vivek
  Ramanujan, Yonatan Bitton, Kalyani Marathe, Stephen Mussmann, Richard Vencu,
  Mehdi Cherti, Ranjay Krishna, Pang~Wei Koh, Olga Saukh, Alexander Ratner,
  Shuran Song, Hannaneh Hajishirzi, Ali Farhadi, Romain Beaumont, Sewoong Oh,
  Alex Dimakis, Jenia Jitsev, Yair Carmon, Vaishaal Shankar, and Ludwig
  Schmidt.
\newblock {DataComp: In search of the next generation of multimodal datasets}.
\newblock In \emph{Advances in Neural Information Processing Systems
  (NeurIPS)}, 2023.
\newblock \href{https://doi.org/10.52202/075280-1179}{[URL]}.

\bibitem[Gao and Long(2024)]{gao2024rabitq}
Jianyang Gao and Cheng Long.
\newblock {RaBitQ: Quantizing High-Dimensional Vectors with a Theoretical Error
  Bound for Approximate Nearest Neighbor Search}.
\newblock \emph{Proceedings of the ACM on Management of Data (SIGMOD)}, 2024.
\newblock \href{https://doi.org/10.1145/3654970}{[URL]}.

\bibitem[Gao et~al.(2025)Gao, Gou, Xu, Yang, Long, and Wong]{gao2025extrabitq}
Jianyang Gao, Yutong Gou, Yuexuan Xu, Yongyi Yang, Cheng Long, and Raymond
  Chi-Wing Wong.
\newblock {Practical and Asymptotically Optimal Quantization of
  High-Dimensional Vectors in Euclidean Space for Approximate Nearest Neighbor
  Search}.
\newblock \emph{Proceedings of the ACM on Management of Data}, 2025.
\newblock \href{https://doi.org/10.1145/3725413}{[URL]}.

\bibitem[Garcia et~al.(2008)Garcia, Debreuve, and Barlaud]{garcia2008gpuknn}
Vincent Garcia, Eric Debreuve, and Michel Barlaud.
\newblock {Fast k nearest neighbor search using GPU}.
\newblock In \emph{2008 IEEE Computer Society Conference on Computer Vision and
  Pattern Recognition Workshops}, pages 1--6, 2008.
\newblock \href{https://doi.org/10.1109/cvprw.2008.4563100}{[URL]}.

\bibitem[Ge et~al.(2013)Ge, He, Ke, and Sun]{ge2013opq}
Tiezheng Ge, Kaiming He, Qifa Ke, and Jian Sun.
\newblock {Optimized Product Quantization for Approximate Nearest Neighbor
  Search}.
\newblock In \emph{IEEE Conference on Computer Vision and Pattern Recognition
  (CVPR)}, pages 2946--2953, 2013.
\newblock \href{https://doi.org/10.1109/CVPR.2013.379}{[URL]}.

\bibitem[Groh et~al.(2023)Groh, Ruppert, Wieschollek, and Lensch]{groh2023ggnn}
Fabian Groh, Lukas Ruppert, Patrick Wieschollek, and Hendrik P.~A. Lensch.
\newblock {GGNN: Graph-Based GPU Nearest Neighbor Search}.
\newblock \emph{IEEE Transactions on Big Data}, 9\penalty0 (1):\penalty0
  267--279, 2023.
\newblock \href{https://doi.org/10.1109/tbdata.2022.3161156}{[URL]}.

\bibitem[Guenter et~al.(2012)Guenter, Finch, Drucker, Tan, and
  Snyder]{guenter2012foveated}
Brian Guenter, Mark Finch, Steven Drucker, Desney Tan, and John Snyder.
\newblock {Foveated 3D graphics}.
\newblock \emph{ACM Transactions on Graphics}, 31\penalty0 (6):\penalty0 1--10,
  2012.
\newblock \href{https://doi.org/10.1145/2366145.2366183}{[URL]}.

\bibitem[Guo et~al.(2020)Guo, Sun, Lindgren, Geng, Simcha, Chern, and
  Kumar]{guo2020scann}
Ruiqi Guo, Philip Sun, Erik Lindgren, Quan Geng, David Simcha, Felix Chern, and
  Sanjiv Kumar.
\newblock {Accelerating Large-Scale Inference with Anisotropic Vector
  Quantization}.
\newblock In \emph{International Conference on Machine Learning (ICML)}, 2020.
\newblock \href{https://proceedings.mlr.press/v119/guo20h.html}{[URL]}.

\bibitem[Higham(2002)]{higham2002accuracy}
Nicholas~J. Higham.
\newblock \emph{{Accuracy and Stability of Numerical Algorithms}}.
\newblock Society for Industrial and Applied Mathematics, 2002.
\newblock \href{https://doi.org/10.1137/1.9780898718027}{[URL]}.

\bibitem[Indyk and Motwani(1998)]{indyk1998lsh}
Piotr Indyk and Rajeev Motwani.
\newblock {Approximate Nearest Neighbors: Towards Removing the Curse of
  Dimensionality}.
\newblock In \emph{Proceedings of the thirtieth annual ACM symposium on Theory
  of computing - STOC '98}, pages 604--613, 1998.
\newblock \href{https://doi.org/10.1145/276698.276876}{[URL]}.

\bibitem[Itti et~al.(1998)Itti, Koch, and Niebur]{itti1998saliency}
L. Itti, C. Koch, and E. Niebur.
\newblock {A model of saliency-based visual attention for rapid scene
  analysis}.
\newblock \emph{IEEE Transactions on Pattern Analysis and Machine
  Intelligence}, 20\penalty0 (11):\penalty0 1254--1259, 1998.
\newblock \href{https://doi.org/10.1109/34.730558}{[URL]}.

\bibitem[Jegou et~al.(2008)Jegou, Douze, and Schmid]{jegou2008hamming}
Herve Jegou, Matthijs Douze, and Cordelia Schmid.
\newblock {Hamming Embedding and Weak Geometric Consistency for Large Scale
  Image Search}.
\newblock In \emph{European Conference on Computer Vision (ECCV)}, pages
  304--317, 2008.
\newblock \href{https://doi.org/10.1007/978-3-540-88682-2_24}{[URL]}.

\bibitem[Jegou et~al.(2011)Jegou, Tavenard, Douze, and
  Amsaleg]{jegou2011bigann}
Herve Jegou, Romain Tavenard, Matthijs Douze, and Laurent Amsaleg.
\newblock {Searching in one billion vectors: Re-rank with source coding}.
\newblock In \emph{2011 IEEE International Conference on Acoustics, Speech and
  Signal Processing (ICASSP)}, pages 861--864, 2011.
\newblock \href{https://doi.org/10.1109/ICASSP.2011.5946540}{[URL]}.

\bibitem[Jim{\'e}nez~Guti{\'e}rrez et~al.(2024)Jim{\'e}nez~Guti{\'e}rrez, Shu,
  Gu, Yasunaga, and Su]{gutierrez2024hipporag}
Bernal Jim{\'e}nez~Guti{\'e}rrez, Yiheng Shu, Yu Gu, Michihiro Yasunaga, and Yu
  Su.
\newblock {HippoRAG: Neurobiologically Inspired Long-Term Memory for Large
  Language Models}.
\newblock In \emph{Advances in Neural Information Processing Systems
  (NeurIPS)}, 2024.
\newblock \href{https://openreview.net/forum?id=hkujvAPVsg}{[URL]}.

\bibitem[Johnson et~al.(2021)Johnson, Douze, and Jegou]{johnson2021faiss}
Jeff Johnson, Matthijs Douze, and Herve Jegou.
\newblock {Billion-Scale Similarity Search with GPUs}.
\newblock \emph{IEEE Transactions on Big Data}, 7\penalty0 (3):\penalty0
  535--547, 2021.
\newblock \href{https://doi.org/10.1109/TBDATA.2019.2921572}{[URL]}.

\bibitem[Johnson and Lindenstrauss(1984)]{johnson1984jl}
William~B. Johnson and Joram Lindenstrauss.
\newblock {Extensions of Lipschitz mappings into a Hilbert space}.
\newblock \emph{Contemporary Mathematics}, pages 189--206, 1984.
\newblock \href{https://doi.org/10.1090/conm/026/737400}{[URL]}.

\bibitem[Jégou et~al.(2011)Jégou, Douze, and Schmid]{jegou2011pq}
Hervé Jégou, Matthijs Douze, and Cordelia Schmid.
\newblock {Product Quantization for Nearest Neighbor Search}.
\newblock \emph{IEEE Transactions on Pattern Analysis and Machine
  Intelligence}, 33\penalty0 (1):\penalty0 117--128, 2011.
\newblock \href{https://doi.org/10.1109/TPAMI.2010.57}{[URL]}.

\bibitem[Khandelwal et~al.(2020)Khandelwal, Levy, Jurafsky, Zettlemoyer, and
  Lewis]{khandelwal2020knnlm}
Urvashi Khandelwal, Omer Levy, Dan Jurafsky, Luke Zettlemoyer, and Mike Lewis.
\newblock {Generalization through Memorization: Nearest Neighbor Language
  Models}.
\newblock In \emph{International Conference on Learning Representations
  (ICLR)}, 2020.
\newblock \href{https://iclr.cc/virtual/2020/poster/1815}{[URL]}.

\bibitem[Lewis et~al.(2020)Lewis, Perez, Piktus, Petroni, Karpukhin, Goyal,
  K{\"u}ttler, Lewis, Yih, Rockt{\"a}schel, Riedel, and Kiela]{lewis2020rag}
Patrick Lewis, Ethan Perez, Aleksandra Piktus, Fabio Petroni, Vladimir
  Karpukhin, Naman Goyal, Heinrich K{\"u}ttler, Mike Lewis, Wen-tau Yih, Tim
  Rockt{\"a}schel, Sebastian Riedel, and Douwe Kiela.
\newblock {Retrieval-Augmented Generation for Knowledge-Intensive NLP Tasks}.
\newblock In \emph{Advances in Neural Information Processing Systems
  (NeurIPS)}, 2020.
\newblock
  \href{https://proceedings.neurips.cc/paper/2020/hash/6b493230205f780e1bc26945df7481e5-Abstract.html}{[URL]}.

\bibitem[Liang et~al.(2022)Liang, Zhang, Kwon, Yeung, and
  Zou]{liang2022mindgap}
Weixin Liang, Yuhui Zhang, Yongchan Kwon, Serena Yeung, and James Zou.
\newblock {Mind the Gap: Understanding the Modality Gap in Multi-modal
  Contrastive Representation Learning}.
\newblock In \emph{Advances in Neural Information Processing Systems
  (NeurIPS)}, 2022.
\newblock
  \href{https://proceedings.neurips.cc/paper_files/paper/2022/hash/702f4db7543a7432431df588d57bc7c9-Abstract-Conference.html}{[URL]}.

\bibitem[Long et~al.(2022)Long, Yin, Ajanthan, Nguyen, Purkait, Garg, Blair,
  Shen, and van~den Hengel]{long2022rac}
Alexander Long, Wei Yin, Thalaiyasingam Ajanthan, Vu Nguyen, Pulak Purkait,
  Ravi Garg, Alan Blair, Chunhua Shen, and Anton van~den Hengel.
\newblock {Retrieval Augmented Classification for Long-Tail Visual
  Recognition}.
\newblock In \emph{2022 IEEE/CVF Conference on Computer Vision and Pattern
  Recognition (CVPR)}, pages 6949--6959, 2022.
\newblock \href{https://doi.org/10.1109/cvpr52688.2022.00683}{[URL]}.

\bibitem[Lowe(2004)]{lowe2004sift}
David~G. Lowe.
\newblock {Distinctive Image Features from Scale-Invariant Keypoints}.
\newblock \emph{International Journal of Computer Vision}, 60\penalty0
  (2):\penalty0 91--110, 2004.
\newblock \href{https://doi.org/10.1023/B:VISI.0000029664.99615.94}{[URL]}.

\bibitem[Malkov and Yashunin(2020)]{malkov2020hnsw}
Yu~A. Malkov and D.~A. Yashunin.
\newblock {Efficient and Robust Approximate Nearest Neighbor Search Using
  Hierarchical Navigable Small World Graphs}.
\newblock \emph{IEEE Transactions on Pattern Analysis and Machine
  Intelligence}, 42\penalty0 (4):\penalty0 824--836, 2020.
\newblock \href{https://doi.org/10.1109/TPAMI.2018.2889473}{[URL]}.

\bibitem[Micikevicius et~al.(2018)Micikevicius, Narang, Alben, Diamos, Elsen,
  Garcia, Ginsburg, Houston, Kuchaiev, Venkatesh, and
  Wu]{micikevicius2018mixed}
Paulius Micikevicius, Sharan Narang, Jonah Alben, Gregory Diamos, Erich Elsen,
  David Garcia, Boris Ginsburg, Michael Houston, Oleksii Kuchaiev, Ganesh
  Venkatesh, and Hao Wu.
\newblock {Mixed Precision Training}.
\newblock In \emph{International Conference on Learning Representations
  (ICLR)}, 2018.
\newblock \href{https://iclr.cc/virtual/2018/poster/288}{[URL]}.

\bibitem[Mnih et~al.(2014)Mnih, Heess, Graves, and Kavukcuoglu]{mnih2014ram}
Volodymyr Mnih, Nicolas Heess, Alex Graves, and Koray Kavukcuoglu.
\newblock {Recurrent Models of Visual Attention}.
\newblock In \emph{Advances in Neural Information Processing Systems
  (NeurIPS)}, 2014.
\newblock
  \href{https://papers.nips.cc/paper/5542-recurrent-models-of-visual-attention}{[URL]}.

\bibitem[Muja and Lowe(2014)]{muja2014flann}
Marius Muja and David~G. Lowe.
\newblock {Scalable Nearest Neighbor Algorithms for High Dimensional Data}.
\newblock \emph{IEEE Transactions on Pattern Analysis and Machine
  Intelligence}, 36\penalty0 (11):\penalty0 2227--2240, 2014.
\newblock \href{https://doi.org/10.1109/tpami.2014.2321376}{[URL]}.

\bibitem[{NVIDIA}(2026)]{nvidia2025cuvs}
{NVIDIA}.
\newblock {cuVS}: Vector search and clustering on the {GPU}.
\newblock \url{https://github.com/NVIDIA/cuvs}, 2026.
\newblock \href{https://github.com/NVIDIA/cuvs}{[URL]}.
\newblock Version 26.10.

\bibitem[Oliva and Torralba(2001)]{oliva2001gist}
Aude Oliva and Antonio Torralba.
\newblock {Modeling the Shape of the Scene: A Holistic Representation of the
  Spatial Envelope}.
\newblock \emph{International Journal of Computer Vision}, 42\penalty0
  (3):\penalty0 145--175, 2001.
\newblock \href{https://doi.org/10.1023/A:1011139631724}{[URL]}.

\bibitem[Ootomo et~al.(2024)Ootomo, Naruse, Nolet, Wang, Feher, and
  Wang]{ootomo2024cagra}
Hiroyuki Ootomo, Akira Naruse, Corey Nolet, Ray Wang, Tamas Feher, and Yong
  Wang.
\newblock {CAGRA: Highly Parallel Graph Construction and Approximate Nearest
  Neighbor Search for GPUs}.
\newblock In \emph{IEEE International Conference on Data Engineering (ICDE)},
  2024.
\newblock \href{https://doi.org/10.1109/ICDE60146.2024.00323}{[URL]}.

\bibitem[Oquab et~al.(2024)Oquab, Darcet, Moutakanni, Vo, Szafraniec, Khalidov,
  Fernandez, Haziza, Massa, El-Nouby, Assran, Ballas, Galuba, Howes, Huang, Li,
  Misra, Rabbat, Sharma, Synnaeve, Xu, Jegou, Mairal, Labatut, Joulin, and
  Bojanowski]{oquab2024dinov2}
Maxime Oquab, Timothée Darcet, Théo Moutakanni, Huy Vo, Marc Szafraniec,
  Vasil Khalidov, Pierre Fernandez, Daniel Haziza, Francisco Massa, Alaaeldin
  El-Nouby, Mahmoud Assran, Nicolas Ballas, Wojciech Galuba, Russell Howes,
  Po-Yao Huang, Shang-Wen Li, Ishan Misra, Michael Rabbat, Vasu Sharma, Gabriel
  Synnaeve, Hu Xu, Hervé Jegou, Julien Mairal, Patrick Labatut, Armand Joulin,
  and Piotr Bojanowski.
\newblock {DINOv2: Learning Robust Visual Features without Supervision}.
\newblock \emph{Transactions on Machine Learning Research}, 2024.
\newblock \href{https://openreview.net/forum?id=a68SUt6zFt}{[URL]}.

\bibitem[Paszke et~al.(2019)Paszke, Gross, Massa, Lerer, Bradbury, Chanan,
  Killeen, Lin, Gimelshein, Antiga, Desmaison, Köpf, Yang, DeVito, Raison,
  Tejani, Chilamkurthy, Steiner, Fang, Bai, and Chintala]{paszke2019pytorch}
Adam Paszke, Sam Gross, Francisco Massa, Adam Lerer, James Bradbury, Gregory
  Chanan, Trevor Killeen, Zeming Lin, Natalia Gimelshein, Luca Antiga, Alban
  Desmaison, Andreas Köpf, Edward Yang, Zach DeVito, Martin Raison, Alykhan
  Tejani, Sasank Chilamkurthy, Benoit Steiner, Lu Fang, Junjie Bai, and Soumith
  Chintala.
\newblock {PyTorch: An Imperative Style, High-Performance Deep Learning
  Library}.
\newblock In \emph{Advances in Neural Information Processing Systems
  (NeurIPS)}, 2019.
\newblock
  \href{https://proceedings.neurips.cc/paper/2019/hash/bdbca288fee7f92f2bfa9f7012727740-Abstract.html}{[URL]}.

\bibitem[Pennington et~al.(2014)Pennington, Socher, and
  Manning]{pennington2014glove}
Jeffrey Pennington, Richard Socher, and Christopher~D. Manning.
\newblock {GloVe: Global Vectors for Word Representation}.
\newblock In \emph{Proceedings of the Conference on Empirical Methods in
  Natural Language Processing (EMNLP)}, pages 1532--1543, 2014.
\newblock \href{https://aclanthology.org/D14-1162/}{[URL]}.

\bibitem[Philbin et~al.(2007)Philbin, Chum, Isard, Sivic, and
  Zisserman]{philbin2007oxford}
James Philbin, Ondrej Chum, Michael Isard, Josef Sivic, and Andrew Zisserman.
\newblock {Object retrieval with large vocabularies and fast spatial matching}.
\newblock In \emph{2007 IEEE Conference on Computer Vision and Pattern
  Recognition}, pages 1--8, 2007.
\newblock \href{https://doi.org/10.1109/cvpr.2007.383172}{[URL]}.

\bibitem[Pizzi et~al.(2022)Pizzi, Roy, Ravindra, Goyal, and
  Douze]{pizzi2022sscd}
Ed Pizzi, Sreya~Dutta Roy, Sugosh~Nagavara Ravindra, Priya Goyal, and Matthijs
  Douze.
\newblock {A Self-Supervised Descriptor for Image Copy Detection}.
\newblock In \emph{IEEE/CVF Conference on Computer Vision and Pattern
  Recognition (CVPR)}, 2022.
\newblock \href{https://doi.org/10.1109/CVPR52688.2022.01413}{[URL]}.

\bibitem[Radenovic et~al.(2018)Radenovic, Iscen, Tolias, Avrithis, and
  Chum]{radenovic2018roxford}
Filip Radenovic, Ahmet Iscen, Giorgos Tolias, Yannis Avrithis, and Ondrej Chum.
\newblock {Revisiting Oxford and Paris: Large-Scale Image Retrieval
  Benchmarking}.
\newblock In \emph{2018 IEEE/CVF Conference on Computer Vision and Pattern
  Recognition}, pages 5706--5715, 2018.
\newblock \href{https://doi.org/10.1109/cvpr.2018.00598}{[URL]}.

\bibitem[Radford et~al.(2021)Radford, Kim, Hallacy, Ramesh, Goh, Agarwal,
  Sastry, Askell, Mishkin, Clark, Krueger, and Sutskever]{radford2021clip}
Alec Radford, Jong~Wook Kim, Chris Hallacy, Aditya Ramesh, Gabriel Goh,
  Sandhini Agarwal, Girish Sastry, Amanda Askell, Pamela Mishkin, Jack Clark,
  Gretchen Krueger, and Ilya Sutskever.
\newblock {Learning Transferable Visual Models From Natural Language
  Supervision}.
\newblock In \emph{International Conference on Machine Learning (ICML)}, 2021.
\newblock \href{https://proceedings.mlr.press/v139/radford21a.html}{[URL]}.

\bibitem[Russakovsky et~al.(2015)Russakovsky, Deng, Su, Krause, Satheesh, Ma,
  Huang, Karpathy, Khosla, Bernstein, Berg, and Fei-Fei]{russakovsky2015ilsvrc}
Olga Russakovsky, Jia Deng, Hao Su, Jonathan Krause, Sanjeev Satheesh, Sean Ma,
  Zhiheng Huang, Andrej Karpathy, Aditya Khosla, Michael Bernstein,
  Alexander~C. Berg, and Li Fei-Fei.
\newblock {ImageNet Large Scale Visual Recognition Challenge}.
\newblock \emph{International Journal of Computer Vision}, 115\penalty0
  (3):\penalty0 211--252, 2015.
\newblock \href{https://doi.org/10.1007/s11263-015-0816-y}{[URL]}.

\bibitem[Schuhmann et~al.(2022)Schuhmann, Beaumont, Vencu, Gordon, Wightman,
  Cherti, Coombes, Katta, Mullis, Wortsman, Schramowski, Kundurthy, Crowson,
  Schmidt, Kaczmarczyk, and Jitsev]{schuhmann2022laion}
Christoph Schuhmann, Romain Beaumont, Richard Vencu, Cade Gordon, Ross
  Wightman, Mehdi Cherti, Theo Coombes, Aarush Katta, Clayton Mullis, Mitchell
  Wortsman, Patrick Schramowski, Srivatsa Kundurthy, Katherine Crowson, Ludwig
  Schmidt, Robert Kaczmarczyk, and Jenia Jitsev.
\newblock {LAION-5B: An open large-scale dataset for training next generation
  image-text models}.
\newblock In \emph{Advances in Neural Information Processing Systems (NeurIPS)
  Datasets and Benchmarks Track}, 2022.
\newblock
  \href{https://proceedings.neurips.cc/paper_files/paper/2022/hash/a1859debfb3b59d094f3504d5ebb6c25-Abstract-Datasets_and_Benchmarks.html}{[URL]}.

\bibitem[Shewchuk(1997)]{shewchuk1997adaptive}
Jonathan~Richard Shewchuk.
\newblock Adaptive precision floating-point arithmetic and fast robust
  geometric predicates.
\newblock \emph{Discrete \& Computational Geometry}, 18\penalty0 (3):\penalty0
  305--363, 1997.
\newblock \href{https://doi.org/10.1007/PL00009321}{[URL]}.

\bibitem[Simhadri et~al.(2022)Simhadri, Williams, Aumüller, Douze, Babenko,
  Baranchuk, Chen, Hosseini, Krishnaswamy, Srinivasa, Subramanya, and
  Wang]{simhadri2022bigann}
Harsha~Vardhan Simhadri, George Williams, Martin Aumüller, Matthijs Douze,
  Artem Babenko, Dmitry Baranchuk, Qi Chen, Lucas Hosseini, Ravishankar
  Krishnaswamy, Gopal Srinivasa, Suhas~Jayaram Subramanya, and Jingdong Wang.
\newblock {Results of the NeurIPS'21 Challenge on Billion-Scale Approximate
  Nearest Neighbor Search}.
\newblock In \emph{Proceedings of the NeurIPS 2021 Competitions and
  Demonstrations Track}, pages 177--189, 2022.
\newblock \href{https://proceedings.mlr.press/v176/simhadri22a.html}{[URL]}.

\bibitem[Siméoni et~al.(2025)Siméoni, Vo, Seitzer, Baldassarre, Oquab, Jose,
  Khalidov, Szafraniec, Yi, Ramamonjisoa, Massa, Haziza, Wehrstedt, Wang,
  Darcet, Moutakanni, Sentana, Roberts, Vedaldi, Tolan, Brandt, Couprie,
  Mairal, Jégou, Labatut, and Bojanowski]{simeoni2025dinov3}
Oriane Siméoni, Huy~V. Vo, Maximilian Seitzer, Federico Baldassarre, Maxime
  Oquab, Cijo Jose, Vasil Khalidov, Marc Szafraniec, Seungeun Yi, Michaël
  Ramamonjisoa, Francisco Massa, Daniel Haziza, Luca Wehrstedt, Jianyuan Wang,
  Timothée Darcet, Théo Moutakanni, Leonel Sentana, Claire Roberts, Andrea
  Vedaldi, Jamie Tolan, John Brandt, Camille Couprie, Julien Mairal, Hervé
  Jégou, Patrick Labatut, and Piotr Bojanowski.
\newblock {DINOv3}.
\newblock \emph{arXiv preprint arXiv:2508.10104}, 2025.
\newblock \href{https://arxiv.org/abs/2508.10104}{[URL]}.

\bibitem[Sivic and Zisserman(2003)]{sivic2003videogoogle}
Josef Sivic and Andrew Zisserman.
\newblock {Video Google: A Text Retrieval Approach to Object Matching in
  Videos}.
\newblock In \emph{Proceedings Ninth IEEE International Conference on Computer
  Vision}, pages 1470--1477 vol.2, 2003.
\newblock \href{https://doi.org/10.1109/iccv.2003.1238663}{[URL]}.

\bibitem[Subramanya et~al.(2019)Subramanya, Devvrit, Kadekodi, Krishnaswamy,
  and Simhadri]{subramanya2019diskann}
Suhas~Jayaram Subramanya, Devvrit, Rohan Kadekodi, Ravishankar Krishnaswamy,
  and Harsha~Vardhan Simhadri.
\newblock {DiskANN: Fast Accurate Billion-point Nearest Neighbor Search on a
  Single Node}.
\newblock In \emph{Advances in Neural Information Processing Systems
  (NeurIPS)}, 2019.
\newblock
  \href{https://papers.nips.cc/paper_files/paper/2019/hash/09853c7fb1d3f8ee67a61b6bf4a7f8e6-Abstract.html}{[URL]}.

\bibitem[Tillet et~al.(2019)Tillet, Kung, and Cox]{tillet2019triton}
Philippe Tillet, H.~T. Kung, and David Cox.
\newblock {Triton: an intermediate language and compiler for tiled neural
  network computations}.
\newblock In \emph{Proceedings of the 3rd ACM SIGPLAN International Workshop on
  Machine Learning and Programming Languages}, pages 10--19, 2019.
\newblock \href{https://doi.org/10.1145/3315508.3329973}{[URL]}.

\bibitem[Treisman and Gelade(1980)]{treisman1980feature}
Anne~M. Treisman and Garry Gelade.
\newblock {A feature-integration theory of attention}.
\newblock \emph{Cognitive Psychology}, 12\penalty0 (1):\penalty0 97--136, 1980.
\newblock \href{https://doi.org/10.1016/0010-0285(80)90005-5}{[URL]}.

\bibitem[Vidal~Ruiz(1986)]{vidal1986aesa}
Enrique Vidal~Ruiz.
\newblock {An algorithm for finding nearest neighbours in (approximately)
  constant average time}.
\newblock \emph{Pattern Recognition Letters}, 4\penalty0 (3):\penalty0
  145--157, 1986.
\newblock \href{https://doi.org/10.1016/0167-8655(86)90013-9}{[URL]}.

\bibitem[Wang et~al.(2023)Wang, Dong, Cheng, Liu, Yan, Gao, and
  Wei]{wang2023longmem}
Weizhi Wang, Li Dong, Hao Cheng, Xiaodong Liu, Xifeng Yan, Jianfeng Gao, and
  Furu Wei.
\newblock {Augmenting Language Models with Long-Term Memory}.
\newblock In \emph{Advances in Neural Information Processing Systems
  (NeurIPS)}, 2023.
\newblock \href{https://openreview.net/forum?id=BryMFPQ4L6}{[URL]}.

\bibitem[Weber et~al.(1998)Weber, Schek, and Blott]{weber1998vafile}
Roger Weber, Hans-J{\"o}rg Schek, and Stephen Blott.
\newblock {A Quantitative Analysis and Performance Study for Similarity-Search
  Methods in High-Dimensional Spaces}.
\newblock In \emph{Proceedings of the International Conference on Very Large
  Data Bases (VLDB)}, pages 194--205, 1998.
\newblock \href{https://www.vldb.org/conf/1998/p194.pdf}{[URL]}.

\bibitem[Wu et~al.(2022)Wu, Rabe, Hutchins, and Szegedy]{wu2022memorizing}
Yuhuai Wu, Markus~N. Rabe, DeLesley Hutchins, and Christian Szegedy.
\newblock {Memorizing Transformers}.
\newblock In \emph{International Conference on Learning Representations
  (ICLR)}, 2022.
\newblock \href{https://iclr.cc/virtual/2022/poster/6064}{[URL]}.

\bibitem[Wu et~al.(2018)Wu, Xiong, Yu, and Lin]{wu2018instance}
Zhirong Wu, Yuanjun Xiong, Stella Yu, and Dahua Lin.
\newblock {Unsupervised Feature Learning via Non-Parametric Instance-level
  Discrimination}.
\newblock In \emph{IEEE/CVF Conference on Computer Vision and Pattern
  Recognition (CVPR)}, 2018.
\newblock \href{https://doi.org/10.1109/CVPR.2018.00393}{[URL]}.

\bibitem[Xu et~al.(2024)Xu, Ping, Wu, McAfee, Zhu, Liu, Subramanian,
  Bakhturina, Shoeybi, and Catanzaro]{xu2024retrievallong}
Peng Xu, Wei Ping, Xianchao Wu, Lawrence McAfee, Chen Zhu, Zihan Liu, Sandeep
  Subramanian, Evelina Bakhturina, Mohammad Shoeybi, and Bryan Catanzaro.
\newblock {Retrieval Meets Long Context Large Language Models}.
\newblock In \emph{International Conference on Learning Representations
  (ICLR)}, 2024.
\newblock \href{https://openreview.net/forum?id=xw5nxFWMlo}{[URL]}.

\bibitem[Yandex and Lempitsky(2016)]{babenko2016deep1b}
Artem~Babenko Yandex and Victor Lempitsky.
\newblock {Efficient Indexing of Billion-Scale Datasets of Deep Descriptors}.
\newblock In \emph{2016 IEEE Conference on Computer Vision and Pattern
  Recognition (CVPR)}, pages 2055--2063, 2016.
\newblock \href{https://doi.org/10.1109/cvpr.2016.226}{[URL]}.

\bibitem[Zhai et~al.(2023)Zhai, Mustafa, Kolesnikov, and Beyer]{zhai2023siglip}
Xiaohua Zhai, Basil Mustafa, Alexander Kolesnikov, and Lucas Beyer.
\newblock {Sigmoid Loss for Language Image Pre-Training}.
\newblock In \emph{2023 IEEE/CVF International Conference on Computer Vision
  (ICCV)}, pages 11941--11952, 2023.
\newblock \href{https://doi.org/10.1109/iccv51070.2023.01100}{[URL]}.

\bibitem[Zhao(2026)]{zhao2026sivf}
Dongfang Zhao.
\newblock {SIVF: GPU-Resident IVF Index for Streaming Vector Analytics}.
\newblock In \emph{ACM International Symposium on High-Performance Parallel and
  Distributed Computing (HPDC)}, pages 31--44, 2026.
\newblock \href{https://doi.org/10.1145/3806645.3807575}{[URL]}.

\bibitem[Zhao et~al.(2020)Zhao, Tan, and Li]{zhao2020song}
Weijie Zhao, Shulong Tan, and Ping Li.
\newblock {SONG: Approximate Nearest Neighbor Search on GPU}.
\newblock In \emph{2020 IEEE 36th International Conference on Data Engineering
  (ICDE)}, pages 1033--1044, 2020.
\newblock \href{https://doi.org/10.1109/icde48307.2020.00094}{[URL]}.

\end{thebibliography}
}
\end{document}